\documentclass[11pt]{report}
\usepackage{USCthesis}
\usepackage{amssymb,amsmath,amsthm,amsfonts}
\usepackage{aurical}
\usepackage{booktabs}
\usepackage{times}
\usepackage{cite}
\usepackage[linesnumbered,ruled]{algorithm2e}
\usepackage{paralist}
\usepackage{latexsym}
\usepackage{wrapfig}
\usepackage{url}
\usepackage{sidecap}
\usepackage{verbatim} 
\usepackage{cases}
\usepackage{appendix}
\usepackage{bm}
\usepackage{xspace}
\usepackage{subfig}
\usepackage{diagbox}
\usepackage{makecell}
\usepackage{ifpdf}

\ifpdf
   \usepackage[pdftex]{graphicx}
\else
  \usepackage{graphicx}
\fi
\graphicspath{{./figures/}}

\usepackage{epsfig}
\usepackage{epstopdf}

\makeatletter
\let\latex@xfloat=\@xfloat
\def\@xfloat #1[#2]{%
  \latex@xfloat #1[#2]%
  \def\baselinestretch{1}
  \@normalsize\normalsize
  \normalsize
}

  \newtheorem{lemma}{Lemma}
  \newtheorem{proposition}{Proposition}

\def\theselectedpubs#1{
  \chapter*{Key Publications}
  \list{[\arabic{enumi}]}{
    \settowidth\labelwidth{[#1]}
    \leftmargin\labelwidth   \advance\leftmargin\labelsep
    \ifopenbib
      \listparindent -1.5em
      \advance\leftmargin-\listparindent
      \itemindent\listparindent
      \parsep 0pt%
    \fi
    \usecounter{enumi}}%
  \ifopenbib
    \def\newblock{\par}
    \let\\=\@centercr
    \@rightskip\@flushglue   \rightskip\@rightskip
    \leftskip\z@
  \else
    \def\newblock{\hskip .11em plus .33em minus -.07em}%
  \fi
  \sloppy
  \sfcode`\.=1000\relax}

\DeclareMathOperator*{\argmax}{argmax}
\DeclareMathOperator{\E}{\mathbb{E}}

\newtheorem{theorem}{Theorem}
\newtheorem{problem}{Problem}

\newcommand{\tuple}[1]{\ensuremath{\left \langle #1 \right \rangle }}
\newcommand{\I}{\mathcal{I}}
\makeatother

\pagetop{1.1 true in}       
\pageleft{1.6 true in}      
\pageheight{8.5 true in}    
\pagewidth{5.85 true in}    
\pagemargin{1.0}    
\submitdate{August 2018}

\begin{document}
\title{Artificial Intelligence for Low-Resource Communities: Influence Maximization in an Uncertain World}
\author{Amulya Yadav}
\majorfield{Computer Science}

\committee{Milind Tambe & (Chairperson) \\*
       Kristina Lerman\\*
       Aram Galstyan\\*
       Eric Rice & (Outside Member) \\*
       Dana Goldman &(External Member)}   
\begin{preface}
\prefacesection{Acknowledgment}
This thesis took a lot of time to write. If I were to live to the average human age of 60 years, then this thesis would have taken almost 10 percent of my life to write. During this period of time, a lot of good things have happened to me. I have had the chance to live in one of the world's best cities (Los Angeles) for five years. I got the opportunity to travel to 17 different countries in 5 different continents, and I also got to extend my student life by five additional years (which was one of the reasons why I came for a PhD). These five years have also made me experience a lot of things, some for the very first time. I have experienced both triumphs and defeats. There have been moments of joy interspersed with moments of sorrow. But in all of these five years, the thing that I cherish the most is that I was fortunate enough to meet and bond with an incredible group of people (who I now call my friends, mentors, and collaborators), who were there with me to share each and every moment of my life in the last five years. They were there when I was happy, and also when I was sad, and writing this thesis would have been impossible without them. As a result, each one of these people is a valued ``co-author" on this thesis.

One of the most important people during my PhD life has been my advisor, Milind Tambe.	Thank you so much Milind for taking me on under your guidance. This thesis would not have seen the light of the day, had it not been for your continuous support and encouragement. You have always given me the freedom to choose my own problems, and for that, I am grateful to you. You have not only taught me how to be a good researcher, but more importantly, you have taught me how to interact with people professionally on a day-to-day basis. I remember you mentioning that I have excellent social communication skills; those are your skills that I have tried to emulate as much as I could. Finally, thank you for creating such a positive and friendly environment for the members of your research group, as today, I am leaving with not just a PhD, but with so many close friends that I can continue to count upon (and be counted upon) in times of need. This would not have been possible without you. Thank you for being my mentor over the years but most importantly, thank you for being my friend. Thank you for making me a Doctor, Milind!

Next, I would like to thank my unofficial co-advisor, Eric Rice. Thank you so much Eric for everything that you have done for me over the years. None of the good things that have happened to me during my PhD would have happened if you had not given me the problem of this thesis in my first year. It has been an absolute pleasure learning from you and working with you. I still remember you answering questions from mean audience members during my conference presentations; and I really appreciate you looking out for me on this and other innumerable occasions. Also, that afternoon that I spent with you walking down Venice beach trying to get data from homeless youth remains by far, the coolest thing that I have done during my PhD. Thank you, Eric!

I would also like to thank my other thesis committee members: Kristina Lerman, Aram Galstyan, and Dana Goldman. Thank you so much for your invaluable feedback and mentorship over the years. This thesis would not be half as good without your timely and wonderful suggestions.

Over the years, I have also had the honor of collaborating and interacting with some great minds around the world. In particular, I would like to thank Eric Shieh and Thanh Nguyen for letting me write my first papers; Albert Jiang, Pradeep Varakantham, Phebe Vayanos, Leandro Marcolino, Eugene Vorobeychik and Hau Chan for your valuable insights that improved my research; Aida Rahmattalabi for being the most hard-working co-author that I have ever worked with; Ritesh Noothigattu for showing such great enthusiasm towards research (and almost everything else); Donnabell Dmello and Venil Noronha for painstakingly developing games for me; and finally, Robin Petering, Jaih Craddock, Laura Onasch-Vera and Hailey Winetrobe for implementing my research in the real-world. I don't think I would have been able to write my papers without your help. Thank you for making this thesis what it is!

This brings me to the present and past members of the Teamcore research group, who have created a home away from home for me during my PhD. It would have been impossible for me to spend these five years in a foreign country without the help and support of all my friends: Bryan Wilder, Elizabeth Bondi, Aida Rahmattalabi, Kai Wang,  Han Ching Ou, Biswarup Bhattacharya, Sarah Cooney, Francesco Della Fave, Albert Jiang, Will Haskell, Eric Shieh, Leandro Marcolino, Matthew Brown, Rong Yang, Jun Young Kwak, Chao Zhang, Yundi Qian, Fei Fang, Ben Ford, Elizabeth Orrico and Becca Funke. Thank you all for providing me with so many laughs over the years, and for your constant support. 

There are some people that I would like to thank in particular. Debarun Kar and Arunesh Sinha, thank you so much for serving as my partners in crime (for all our devious plans) during my PhD, both of you are no less than brothers to me, and my PhD would have been a lot less colorful had it not been for both of you. Yasaman Abbasi, thank you so much for treating me like your little brother, for bringing me so much free food, and for hiding my phone so many times :). Sara Marie McCarthy, thank you so much for all our post-gym chats on literally every topic known to mankind, for our heart-to-heart discussions on what we want to do in life, for your infectious bubbly nature, and for never ever giving me M\&M's when I wanted them. I'm going to miss you :). Shahrzad Gholami, thank you so much for being such a wonderful neighbor, I'm sorry for annoying and disturbing you so much over the years, and I promise I won't do it again :). Aaron Schlenker, thank you for trusting me with all your secrets, for those innumerable FIFA games that we played; Thanh Hong Nguyen and Haifeng Xu for being such great travel companions during conferences; and finally, Omkar Thakoor for our numerous discussions on football and life. I would also like to thank Aaron Ferber for teaching me about the world of finance, for being a reluctant gym buddy, and for being a great sport. I don't know how to even begin to say goodbye to all of you, and thus, I won't. We'll stay in touch, guys. I promise :)

There have been some other people who have played an important role during my PhD. Thank you Amandeep Singh and Vivek Tiwari for being such great friends over the years, for the constant laughs, and for keeping me sane during these five years. I look forward to many more years of friendship in the future. I would also like to thank my roommates: Vishnu Ratnam, Swarnabha Chattaraj and Pankaj Rajak. Thank you for putting up with my sub-standard cooking over the years, and for your constant care when I was sick in the last five years, and for all the amazing memories that we co-created.

Lastly, but most importantly, I would like to thank my family for their love and support throughout my life, which has made me the person I am today. Thank you Mamma and Dadi for being so protective of me, and for pampering me with delicious food over the years. Thank you to my sisters: Riya, Nupur and Shalaka for sending rakhis all these years, even though I was not in India. Thank you to my brothers: Ayush, Vasu and Mannu for being my cute little bros. Thank you Suman Mausi and Munna Mausi for always being there for me when I needed you. Thank you Ramsingh Mausaji and Mohan Mausaji for taking such good care of me throughout my life. Thank you Mamaji and Mamiji for your constant encouragement and motivation.

In particular, I want to thank my parents. Thank you Mummy and Papa: I cannot even begin to imagine how I could have made this journey if it were not for your blessings. Thank you for always supporting my dreams and for letting me stay away from home for such an extended period of time. Both of you have sacrificed your careers and lives just so that I could become something in life, and for that I am eternally grateful to you. Thank you for always putting my well-being first, and I promise that I will put your well-being first in the years to come. Thank you for always being by my side in my good and bad times. Whatever I have done (or will do) in my life, I did it to make you proud! I love you! 

\tableofcontents
\listoffigures 
\listoftables

\prefacesection{Abstract}
The potential of Artificial Intelligence (AI) to tackle challenging problems that afflict society is enormous, particularly in the areas of healthcare, conservation and public safety and security. Many problems in these domains involve harnessing social networks of under-served communities to enable positive change, e.g., using social networks of homeless youth to raise awareness about Human Immunodeficiency Virus (HIV) and other STDs. Unfortunately, most of these real-world problems are characterized by uncertainties about social network structure and influence models, and previous research in AI fails to sufficiently address these uncertainties, as they make several unrealistic simplifying assumptions for these domains. 

This thesis addresses these shortcomings by advancing the state-of-the-art to a new generation of algorithms for interventions in social networks. In particular, this thesis describes the design and development of new influence maximization algorithms which can handle various uncertainties  that commonly exist in real-world social networks (e.g., uncertainty in social network structure, evolving network state, and availability of nodes to get influenced). These algorithms utilize techniques from sequential planning problems and social network theory to develop new kinds of AI algorithms. Further, this thesis also demonstrates the real-world impact of these algorithms by describing their deployment in three pilot studies to spread awareness about HIV among actual homeless youth in Los Angeles. This represents one of the first-ever deployments of computer science based influence maximization algorithms in this domain. Our results show that our AI algorithms improved upon the state-of-the-art by $\sim$ 160\% in the real-world. We discuss research and implementation challenges faced in deploying these algorithms, and lessons that can be gleaned for future deployment of such algorithms. The positive results from these deployments illustrate the enormous potential of AI in addressing societally relevant problems.
\end{preface}

\chapter{Introduction}
\label{sec:intro}
The field of Artificial Intelligence (AI) has pervaded into many aspects of urban human living, and there are many AI based applications that we use on a daily basis. For example, we use AI based navigation systems (e.g., Google Maps, Waze) to find the quickest way home; we use AI based search engines (e.g., Google, Bing) to search for relevant information; and we use AI based personal assistant systems (e.g., Siri, Alexa) to organize our daily schedules, among other things.

Unfortunately, a significant proportion of people all over the world have not benefited from these AI technologies, primarily because they do not have access to these technologies. In particular, this lack of access to AI technologies is endemic to ``\textit{low-resource communities}", which are communities suffering from financial and social impoverishment, among other ills. Moreover, apart from lack of technology access, these communities suffer from completely different kinds of problems, which have not been tackled by AI and computer science as much \cite{Gomes2019}. For example, as shown in Figure \ref{fig:example}, homeless youth communities in North America do not have access to public health services, drug addicted people in North America \cite{rahmattalabi2019exploring} do not have access to rehabilitation facilities, and low-literate farmers in India do not have access to good governance \cite{mothilal2019optimizing}, etc. \textit{At a very high level, this thesis attempts to answer whether Artificial Intelligence can be utilized to solve the problems faced by these (and other) low-resource communities.} 

\begin{figure}[ht]    
    \centering
    \subfloat[Homeless Youth: Access to Public Health Facilities]{
    \includegraphics[height = 0.18\textwidth]{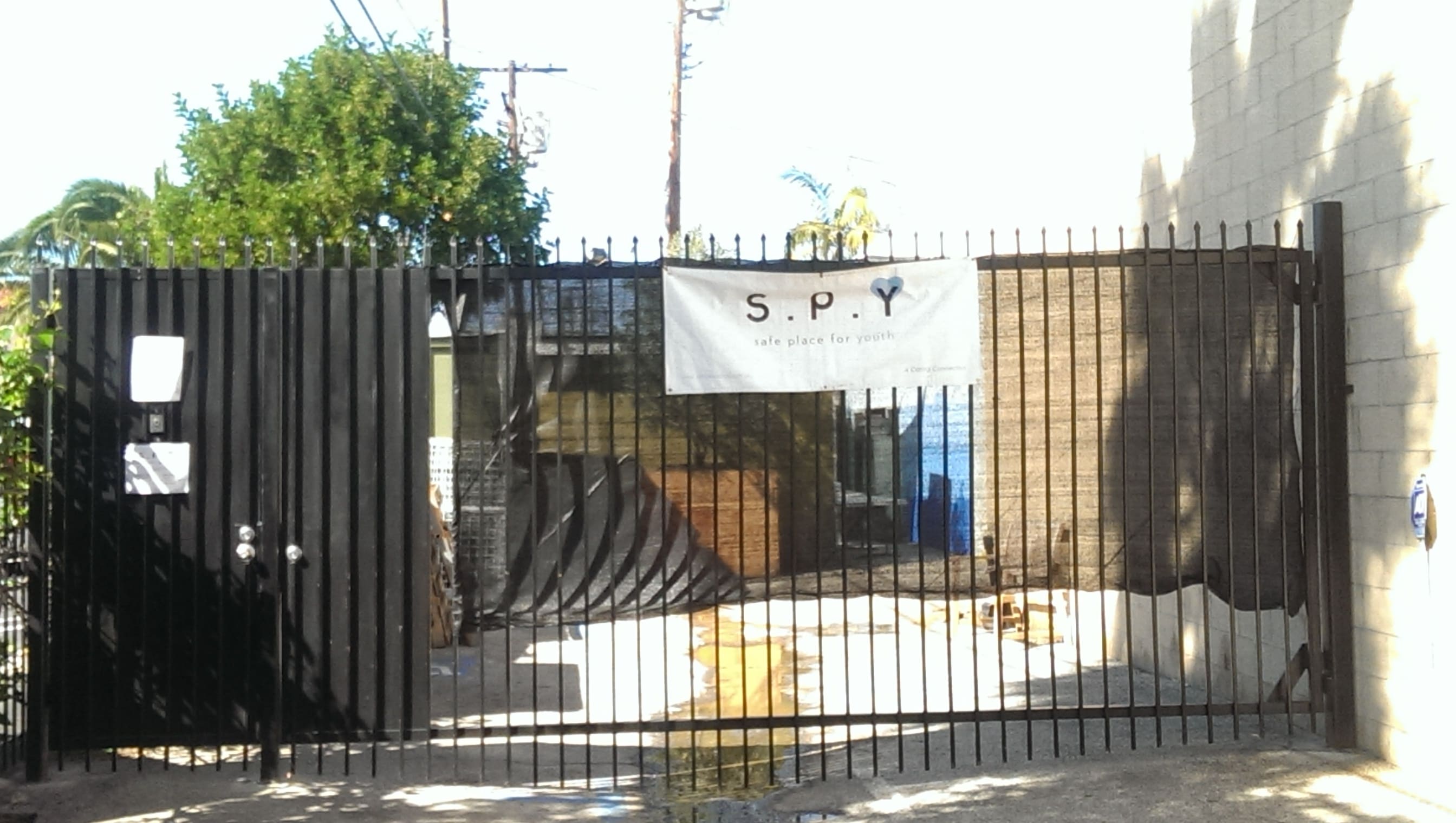}    
    }
    \subfloat[Drug Addicts: Access to Rehabilitation Facilities]{
    \includegraphics[height = 0.18\textwidth]{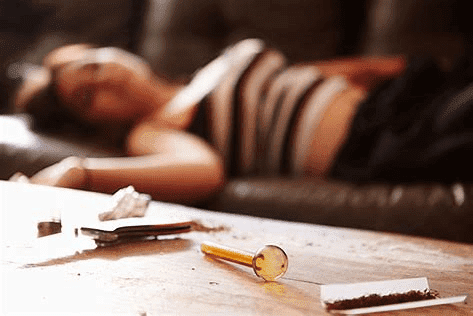}    
    }
    \subfloat[Low-Literate Farmers: Access to Grievance Redressal]{
    \includegraphics[height = 0.18\textwidth]{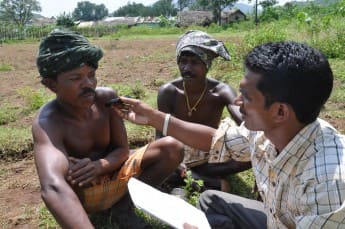}    
    }
    \caption{Low-Resource Communities and Their Problems}
    \label{fig:example}
\end{figure}

More specifically, this thesis focuses on how several challenges faced by these low-resource communities can be tackled by harnessing the real-world social networks of these communities. Since ancient times, humans have intertwined themselves into various social networks. These networks can be of many different kinds, such as friendship based networks, professional networks, etc. Besides these networks being used for more direct reasons (e.g., friendship based networks used for connecting with old and new friends, etc.), these networks also play a critical role in the formulation and propagation of opinions, ideas and information among the people in that network. In recent times, this property of social networks has been exploited by governments and non-profit organizations to conduct social and behavioral interventions among low-resource communities, in order to enable positive behavioral change among these communities. For example, non-profit agencies called homeless youth service providers conduct intervention camps periodically, where they train a small set of influential homeless youth as ``\textit{peer leaders}'' to spread awareness and information about HIV and STD prevention among their peers in their social circles. Unfortunately, such real-world interventions are almost always plagued by limited resources and limited data, which creates a computational challenge. This thesis addresses these challenges by providing algorithmic techniques to enhance the targeting and delivery of these social and behavioral interventions.

From a computer science perspective, the question of finding out the most ``\textit{influential}'' people in a social network is well studied in the field of influence maximization, which looks at the problem of selecting the $K$ (an input parameter) most influential nodes in a social network (represented as a graph), who will be able to influence the most number of people in the network within a given time periosd. Influence in these networks is assumed to spread according to a known \textit{influence model} (popular ones are independent cascade \cite{leskovec2007cost} and linear threshold \cite{chen2010scalable}). Since the field's inception in 2003 by Kempe et. al. \cite{kempe2003maximizing}, influence maximization has seen a lot of progress over the years \cite{leskovec2007cost,kimura2006tractable,chen2010scalable,cohen2014sketch,Borgs14,tang2014influence,bharathi2007competitive,kostka2008word,borodin2010threshold,lerman2016majority,ghosh2009leaders,ghosh2010community,ver2013information,galstyan2009maximizing,galstyan2008influence}.

\section{Problem Addressed}
\label{intro.sec:problem}
Unfortunately, most models and algorithms from previous work suffer from serious limitations. In particular, there are different kinds of uncertainties, constraints and challenges that need to be addressed in real-world domains involving low-resource communities, and previous work has failed to provide satisfactory solutions to address these limitations. 

Specifically, most previous work suffers from \textit{five major limitations}. First, almost every previous work focuses on single-shot decision problems, where only a single subset of graph nodes is to be chosen and then evaluated for influence spread. Instead, most realistic applications of influence maximization would require selection of nodes in multiple stages. For example, homeless youth service providers conduct multiple intervention camps sequentially, until they run out of their financial budget (instead of conducting just a single intervention camp). 

Second, the state of the social network is not known at any point in time; thus, the selection of nodes in multiple stages (which is un-handled in previous work) introduces additional uncertainty about which network nodes are influenced at a given point in time, which complicates the node selection procedure. Addressing this uncertainty is critical as otherwise, you can keep re-influencing nodes which have been already influenced via diffusion of information in previous interventions.

Third, network structure is assumed to be known with certainty in most previous work, which is untrue in reality, considering that there is always noise in any network data collection procedure. In particular, collecting network data from low-resource communities is cumbersome, as it entails surveying members of the community (e.g., homeless youth) about their friend circles. Invariably, the social networks that we get from homeless youth service providers have some friendships which we know with certainty (i.e., certain friendships) , and some other friendships which we are uncertain about (i.e., uncertain friendships). 

Fourth, previous work assumes that seed nodes of our choice can be influenced deterministically, which is also an unrealistic assumption. In reality, some of our chosen seed nodes (e.g., homeless youth) may be unwilling to spread influence to their peers, so one needs to explicitly consider a situation when the influencers cannot be influenced.

Finally, despite two decades of research in influence maximization algorithms, none of these previous algorithms and models have ever been tested in the real world (atleast with low-resource communities). This leads us to a natural question: Are these sophisticated AI algorithms actually needed in the real-world? Can one get near-optimal empirical performance from simple heuristics instead? Finally, the usability of these algorithms is also unknown in the real-world.


\section{Motivating Domain}
\label{domain}
This thesis attempts to resolve these limitations by developing fundamental algorithms for influence maximization which can handle these uncertainties and constraints in a principled manner. While these algorithms are not domain specific, and can easily be applied to other domains (e.g., preventing drug addiction, raising awareness about governance related grievances of low-literate farmers, etc.), this thesis uses an important domain for motivation, where influence maximization could be used for social good: raising awareness about Human Immunodeficiency Virus (HIV) among homeless youth.

HIV-AIDS is a dangerous disease which claims 1.5 million lives annually \cite{unaids}, and homeless youth are particularly vulnerable to HIV due to their involvement in high risk behavior such as unprotected sex, sharing drug needles, etc. \cite{nchc}. To prevent the spread of HIV, many homeless shelters conduct intervention camps, where a select group of homeless youth are trained as ``peer leaders" to lead their peers towards safer practices and behaviors, by giving them information about safe HIV prevention and treatment practices. These peer leaders are then tasked with spreading this information among people in their social circle. 

\begin{figure}[htp]
\subfloat[Homeless youth at My Friend's Place]{\includegraphics[height=1.5in,width=0.48\columnwidth]{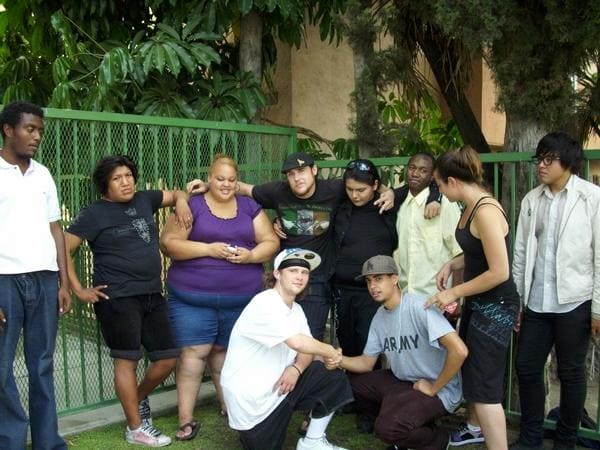}\label{fig:shelterpoe}}
\hspace{2mm}
\subfloat[My Friend's Place Team with a Co-Author]{\includegraphics[height=1.5in,width=0.48\columnwidth]{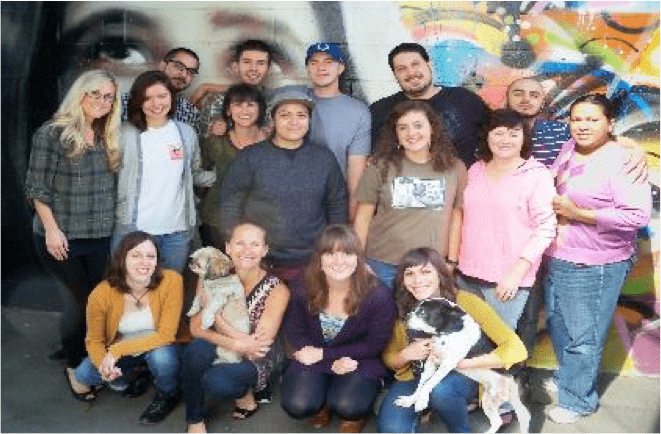}\label{fig:shelterloi}}
\caption{One of the Homeless Shelters Where We Conducted Deployments of our Algorithms}
\end{figure}

However, due to financial/manpower constraints, the shelters can only organize a limited number of intervention camps. Moreover, in each camp, the shelters can only manage small groups of youth ($\sim$3-4) at a time (as emotional and behavioral problems of youth makes management of bigger groups difficult). Thus, the shelters prefer a series of small sized camps organized sequentially \cite{rice2012}. As a result, the shelter cannot intervene on the entire target (homeless youth) population. Instead, it tries to maximize the spread of awareness among the target population (via word-of-mouth influence) using the limited resources at its disposal. To achieve this goal, the shelter uses the friendship based social network of the target population to strategically choose the participants of their limited intervention camps. Unfortunately, the shelters' job is further complicated by a lack of complete knowledge about the social network's structure \cite{rice2010positive}. Some friendships in the network are known with certainty whereas there is uncertainty about other friendships. 

Thus, the shelters face an important challenge: they need a sequential plan to choose the participants of their sequentially organized interventions. This plan must address three key points: (i) it must deal with network structure uncertainty;  (ii) it needs to take into account new information uncovered during the interventions, which reduces the uncertainty in our understanding of the network; and (iv) the intervention approach should address the challenge of gathering information about social networks of homeless youth, which usually costs thousands of dollars and many months of time \cite{rice2012}. 

\section{Main Contributions}
\label{intro.sec:motivation}
This thesis focuses on providing innovative techniques and significant advances for addressing the challenges of uncertainties in influence maximization problems, including 1) uncertainty in the social network structure; 2) uncertainty about the state of the network; and 3) uncertainty about willingness of nodes to be influenced. Some key research contributions include:
\begin{itemize}
\item new influence maximization algorithms for homeless youth service providers based on Partially Observable Markov Decision Process (or POMDP) planning.
\item real-world evaluation of these algorithms with 173 actual homeless youth across two different homeless shelters in Los Angeles. 
\end{itemize}

\subsection{Influence Maximization Under Real-World Constraints}
As the \textit{first contribution}, we use the homeless youth domain to motivate the definition of the Dynamic Influence Maximization Under Uncertainty (DIME) problem \cite{yadav2016using}, which models the aforementioned challenge faced by the homeless youth service providers accurately. Infact, the sequential selection of network nodes in multiple stages in DIME sets it apart from any other previous work in influence maximization \cite{leskovec2007cost,kimura2006tractable,chen2010scalable,cohen2014sketch}. As the \textit{second contribution}, we introduce a novel Partially Observable Markov Decision Process (POMDP) based model for solving DIME, which takes into account uncertainties in network structure and evolving network state . As the \textit{third contribution}, since conventional POMDP solvers fail to scale up to sizes of interest (our POMDP had $2^{300}$ states and ${150 \choose 6}$ actions), we design three scalable (and more importantly, ``\textit{deployable}") algorithms, which use our POMDP model to solve the DIME problem. 

Our first algorithm PSINET \cite{yadav2015preventing} relies on the following key ideas: (i) compact representation of transition probabilities to manage the intractable state and action spaces; (ii) combination of the QMDP heuristic with Monte-Carlo simulations to avoid exhaustive search of the entire belief space; and (iii) voting on multiple POMDP solutions, each of which efficiently searches a portion of the solution space to improve accuracy. Unfortunately, even though PSINET was able to scale up to real-world sized networks, it completely failed at scaling up in the number of nodes that get picked in every round (intervention). To address this challenge, we designed HEAL, our second algorithm. HEAL \cite{yadav2016using} hierarchically subdivides our \textit{original POMDP} at two layers: (i) In the top layer, graph partitioning techniques are used to divide the \textit{original POMDP} into \textit{intermediate POMDPs}; (ii) In the second level, each of these \textit{intermediate POMDPs} is further simplified by sampling uncertainties in network structure repeatedly to get \textit{sampled POMDPs}; (iii) Finally, we use aggregation techniques to combine the solutions to these simpler POMDPs, in order to generate the overall solution for the \textit{original POMDP}. Finally, unlike PSINET and HEALER, our third algorithm CAIMS \cite{yadav2018} explicitly models uncertainty in availability (or willingness) of network nodes to get influenced, and relies on the following key ideas: (i) action factorization in POMDPs to scale up to real-world network sizes; and (ii) utilization of Markov nets to represent the exponentially sized belief state in a compact and lossless manner.

\subsection{Real World Evaluation of Influence Maximization Algorithms}
For real-world evaluation, we deployed our influence maximization algorithms in the field (with homeless youth) to provide a head-to-head comparison of different influence maximization algorithms \cite{yadav2017influence}. Incidentally, these turned out to be the first such deployments in the real-world. We collaborated with Safe Place for Youth\footnote{\label{ftnote1}http://safeplaceforyouth.nationbuilder.com/} and My Friends Place \footnote{\label{ftnote2}http://myfriendsplace.org/}(two homeless youth service providers in Los Angeles) to conduct three different pilot studies with 173 homeless youth in these centers. These deployments helped in establishing the superiority of my AI based algorithms (HEALER and DOSIM), which significantly outperformed Degree Centrality (the current modus operandi at drop-in centers for selecting influential seed nodes) in terms of both spread of awareness and adoption of safer behaviors. Specifically, HEALER and DOSIM outperformed Degree Centrality (the current modus operandi) by $\sim$160\% in terms of information spread among homeless youth in the real-world. These highly encouraging results are starting to lead to a change in standard operating practices at drop-in centers as they have begun to discard their previous approaches of spreading awareness in favor of our AI based algorithms. More importantly, it illustrates one way (among many others) in which Artificial Intelligence techniques can be harnessed for benefiting low-resource communities such as the homeless youth.

\section{Overview of Thesis}
This thesis is organized as follows. Chapter \ref{chapter:relatedwork} provides an overview of the related work in this area. Chapter \ref{chapter:background} introduces fundamental background material necessary to understand the research presented in the thesis.  Next, in Chapter \ref{chapter:DIME}, we present a mathematical formulation of the Dynamic Influence Maximization under Uncertainty (DIME) problem (which is the problem faced by homeless youth service providers) and provides a characterization of its theoretical complexity. Chapter \ref{chapter:POMDP} then introduces the Partially Observable Markov Decision Process (POMDP) model for DIME. Chapter \ref{chapter:PSINET} explains the first PSINET algorithm which utilizes the QMDP heuristic to solve the DIME problem. Next, Chapter \ref{chapter:HEALER} introduces the HEALER algorithm which relies on a hierarchical ensembling heuristic approach to scale up to larger instances of the DIME problem. For real-world evaluation of the algorithms, Chapter \ref{chapter:PILOT} presents results from the real-world pilot studies that we conducted. Chapter \ref{chapter:CAIMS} introduces the CAIMS algorithm which handles uncertainty in availability of nodes to get influenced. Finally, chapter \ref{chapter:conclusion} concludes the thesis and presents possible future directions.

\chapter{Related Work}
\label{chapter:relatedwork}

\section{Social Work Research in Peer-Led Interventions}
Given the important role that peers play in the HIV risk behaviors of homeless youth \cite{rice2012position,green2013shared}, it has been suggested in social work research that \textit{peer leader based interventions} for HIV prevention be developed for these youth \cite{arnold2009comparisons,rice2012position,green2013shared}. These interventions are desirable for homeless youth (who have minimal health care access, and are distrustful of adults), as they take advantage of existing relationships \cite{rice2013should}. These interventions are also successful in focusing limited resources to select influential homeless youth in different portions of large social networks \cite{arnold2009comparisons,medley2009effectiveness}.  However, there are still open questions about ``\textit{correct}" ways to select peer leaders in these interventions, who would maximize awareness spread in these networks.

Unfortunately, very little previous work in the area of real-world implementation of influence maximization has used AI or algorithmic approaches for peer leader selection, despite the scale and uncertainty in these networks; instead relying on convenience selection or simple centrality measures. Kelly et. al. \cite{kelly1997randomised} identify peer leaders based on personal traits of individuals, irrespective of their structural position in the social network.  Moreover, selection of the most popular youth (i.e., Degree Centrality based selection) is the most popular heuristic for selecting peer leaders \cite{valente2012network}. However, as we show later, Degree Centrality is ineffective for \textit{peer-leader based interventions}, as it only selects peer leaders from a particular area of the network, while ignoring other areas. 

\section{Influence Maximization}
On the other hand, a significant amount of research has occured in Computer Science in the field of computational influence maximization, which has led to the development of several algorithms for selecting ``seed nodes" in social networks. The influence maximization problem, as stated by Kempe et. al. \cite{kempe2003maximizing}, takes in a social network as input (in the form of a graph), and outputs a set of $K$ `\textit{seed nodes}' which maximize the expected influence spread in the social network within $T$ time steps. Note that the expectation of influence spread is taken with respect to a probabilistic influence model, which is also provided as input to the problem.  

\subsection{Standard Influence Maximization}
There are many algorithms for finding `\textit{seed sets}' of nodes to maximize influence spread in networks \cite{kempe2003maximizing,leskovec2007cost,Borgs14,tang2014influence,lerman2016majority,ghosh2009leaders,ghosh2010community,ver2013information,galstyan2009maximizing,galstyan2008influence}. However, all these algorithms assume \textit{no uncertainty in the network structure} and select a single seed set. In contrast, we select several seed sets sequentially in our work to select intervention participants, as that is a natural requirement arising from our homeless youth domain. Also, our work takes into account uncertainty about the network structure and influence status of network nodes (i.e., whether a node is influenced or not). Finally, unlike most previous work \cite{kempe2003maximizing,leskovec2007cost,Borgs14,tang2014influence,lerman2016majority,ghosh2009leaders,ghosh2010community,ver2013information,galstyan2009maximizing,galstyan2008influence}, we use a different influence model as we explain later. 

There is another line of work by Golovin et. al. \cite{golovin2011adaptive}, which introduces adaptive submodularity and discusses adaptive sequential selection (similar to our problem). They prove that a Greedy algorithm provides a $(1-1/e)$ approximation guarantee. However, unlike our work, they assume no uncertainty in network structure. Also, while our problem can be cast into the adaptive stochastic optimization framework of \cite{golovin2011adaptive}, our influence function is not adaptive submodular (as shown later), because of which their Greedy algorithm loses its approximation guarantees. 

Recently, after the development of the algorithms in this thesis, some other algorithms have also been proposed in the literature to solve similar influence maximization problems in the homeless youth domain. For example, \cite{wilder122018maximizing} proposes the ARISEN algorithm which deals with situations where you do not know anything about the social network of homeless youth at all, and it proposes a policy which trades off network mapping with actual influence spread in the social network. However, ARISEN was found to be difficult to implement in practice, and as a result, \cite{wilder2018end} proposes the CHANGE agent which utilizes insights from the friendship paradox \cite{feld1991your} to learn about the most promising parts of the social network as quickly as possible. Moreover, based on results from the real-world pilot studies detailed in this thesis, \cite{lilyhu} proposes new diffusion models for real-world networks which fit empirical diffusion patterns observed in the pilot studies much more convincingly.

An orthogonal line of work is \cite{singer2012win} which solves the following problem: how to incentivize people in order to be influencers? Unlike us, they solve a mechanism-design problem where nodes have private costs, which need to be paid for them to be influencers. However, in our domains of interest, while there is a lot of uncertainty about which nodes can be influenced in the network, monetary gains/losses are not the reason behind nodes getting influenced or not. Instead, nodes do not get influenced because they are either not available or willing to get influenced. 

In another orthogonal line of work, \cite{yadav-explain} proposed XplainIM, a machine learning based explanation system  to explain the solutions of HEALER \cite{yadav2016using} to human subjects. The problem of explaining solutions of influence maximization algorithms was first posed in\cite{yadav-ideas}, but their main focus was on justifying the need for such explanations, as opposed to providing any practical solutions to this problem. Thus, XplainIM represents the first step taken towards solving this problem (of explaining influence maximization solutions). Essentially, they propose using a Pareto frontier of decision trees as their interpretable classifier in order to explain the solutions of HEALER. 

\subsection{Competitive Influence Maximization}
Yet another field of related work involves two (or more) players trying to spread their own `competing' influence in the network (broadly called influence blocking maximization, or IBM). Some research exists on IBM where all players try to maximize their own influence spread in the network, instead of limiting others \cite{bharathi2007competitive,kostka2008word,borodin2010threshold}. \cite{Tsai12a} try to model IBM as a game theoretic problem and provide scale up techniques to solve large games. Just like our work, \cite{tsai2013bayesian} consider uncertainty in network structure. However, \cite{tsai2013bayesian} do not consider sequential planning (which is essential in our domain) and thus, their methods are not reusable in our domain.

\section{POMDP Planning}
The final field of related work is planning for reward/cost optimization. In POMDP literature, a lot of work has happened along two different paradigms: \textit{offline and online POMDP planning}. 

\subsection{Offline POMDP Planning}
In the paradigm of offline POMDP planning, algorithms are desired which precompute the entire POMDP policy (i.e., a mapping from every possible belief state to the optimal action for that belief) ahead of time, i.e., before execution of the policy begins. In 1973,  \cite{smallwood1973optimal} proposed a dynamic programming based algorithm for optimally solving a POMDP. Improving upon this, a number of exact algorithms leveraging the piecewise-linear and convex aspects of the POMDP value function have been proposed in the POMDP literature \cite{monahan1982state,littman1996algorithms,cassandra1997incremental,zhang2001speeding}. Recently, several approximate offline POMDP algorithms have also been proposed \cite{hauskrecht2000value,pineau2006anytime}. Some notable offline planners include GAPMIN \cite{poupart2011closing} and Symbolic Perseus \cite{spaan2005perseus}. Currently, the leading offline POMDP solver is SARSOP \cite{kurniawati2008sarsop}. Unfortunately, all of these offline POMDP methods fail to scale up to any realistic problem sizes, which makes them difficult to use for real-world problems.

\subsection{Online POMDP Planning}
In the paradigm of online POMDP planning, instead of computing the entire POMDP policy, only the best action for the current belief state is found. Upon reaching a new belief state, online planning again plans for this new belief. Thus, online planning interleaves planning and execution at every time step. Recently, it has been suggested that online planners are able to scale up better \cite{paquet2005online}, and therefore we focus on online POMDP planners in this thesis.  For online planning, we mainly focus on the literature on Monte-Carlo (MC) sampling based online POMDP solvers since this approach allows significant scale-ups. \cite{silver2010monte} proposed the \textit{Partially Observable Monte Carlo Planning} (POMCP) algorithm that uses Monte-Carlo tree search in online planning. Also, \cite{somani2013despot} present the DESPOT algorithm, that improves the worst case performance of POMCP. \cite{bai2014thompson} used Thompson sampling to intelligently trade-off between exploration and exploitation in their D\textsuperscript{2}NG-POMCP algorithm. These algorithms maintain a search tree for all sampled histories to find the best actions, which may lead to better solution qualities, but it makes these techniques less scalable (as we show in our experiments). Therefore, our algorithm does not maintain a search tree and uses ideas from {\it Q}\textsubscript{\it MDP} heuristic \cite{littman1995learning} and hierarchical ensembling to find best actions. Yet another related work is FV-POMCP \cite{amato2015scalable,Katt17ICML}, which was proposed to handle issues with POMCP's \cite{silver2010monte} scalability. Essentially, FV-POMCP relies on a factorized action space to scale up to larger problems. In our work, we complement their advances to build CAIMS, which leverages insights from social network theory to factorize action spaces in a provably ``\textit{lossless}" manner, and to represent beliefs in an accurate manner using Markov networks.

\chapter{Background}
\label{chapter:background}
In this chapter, we provide general background information about influence maximization problems and how we represent social networks inside our influence maximization algorithms. Next, we discuss well-known diffusion spread models (along with the model that we use) used inside influence maximization. We also describe the well-known Greedy algorithm for influence maximization. Finally, we describe background information about the POMDP model and a well-known algorithm for solving POMDPs.

\section{Influence Maximization Problem}
Given a social network $G$ and a parameter $K$, the influence maximization problem asks to find an optimal $K$ sized set of nodes of maximum influence in the social network. In other words, given a social network $G$ and an influence model $M$ of a diffusion process that take place on network $G$, the goal is to find $K$ initial seeders in the network who will lead to most number of people receiving the message. More formally, for any $K$ sized subset of nodes $A$, let $\delta(A)$ denote the expected number of individuals in the network who will receive the message, given that $A$ is the initial set of seeders. Then, the influence maximization problem takes as \textit{input} (i) the social network $G$, (ii) the influence model $M$, and (iii) the number of nodes to choose $K$, and produces as \textit{output} an optimal subset of nodes $S=\argmax_A \delta(A)$.

\section{Network Representation}
In influence maximization problems, we represent social networks $G=(V,E)$ as directed graphs (consisting of \textit{nodes} and \textit{directed edges}) where each \textit{node} represents a person in the social network and a \textit{directed edge} between two nodes $A$ and $B$ (say) represents that node $A$ \textit{considers} node $B$ as his/her friend. \textit{We assume directed-ness of edges as sometimes homeless shelters assess that the influence in a friendship is very much uni-directional; and to account for uni-directional follower links on Facebook}. Otherwise friendships are encoded as two uni-directional links.
 
\subsection{Uncertain Network} The uncertain network is a directed graph $G=(V,E)$  with $|V| = N$ nodes and $|E| = M$ edges. The edge set $E$ consists of two disjoint subsets of edges: $E_c$ (the set of certain edges, i.e., friendships which we are certain about) and $E_u$ (the set of uncertain edges, i.e., friendships which we are uncertain about). Note that uncertainties about friendships exist because HEALER's Facebook application misses out on some links between people who are friends in real life, but not on Facebook.

\begin{figure}[h]
\center{\includegraphics[scale=.55]
{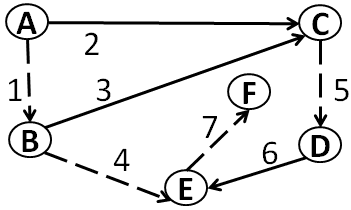}}
\caption{\label{fig:backgrnd_uncertainG} Uncertain Network}
\end{figure} 

To model the uncertainty about missing edges, every uncertain edge $e \in E_u$ has an existence probability $u(e)$ associated with it, which represents the likelihood of ``existence" of that uncertain edge. For example, if there is an uncertain edge $(A,B)$ (i.e., we are unsure whether node $B$ is node $A$'s friend), then $u(A,B) = 0.75$ implies that $B$ is $A$'s friend with a 0.75 chance. In addition, each edge $e \in E$ (both certain and uncertain) has a propagation probability $p(e)$ associated with it. A propagation probability of 0.5 on directed edge $(A,B)$ denotes that if node $A$ is influenced (i.e., has information about HIV prevention), it influences node $B$ (i.e., gives information to node $B$) with a 0.5 probability in each subsequent time step. This graph $G$ with all relevant $p(e)$ and $u(e)$ values represents an uncertain network and serves as an input to the DIME problem. Figure \ref{fig:backgrnd_uncertainG} shows an uncertain network on 6 nodes (\textit{A} to \textit{F}) and 7 edges. The dashed and solid edges represent uncertain (edge numbers 1, 4, 5 and 7) and certain (edge numbers 2, 3 and 6) edges, respectively.

\section{Influence Model}
In previous work, different kinds of influence spread models have been proposed and used. We now discuss some of the well-known models and then describe the influence model that is used in this thesis.

\subsection{Independent Cascade Model}
The independent cascade model \cite{kempe2003maximizing} associates a propagation probability $p(e)$ to each edge $e \in E$ of the social network. This propagation probability $p(e)$ denotes the likelihood with which influence spreads along edge $e$ in the network. The influence spread process begins with an initial set of activated (or \textit{influenced}) nodes called ``\textit{seed nodes}" $A_0$ and then proceeds in a series of discrete time-steps $t \in [1,T]$. At each time step $t$, every node that was influenced at time step $t-1$ tries to influence their un-influenced neighbors (and they do so according to the propagation probabilities on the respective edges). This process keeps on repeating until either $T$ time steps are reached or the entire network is influenced.
 
\subsection{Linear Threshold Model}
The linear threshold model \cite{kempe2003maximizing} associates a weight $w_e$ on each edge $e \in E$ of the social network. Further, each node $v \in V$ has a threshold $\epsilon_v \in [0,1]$. This threshold represents the fraction of neighbors of $v$ that must become influenced in order for node $v$ to become influenced. Again, the influence spread process begins with an initial set of ``\textit{seed nodes}" $A_0$ and then proceeds in a series of discrete time-steps $t \in [1,T]$. At each time step $t$, each un-influenced node which satisfies the following condition becomes influenced: $\sum\limits_{e \sim v} w_e \geqslant \epsilon_v$.  This process keeps on repeating until either $T$ time steps are reached or the entire network is influenced.

\subsection{Our Influence Model}
We use a variant of the independent cascade model \cite{yan2011influence}. In the standard independent cascade model, all nodes that get influenced at round $t$ get a \textbf{single} chance to influence their un-influenced neighbors at time $t+1$. If they fail to spread influence in this \textbf{single} chance, they don't spread influence to their neighbors in future rounds. Our model is different in that we assume that nodes get \textbf{multiple} chances to influence their un-influenced neighbors. If they succeed in influencing a neighbor at a given time step $t'$, they stop influencing that neighbor for all future time steps. Otherwise, if they fail in step $t'$, they try to influence again in the next round. This variant of independent cascade has been shown to empirically provide a better approximation to real influence spread than the standard independent cascade model \cite{cointet2007,yan2011influence}. Further, we assume that nodes that get influenced at a certain time step remain influenced for all future time steps. 

\begin{algorithm}[h]
\label{alg:greedyalgo}
\caption{Greedy Algorithm}
\KwIn{Graph $G$, Influence Model $M$, Number of Nodes $K$}
\KwOut{Best Action $A$}
$A \leftarrow \phi$ ; \\\label{greflow:1}
\For {$i \in [1,K]$} {
	$Pick\mbox{ }v:\mbox{ }\delta(A \cup \{v\}) - \delta(A) \mbox{ is maximized}$;\\\label{greflow:2}
	$A \leftarrow A	\cup \{v\}$;\\\label{greflow:3}
	}
	return $A$;\\
\end{algorithm}

\section{Greedy Algorithm for Influence Maximization}
In order to solve influence maximization problems, there exists a well-known approximation algorithm (called the Greedy algorithm), which was first proposed by \cite{kempe2003maximizing} in the context of influence maximization. Algorithm 1 shows the overall flow of this algorithm, which iteratively builds the set of $K$ nodes that should be output for the influence maximization problem. In each iteration, the node which increases the marginal gain (in the expected solution value) by the maximum amount is added (Steps \ref{greflow:2} and \ref{greflow:3}) to the output set. Finally, after $K$ iterations, the set of nodes is returned. It is well known that if the function $\delta(A)$ can be shown to be submodular, then this Greedy algorithm outputs a $(1-1/e)$ approximation guarantee \cite{kempe2003maximizing}. Unfortunately, we later show that for our problem, this Greedy algorithm does not have any guarantees due to submodularity. As a result, we use POMDPs to solve our problem. Here, we provide a high-level overview of the POMDP model, and in later sections, we will describe how our POMDP algorithms work.

\section{Partially Observable Markov Decision Processes}
Partially Observable Markov Decision Processes (POMDPs) are a well studied model for sequential decision making under uncertainty \cite{puterman2009markov}. Intuitively, POMDPs model situations wherein an agent tries to maximize its expected long term \textit{rewards} by taking various \textit{actions}, while operating in an environment (which could exist in one of several \textit{states} at any given point in time) which reveals itself in the form of various \textit{observations}. The key point is that the exact state of the world is not known to the agent and thus, these actions have to be chosen by reasoning about the agent's probabilistic beliefs (belief state). The agent, thus, takes an action (based on its current belief), and the environment transitions to a new world state. However, information about this new world state is only partially revealed to the agent through observations that it gets upon reaching the new world state. Hence, based on the agent's current belief state, the action that it took in that belief state, and the observation that it received, the agent updates its belief state. The entire process repeats several times until the environment reaches a terminal state (according to the agent's belief).

More formally, a full description of the POMDP includes the sets of possible environment \textit{states}, the set of \textit{actions} that the agent can take, and the set of possible \textit{observations} that the agent can observe. In addition, the full POMDP description includes a \textit{transition matrix}, for storing \textit{transition probabilities}, which specify the probability with which the environment transitions from one \textit{state} to another, conditioned on the immediate \textit{action} taken. Another component of the POMDP description is the \textit{observation matrix}, for storing \textit{observation probabilities}, which specify the probability of getting different \textit{observations} in different \textit{states}, conditioned on the \textit{action} taken to reach that \textit{state}. Finally, the POMDP description includes a \textit{reward} matrix, which specifies the agent's \textit{reward} of taking \textit{actions} in different \textit{states}.	

A POMDP policy $\Pi$ provides a mapping from every possible belief state (which is a probability distribution over world states) to an action $a=\Pi(\beta)$. Our aim is to find an optimal policy $\Pi^*$ which, given an initial belief $\beta_0$, maximizes the expected cumulative long term reward over H horizons (where the agent takes an action and gets a reward in each time step until the horizon H is reached). Computing optimal policies offline for finite horizon POMDPs is PSPACE-Complete. Thus, focus has recently turned towards online algorithms, which only find the best action for the current belief state \cite{paquet2005online,silver2010monte}. In order to illustrate some solution methods for POMDPs, we now provide a high-level overview of POMCP \cite{silver2010monte}, a highly popular online POMDP algorithm. 

\begin{figure}[h]
\center{\includegraphics[scale=.4]
{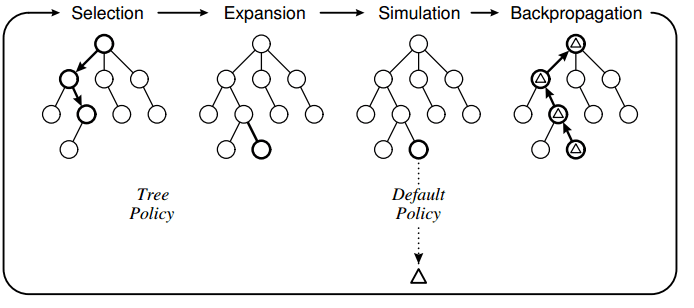}}
\caption{\label{fig:mctsoutline} UCT Tree Generation In POMCP}
\end{figure} 

\section{POMCP: An Online POMDP Planner}
POMCP \cite{silver2010monte} uses UCT based Monte-Carlo tree search (MCTS) \cite{browne2012survey} to solve POMDPs. At every stage, given the current belief state $b$, POMCP incrementally builds a UCT tree (as in Figure \ref{fig:mctsoutline}) that contains statistics that serve as empirical estimators (via MC samples) for the POMDP Q-value function $Q(b,a) = R(b,a) + \sum\limits_z P(z|b,a) max_{a'} Q(b',a')$. The algorithm avoids expensive belief updates by maintaining the belief at each UCT tree node as an unweighted particle filter (i.e., a collection of all states that were reached at that UCT tree node via MC samples). In each MC simulation, POMCP samples a start state from the belief at the root node of the UCT tree, and then samples a trajectory that first traverses the partially built UCT tree, adds a node to this tree if the end of the tree is reached before the desired horizon, and then performs a random rollout to get one MC sample estimate of $Q(b,a)$. Finally, this MC sample estimate of $Q(b,a)$ is propagated up the UCT tree to update Q-value statistics at nodes that were visited during this trajectory. Note that the UCT tree grows exponentially large with increasing state and action spaces. Thus, the search is directed to more promising areas of the search space by selecting actions at each tree node $h$ according to the UCB1 rule \cite{kocsis2006bandit}, which is given by: $a = argmax_a \hat{Q}(b_h, a) + c \sqrt{log(N_h+1)/n_{ha}}$. Here, $\hat{Q}(b_h,a)$ represents the Q-value statistic (estimate) that is maintained at node $h$ in the UCT tree. Also, $N_h$ is the number of times node $h$ is visited, and $n_{ha}$ is the number of times action $a$ has been chosen at tree node $h$ (POMCP maintains statistics for $N_h$ and $n_{ha} \forall a \in A$ at each tree node $h$). While POMCP handles large state spaces (using MC belief updates), it is unable to scale up to large action sizes (as the branching factor of the UCT tree blows up).

\chapter{Dynamic Influence Maximization Under Uncertainty}
\label{chapter:DIME}
In this chapter, we formally define the Dynamic Influence Maximization Under Uncertainty (or DIME) problem, which models the problems faced by homeless youth service providers, and is the primary focus of attention in this thesis. We also characterize the theoretical complexity of the DIME problem in this chapter.

\section{Problem Definition}
Given the \textit{uncertain network} as input, we plan to run for \textit{$T$ rounds} (corresponding to the number of interventions organized by the homeless shelter).
In each round, we will choose \textit{$K$ nodes} (youth) as intervention participants. These participants are assumed to be influenced post-intervention with certainty. Upon influencing the chosen nodes, we will `\textit{observe}' the true state of the \textit{uncertain edges} (friendships) out-going from the selected nodes. This translates to asking intervention participants about their 1-hop social circles, which is within the homeless shelter's capabilities \cite{rice2012position}. 

After each round, influence spreads in the network according to our influence model for \textit{$L$ time steps}, before we begin the next round. This $L$ represents the time duration in between two successive intervention camps. \textit{In between rounds, we do not observe the nodes that get influenced during $L$ time steps}. We only know that explicitly chosen nodes (our intervention participants in all past rounds) are influenced. Informally then, given an uncertain network $G_0=(V, E)$ and integers $T$, $K$, and $L$ (as defined above), our goal is to find an online policy for choosing \textit{exactly} $K$ nodes for $T$ successive rounds (interventions) which maximizes influence spread in the network at the end of $T$ rounds.

We now provide notation for defining an online policy formally. Let $\bm{\mathcal{A}}=\{A \subset V \mbox{ s.t. } |A|=K \}$ denote the set of $K$ sized subsets of $V$, which represents the set of possible choices that we can make at every time step $t \in [1,T]$. Let $A_i \in \bm{\mathcal{A}} \mbox{ } \forall i \in [1,T]$ denote our choice in the $i^{th}$ time step. Upon making choice $A_i$, we `\textit{observe}' uncertain edges adjacent to nodes in $A_i$, which updates its understanding of the network. Let $G_i \mbox{ } \forall \mbox{ } i \in [1,T]$ denote the uncertain network resulting from $G_{i-1}$ with \textit{observed} (additional edge) information from $A_i$. Formally, we define a history $H_i \mbox{ } \forall \mbox{ } i \in [1,T]$ of length $i$ as a tuple of past choices and observations $H_i = \tuple{G_0, A_1, G_1, A_2,..,A_{i-1},G_i}$. Denote by $\bm{\mathcal{H}_i} = \{ H_k \mbox{ s.t. } k \leqslant i \}$ the set of all possible histories of length less than or equal to $i$. Finally, we define an $i$-step policy $\bm{\Pi_i} \colon \bm{\mathcal{H}_i} \to \bm{\mathcal{A}}$ as a function that takes in histories of length less than or equal to $i$ and outputs a $K$ node choice for the current time step. We now provide an explicit problem statement for DIME.

\begin{problem}{\textbf{DIME Problem}}
Given as input an uncertain network $G_0=(V, E)$ and integers $T$, $K$, and $L$ (as defined above). Denote by $\mathcal{R}(H_T, A_T)$ the \textit{expected total number of influenced nodes at the end of round $T$}, given the $T$-length history of previous observations and actions $H_T$, along with $A_T$, the action chosen at time $T$. Let $\E_{H_T,A_T \sim \Pi_T} [\mathcal{R}(H_T,A_T)]$ denote the expectation over the random variables $H_T=\tuple{G_0, A_1,..,A_{T-1},G_T}$ and $A_T$, where $A_i$ are chosen according to $\Pi_T(H_i)  \mbox{ } \forall \mbox{ } i \in [1,T]$, and $G_i$ are drawn according to the distribution over uncertain edges of $G_{i-1}$ that are revealed by $A_i$. The objective of DIME is to find an optimal $T$-step policy $\bm{\Pi_T^*} = \argmax_{\Pi_T} \E_{H_T,A_T \sim \Pi_T}[\mathcal{R}(H_T, A_T)]$. 
\end{problem} 

\section{Characterization of Theoretical Complexity}
Next, we show hardness results about the DIME problem. First, we analyze the value of having complete information in DIME. Then, we characterize the computational hardness of DIME.

\subsection{The Value of Information} We characterize the impact of insufficient information (about the uncertain edges) on the achieved solution value. We show that no algorithm for DIME is able to provide a good approximation to the \textit{full-information solution value} (i.e., the best solution achieved w.r.t. the underlying ground-truth network), even with infinite computational power. 

\begin{figure}[htb]
\center{\includegraphics[scale=.5]
{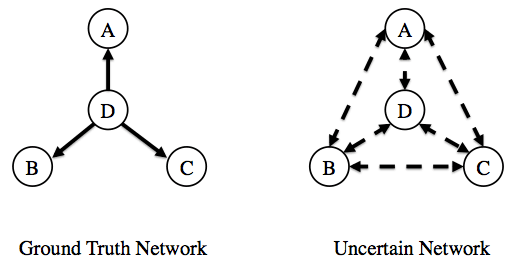}}
\caption{\label{fig:Figure1} Counter-example for Theorem 1}
\end{figure}

\begin{theorem}\label{Th:4}
Given an uncertain network with $n$ nodes, for any $\epsilon > 0$, there is no algorithm for the DIME problem which can guarantee a $n^{-1+\epsilon}$ approximation to $OPT_{full}$, the \textit{full-information solution value}. 
\end{theorem}
\begin{proof}[Sketch]
We prove this statement by providing a counter-example in the form of a specific (ground truth) network for which there can exist no algorithm which can guarantee a $n^{-1+\epsilon}$ approximation to $OPT_{full}$. Consider an input to the DIME problem, an \textit{uncertain network} with $n$ nodes with $2 * {n \choose 2}$ uncertain edges between the $n$ nodes, i.e., it's a completely connected uncertain network consisting of \textit{only} uncertain edges (an example with $n=3$ is shown in Figure \ref{fig:Figure1}). Let $p(e)=1$ and $u(e)=0.5$ on all edges in the \textit{uncertain network}, i.e., all edges have the same propagation and existence probability. Let $K=1$, $L=1$ and $T=1$, i.e., we just select a single node in one shot (in a single round). 

Further, consider a star graph (as the ground truth network) with $n$ nodes such that propagation probability $p(e) = 1$ on all edges of the star graph (shown in Figure 1). Now, any algorithm for the DIME problem would select a single node in the \textit{uncertain network} uniformly at random with equal probability of $1/n$ (as information about all nodes is symmetrical). In expectation, the algorithm will achieve an expected reward  $\{1/n \times (n)\} + \{1/n \times (1) + ... + 1/n \times (1)\} = 1/n \times(n) + (n-1)/n \times 1 = 2 - 1/n$. However, given the ground truth network, we get $OPT_{full}=n$, because we always select the star node. As $n$ goes to infinity, we can at best achieve a $n^{-1}$ approximation to $OPT_{full}$. Thus, no algorithm can achieve a $n^{-1+\epsilon}$ approximation to $OPT_{full}$ for any $\epsilon > 0$.   
\end{proof}

\subsection{Computational Hardness}
We now analyze the hardness of computation in the DIME problem in the next two theorems.

\begin{theorem}\label{Th:1}
The DIME problem is NP-Hard.
\end{theorem}
\begin{proof}[Sketch]
Consider the case where $E_u = \Phi$, $L=1$, $T=1$ and $p(e) = 1\mbox{} \forall e \in E$. This degenerates to the classical influence maximization problem which is known to be NP-hard. Thus, the DIME problem is also NP-hard.   
\end{proof}

Some NP-Hard problems exhibit nice properties that enable approximation guarantees for them. Golovin et. al. \cite{golovin2011adaptive} introduced adaptive submodularity, an analog of submodularity for adaptive settings. Presence of adaptive submodularity ensures that a simply greedy algorithm provides a $(1-1/e)$ approximation guarantee w.r.t. the optimal solution defined on the \textit{uncertain network}. However, as we show next, while DIME can be cast into the adaptive stochastic optimization framework of \cite{golovin2011adaptive}, our influence function is not adaptive submodular, because of which their Greedy algorithm does not have a $(1-1/e)$ approximation guarantee. 

\begin{figure}[htb]
\center{\includegraphics[scale=0.3]
{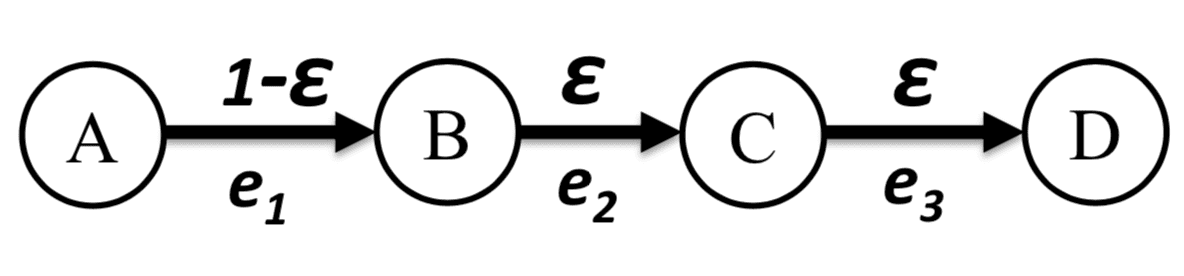}}
\caption{\label{fig:failure} Failure of Adaptive Submodularity}
\end{figure}

\begin{theorem}\label{Th:2}
The influence function of DIME is not adaptive submodular.
\end{theorem}
\begin{proof}
	The definition of adaptive submodularity requires that the expected marginal increase of influence by picking an additional node $v$ is more when we have less observation. Here the expectation is taken over the random states that are consistent with current observation. We show that this is not the case in DIME problem. Consider a path with $4$ nodes $a,b,c,d$ and three \emph{directed} edges $e_1 = (a,b)$ and $e_2 = (b,c)$ and $e_3 = (c,d)$ (see Figure \ref{fig:failure}). Let $p(e_1) = p(e_2) = p(e_3)=1$, i.e., propagation probability is $1$; $L=2$, i.e., influence stops after two round; and $u(e_1) =1-\epsilon$ $u(e_2) = u(e_3) = \epsilon$ for some small enough $\epsilon$ to be set. That is the only uncertainty comes from incomplete knowledge of the existence of edges. 
	
	Let $\Psi_1 = \{e_1 \text{ exists} \}$ and $\Psi_2 = \{e_1,e_3 \text{ exists} \}$. Then $\mathbb{E}_{\Phi}\left[f(a,b,c)|\Phi \sim \Psi_2 \right] = 4$ since all nodes will be influenced. $\mathbb{E}_{\Phi}\left[f(a,c)|\Phi \sim \Psi_2 \right] = 4-\epsilon$ since the only uncertain node is $b$ which will be influenced with probability $1-\epsilon$. Therefore,
	\begin{equation}\label{eq:withb}
		\mathbb{E}_{\Phi}\left[f(a,b,c)|\Phi \sim \Psi_2 \right] - \mathbb{E}_{\Phi}\left[f(a,c)|\Phi \sim \Psi_2 \right] = \epsilon.
	\end{equation}
		Now 
	 $\mathbb{E}_{\Phi}\left[f(a,b)|\Phi \sim \Psi_1 \right] = 2 + \epsilon + \epsilon^2$ since $a,b$ will be surely influenced, $c$ and $d$ will be influenced with probability $\epsilon$ and $\epsilon^2$ respectively. On the other hand, $\mathbb{E}_{\Phi}\left[f(a)|\Phi \sim \Psi_1 \right] = 2 + \epsilon $ since $b$ will be surely influenced (since $e_1$ exists) and $c$ will be influenced with probability $\epsilon$. Since $L=2$, $d$ cannot be influenced. As a result, 
	 	\begin{equation}\label{eq:nob}
	 	\mathbb{E}_{\Phi}\left[f(a,b)|\Phi \sim \Psi_2 \right] - \mathbb{E}_{\Phi}\left[f(a)|\Phi \sim \Psi_2 \right] = \epsilon^2.
	 	\end{equation}
	 	
	 	Combining Equation \eqref{eq:withb} and \eqref{eq:nob}, we know that DIME is not adaptive submodular.
\end{proof}

\chapter{POMDP Model for DIME Problem}
\label{chapter:POMDP}
The above theorems show that DIME is a hard problem as it is difficult to even obtain any reasonable approximations. We model DIME as a POMDP \cite{puterman2009markov} because of two reasons. First, POMDPs are a good fit for DIME as (i) we conduct several interventions sequentially, similar to sequential POMDP actions; and (ii) we have \textit{partial observability} (similar to POMDPs) due to uncertainties in network structure and influence status of nodes. Second, POMDP solvers have recently shown great promise in generating near-optimal policies efficiently \cite{silver2010monte}. We now explain how we map DIME onto a POMDP.

\section{POMDP States} A POMDP state in our problem is a pair of binary tuples $s = \tuple{W, F}$ where $W$ denotes the influence status of network nodes, i.e., $W_i = 1$ denotes that node $i$ is influenced and $W_i = 0$ denotes that node $i$ is not influenced. Similarly, $F$ denotes the existence of uncertain edges, where $F_i = 1$ denotes that the $i^{th}$ uncertain edge exists in reality, and $F_i = 0$ denotes that the $i^{th}$ uncertain edge does not exist in reality. We have an exponential state space, as in a social network with $N$ nodes and $M$ uncertain edges, the total number of possible states in our POMDP is $2^{N+M}$.

\section{POMDP Actions } Every choice of a subset of $K$ nodes in the social network is a possible POMDP action. More formally, $A = \{ a \subset V s.t. |a| = K\}$ represents the set of all valid actions in our POMDP. For example, in Figure \ref{fig:backgrnd_uncertainG}, one possible action is $\{A,B\}$ (when $K=2$). We have a combinatorial action space, as in a social network with $N$ nodes and the size of selected subset is $K$, the total number of possible actions in our POMDP is ${N \choose K}$.

\section{POMDP Observations } Upon taking a POMDP action, we ``\textit{observe}" the ground reality of the uncertain edges outgoing from the nodes chosen in that action. Consider $\Theta(a) = \{ \mbox{e }|\mbox{ e = (x,y) \text{s.t.} x} \in a \mbox{ } \wedge\mbox{ e} \in E_u \}\mbox{ }\forall a \in A$, which represents the (ordered) set of uncertain edges that are observed when we take POMDP action $a$. Then, our POMDP observation upon taking action $a$ is defined as $o(a) = \{F_{e} | e \in \Theta(a)\}$, i.e., the F-values (described in the POMDP state description) of the observed uncertain edges. For example, by taking action $\{B,C\}$ in Figure \ref{fig:backgrnd_uncertainG}, the values of $F_4$ and $F_5$ (i.e., the F-values of uncertain edges in the 1-hop social circle of nodes $B$ and $C$) would be observed. We have an exponential observation space, as the number of possible observations is exponential in the number of edges that are outgoing from the nodes selected in the action.

\section{POMDP Rewards } The reward $R(s,a,s')$ of taking action $a$ in state $s$ and reaching state $s'$ is the number of newly influenced nodes in $s'$. More formally, $R(s,a,s') = (\|s'\|-\|s\|)$, where $\|s'\|$ is the number of influenced nodes in $s'$. Over a time horizon, the long term reward of the POMDP equals the total number of nodes that are influenced in the social network (because of telescoping sum rule).

\section{POMDP Initial Belief State } The initial belief state is a distribution $\beta_0$ over all states $s \in S$. The support of $\beta_0$ consists of all states $s = \tuple{W, F}$ s.t. $W_i=0 \mbox{ } \forall \mbox{ } i \in [1,|V|]$, i.e., all states in which all network nodes are un-influenced (as we assume that all nodes are un-influenced to begin with). Inside its support, each $F_i$ is distributed independently according to $P(F_i=1)= u(e)$ (where $u(e)$ is the existence probability on edge $e$).

\section{POMDP Transition And Observation Probabilities } Computation of exact transition probabilities $T(s'|s,a)$ requires considering all possible paths in a graph through which influence could spread, which is $\mathcal{O}(N!)$ ($N$ is number of nodes in the network) in the worst case. Moreover, for large social networks, the size of the transition and observation probability matrix is prohibitively large (due to exponential sizes of state and action space). Therefore, instead of storing huge transition/observation matrices in memory, we follow the paradigm of large-scale online POMDP solvers \cite{silver2010monte,eck2015ask} by using a generative model $\Lambda(s, a) \sim (s', o, r)$ of the transition and observation probabilities. This generative model allows us to generate on-the-fly samples from the exact distributions $T(s'|s,a)$ and $\Omega(o|a,s')$ at very low computational costs. Given an initial state $s$ and an action $a$ to be taken, our generative model $\Lambda$ simulates the random process of influence spread to generate a random new state $s'$, an observation $o$ and the obtained reward $r$. Simulation of the random process of influence spread is done by ``\textit{playing}" out propagation probabilities (i.e., flipping weighted coins with probability $p(e)$) according to our influence model to generate sample $s'$. The observation sample $o$ is then determined from $s'$ and $a$. Finally, the reward sample $r = (\|s'\|-\|s\|)$ (as defined above). This simple design of the generative model allows significant scale and speed up (as seen in previous work \cite{silver2010monte} and also in our experiments).

This completes the discussion of our POMDP model for the DIME problem. Unfortunately, for real-world networks of homeless youth (which had $\sim$300 nodes), our POMDP model had $\sim2^450$ states and ${300 \choose 5}$ actions. Due to this huge explosion in the state and action spaces, current state-of-the-art offline and online POMDP solvers were unable to scale up to this problem. Initial experiments with the ZMDP solver \cite{zmdp} showed that state-of-the-art offline POMDP planners ran out of memory on networks having a mere 10 nodes. Thus, we focused on online planning algorithms and tried using POMCP \cite{silver2010monte}. Unfortunately, our experiments showed that even POMCP runs out of memory on networks having 30 nodes. This happens because POMCP \cite{silver2010monte} keeps the entire search tree over sampled histories in memory, disabling scale-up to the problems of interest in this paper. Hence, in the next couple of chapters, we will discuss two novel POMDP algorithms, i.e., PSINET \cite{yadav2015preventing} and HEALER \cite{yadav2016using}, that exploit problem structure to scale up to real-world nework sizes.



\chapter{PSINET}
\label{chapter:PSINET}
This chapter presents PSINET (or \textbf{P}OMDP based \textbf{S}ocial \textbf{I}nterventions in \textbf{N}etworks for \textbf{E}nhanced HIV \textbf{T}reatment), a novel Monte Carlo (MC) sampling online POMDP algorithm which addresses the shortcomings in POMCP \cite{silver2010monte}. At a high level, PSINET \cite{yadav2015preventing,yadav2016psinet} makes two significant advances over POMCP. First, it introduces a novel transition probability heuristic (by leveraging ideas from social network analysis) that allows storing the entire transition probability matrix in an extremely compact manner (for the real-world homeless youth network, the size of the transition probability matrix is reduced from a matrix containing $2^{300}\times{450 \choose 5} \times2^{300}$ numbers to just $300$ numbers). Second, PSINET utilizes the QMDP heuristic \cite{littman1995learning} to enable scale-up and eliminates the search tree of POMCP.


\section{$1^{st}$ Key Idea: Transition Probability Heuristic}
 In this section, we explain our transition probability heuristic that we use for estimating our POMDP's transition probability matrix. Essentially, we need to come up with a way of finding out the final state of the network (probabilistically) prior to the beginning of the next intervention round. Prior to achieving the final state, the network evolves in a pre-decided number of time-steps. Each time step corresponds to a period in which friends can talk to their friends. Therefore, a time step value of 3 implies allowing for friends at 3 hops distance to be influenced. 

However, we make an important assumption that we describe next. Consider two different chains of length four (nodes) as shown in Figure \ref{fig:FigureChain}. In Chain 1, only the node at the head of the chain is influenced (shown in black) and the remaining three nodes are not influenced (shown in white). The probability of the tail node of this chain getting influenced is \textit{(0.5)}\textsuperscript{3} (assuming no edge is uncertain and probability of propagation is 0.5 on all edges). In Chain 2, all nodes except the tail node is already influenced. In this case, the tail node gets influenced with a probability \textit{0.5 + (0.5)}\textsuperscript{2}\textit{ + (0.5)}\textsuperscript{3}. Thus, it is highly unlikely that influence will spread to the end node of the first chain as opposed to the second chain. For this reason, we only keep chains of the form of Chain 2 and accordingly prune our graph (explained next).

\begin{figure}[htb]
\center{\includegraphics[scale=.45]
{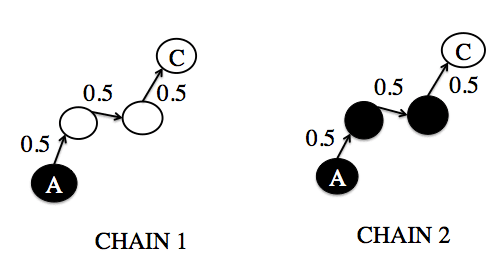}}
\caption{\label{fig:FigureChain} Chains in social networks}
\end{figure}

Given action $\alpha$, we construct a weighted adjacency matrix for graph $G_{\sigma}$ (created from graph $G$) s.t.

\begin{equation}
G_{\sigma}(i,j) = \begin{cases} 1 &\mbox{if } (i,j) \in E_c \wedge (W[i]=1 \vee \alpha[i]=1)\\
u(i,j) & \mbox{if } (i,j) \in E_u \wedge (W[i]=1 \vee \alpha[i]=1)\\ 
0 & \mbox{if } otherwise. \end{cases} 
\end{equation}

$G_{\sigma}$ is a \emph{pruned} graph which contains only edges outgoing from influenced nodes. We prune the graph because influence can only spread through edges which are outgoing from influenced nodes. Note that $G_{\sigma}$ does not consider influence spreading along a path consisting of more than one uninfluenced node, as this event is highly unlikely in the limited time in between successive interventions. However, nodes connected to a chain (of arbitrary length) of influenced nodes get influenced more easily due to reinforced efforts of all influenced nodes in the chain. Note that with respect to the chains in Figure \ref{fig:FigureChain},  $G_{\sigma}$ only considers chains of type 2 and prunes away chains of type 1.

Using these assumptions, we use $G_{\sigma}$ to construct a diffusion vector $\bf{D}$, the $i^{th}$ element of which gives us a measure of the probability of the $i^{th}$ node to get influenced.  This diffusion vector $\bf{D}$ is then used to estimate $T(s,\alpha,s')$.  

\begin{figure}[htb]
\center{\includegraphics[scale=.35]
{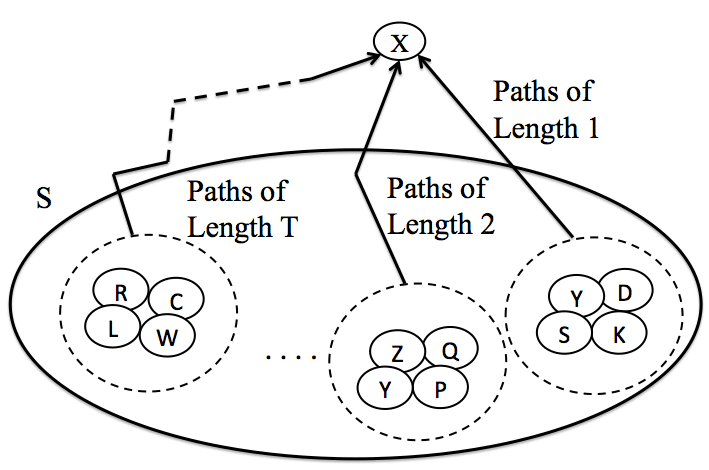}}
\caption{\label{fig:Figureintuition}X is any uninfluenced node. S (the big oval) denotes the set of all influenced nodes. All these nodes have been categorized according to their path length from node X. For e.g., all nodes having a path of length 1 (i.e., Y, D, S, K) are distinguished from all nodes having path of length T (i.e., R, W, L, C). Note that node Y has paths of length 1 and 2 to node X.}
\end{figure} 

Figure \ref{fig:Figureintuition} illustrates the intuition behind our transition probability heuristic. For each uninfluenced node \textit{X} in the graph, we calculate the total number of paths (like Chain 2 in Figure \ref{fig:FigureChain}) of different lengths \textit{L=1, 2,...,$\mathcal{T}$} from influenced nodes to node \textit{X}. Since influence spreads on chains of different lengths according to different probabilities, the probabilities along all paths of different lengths are combined together to determine an approximate probability of node \textit{X} to get influenced before the next intervention round. Since we consider all these paths independently (instead of calculating joint probabilities), our approach produces an approximation. Next, we formalize this intuition of the transition probability heuristic.


A known result states that if $G$ is a graph's adjacency matrix, then $G^r(i,j)$ ($G^r$ = $G$ multiplied $r$ times) gives the number of paths of length $r$ between nodes $i$ and $j$ \cite{diestel2005graph}. Additionally, note that if all edges $e_i$ in a path of length $r$ have different propagation probabilities $p(e_i)\mbox{ }\forall\mbox{ } i \mbox{ }\in [1,r]$, the probability of influence spreading between two nodes connected through this path of length $r$ is $\Pi_{i=1}^r p(e_i)$. For simplicity, we assume the same $p(e)\mbox{ } \forall e \in E$; hence, the probability of influence spreading becomes $p^r$. Using these results, we construct diffusion vector $\bf{D}$:
\begin{equation}
\bf{D(p,T)_{nx1}} = \sum\nolimits_{t\in[1,T]} \Big( \left( p \overline{G}_{\sigma} \right)^t * \bf{1_{nx1}} \Big)
\end{equation}

Here, $\bf{D(p, T)}$ is a column vector of size nx1, $\bf{p}$ is the constant propagation probability on the edges, $\bf{T}$ is a variable parameter that measures number of hops considered for influence spread (higher values of $\bf{T}$ yields more accurate $\bf{D(p, T)}$ but increases the runtime), $\bf{1_{nx1}}$ is a nx1 column vector of 1's and $\overline{G}_{\sigma}$ is the transpose of $G_{\sigma}$. This formulation is similar to diffusion centrality \cite{banerjee2013diffusion} where they calculate influencing power of nodes. However, we calculate power of nodes to get influenced (by using $\overline{G}_{\sigma}$). 

\begin{proposition}\label{lem:(2)}
$\bf{D_i}$, the $i^{th}$ element of $\bf{D(p,T)_{nx1}}$, upon normalization, gives an approximate probability of the $i^{th}$ graph node to get influenced in the next round.
\end{proposition}

Consider the set $\bigtriangleup = \{i \mbox{ } | \mbox{ }W'[i]=1 \wedge W[i]=0 \wedge \alpha[i]=0\}$, which represents nodes which were uninfluenced in the initial state $s$ ($W[i]=0$) and which were not selected in the action ($\alpha[i]=0$), but got influenced by other nodes in the final state $s'$ ($W'[i]=1$). Similarly, consider the set $\Phi = \{j \mbox{ } | \mbox{ }W'[j]=0 \wedge W[j]=0 \wedge \alpha[j]=0\}$, which represents nodes which were not influenced even in the final state $s'$ ($W'[j]=0$).\ 
Using $\bf{D_i}$ values, we can now calculate $T(s,\alpha,s') = {\displaystyle \Pi_{i \in \bigtriangleup} \bf{D_i}}{\displaystyle \Pi_{j \in \Phi} (1 - \bf{D_j})}$, i.e., we multiply influence probabilities $\bf{D_i}$ for nodes which are influenced in state $s'$, along with probabilities of not getting influenced $(1 - \bf{D_j})$ for nodes which are not influenced in state $s'$. This heuristic allows storing transition probability matrices in a compact manner, as only a single number for each network node (specifying the probability that the node will be influenced) needs to be maintained. Next, we discuss the QMDP heuristic used inside PSINET and the overall flow of the algorithm.

 
\section{$2^{nd}$ Key Idea: Leveraging the QMDP Heuristic}

\subsection{QMDP} It is a well known approximate offline planner, and it relies on $Q(s,a)$ values, which represents the value of taking action $a$ in state $s$. It precomputes these $Q(s,a)$ values for every $(s,a)$ pair by approximating them by the future expected reward obtainable if the environment is fully observable \cite{littman1995learning}. Finally, QMDP's approximate policy $\Pi$ is given by $\Pi(b) = \argmax_{a} \sum_s Q(s,a)b(s)$ for belief $b$. Our intractable POMDP state/action spaces makes it infeasible to calculate $Q(s,a)$  $\forall$ $(s,a)$. Thus, we propose to use a MC sampling based online variant of QMDP in PSINET.

\subsection{PSINET Algorithm Flow} Algorithm 2 shows the flow of PSINET. In Step \ref{qmdpflow:1}, we randomly sample all $e \in E_u$ in $G$ (according to $u(e)$) to get $\Delta$ different graph instances. Each of these instances is a different POMDP as the h-values of nodes are still partially observable. Since each of these instances fixes $f(e)\mbox{ } \forall e \in E_u$, the belief $\beta$ is represented as an un-weighted particle filter where each particle is a tuple of h-values of all nodes. This belief is shared across all instantiated POMDPs. For every graph instance $\delta \in \Delta$, we find the best action $\alpha_{\delta}$ in graph $\delta$, for the current belief $\beta$ in step \ref{qmdpflow:3}. In step \ref{qmdpflow:5}, we find the best action $\kappa$ for belief $\beta$, over all $\delta \in \Delta$ by voting amongst all the actions chosen by $\delta \in \Delta$. Then, in step \ref{qmdpflow:6}, we update the belief state based on the chosen action $\kappa$ and the current belief $\beta$. PSINET can again be used to find the best action for this or any future updated belief states. We now detail the steps in Algorithm 2.

\begin{algorithm}[t!]
\label{alg:OnlineQMDP}
\caption{PSINET}
\KwIn{Belief state $\beta$, Uncertain graph $G$}
\KwOut{Best Action $\kappa$}
Sample graph to get $\Delta$ different instances; \\\label{qmdpflow:1}
\For {$\delta \in \Delta$} {
	$FindBestAction(\delta, \alpha_{\delta}, \beta)$;\\\label{qmdpflow:3}
	}
	$\kappa = VoteForBestAction(\Delta, \alpha)$\\\label{qmdpflow:5}
	$UpdateBeliefState(\kappa, \beta)$;\\\label{qmdpflow:6}
	return $\kappa$;\\
\end{algorithm}

\textbf{Sampling Graphs} In Step \ref{qmdpflow:1}, we randomly keep or remove uncertain edges to create one graph instance. As a single instance might not represent the real network well, we instantiate the graph $\Delta$ times and use each of these instances to vote for the best action to be taken. 

\textbf{FindBestAction} Step \ref{qmdpflow:3} uses Algorithm 3, which finds the best action for a single network instance, and works similarly for all instances. For each instance, we find the action which maximizes long term rewards averaged across $n$ (we use $n=2^8$) MC simulations starting from states (particles) sampled from the current belief $\beta$. Each MC simulation samples a particle from $\beta$ and chooses an action to take (choice of action is explained later). Then, upon taking this action, we follow a uniform random rollout policy (until either termination, i.e., all nodes get influenced, or the horizon is breached) to find the long term reward, which we get by taking the ``selected" action. This reward from each MC simulation is analogous to a $Q(s,a)$ estimate. Finally, we pick the action with the maximum average reward.

\begin{algorithm}[t!]
\label{alg:FindBestAction}
\caption{FindBestAction}
\KwIn{Graph instance $\delta$, belief $\beta$, $\bf{N}$ simulations}
\KwOut{Best Action $\alpha_{\delta}$}
Initialize $counter = 0$;\\\label{newflow:1}
\While {$counter++ < \bf{N}$} {
	$s = SampleStartStateFromBelief(\beta)$;\\\label{newflow:2}
	$a = UCT\_MultiArmedBandit(s)$;\\\label{newflow:3}
	$\{s',r\} = SimulateRolloutPolicy(s,a)$;\\\label{newflow:4}
	}
	$\alpha_{\delta} = \mbox{action with max average reward}$;\\\label{newflow:7}
	return $\alpha_{\delta}$;
\end{algorithm}

\textbf{Multi-Armed Bandit} We can only calculate $Q(s,a)$ for a select set of actions (due to our intractable action space). To choose these actions, we use a UCT implementation of a multi-armed bandit to select actions, with each bandit arm being one possible action. Every time we sample a new state from the belief, we run UCT, which returns the action which maximizes this quantity: $\Upsilon(s,a)=Q_{MC}(s,a)+c_0\sqrt{\frac{\log N(s)}{N(s,a)}}$. Here, $Q_{MC}(s,a)$ is the running average of Q(s,a) values across all MC simulations run so far. $N(s)$ is number of times state $s$ has been sampled from the belief. $N(s,a)$ is number of times action $a$ has been chosen in state $s$ and $c_0$ is a constant which determines the exploration-exploitation tradeoff for UCT. High $c_0$ values make UCT choose rarely tried actions more frequently, and low $c_0$ values make UCT select actions having high $Q_{MC}(s,a)$ to get an even better $Q(s,a)$ estimate. Thus, in every MC simulation, UCT strategically chooses which action to take, after which we run the rollout policy to get the long term reward.

\textbf{Voting Mechanisms} In Step \ref{qmdpflow:5}, each network instance votes for the best action (found using Step \ref{qmdpflow:3}) for the uncertain graph and the action with the highest votes is chosen. We propose three different voting schemes:

\begin{itemize}
\item \textbf{PSINET-S} Each instance's vote gets equal weight.

\item \textbf{PSINET-W} Every instance's vote gets weighted differently. The instance which removes $x$ uncertain edges has a vote weight of $W(x)= x \mbox{ }\forall x \leq m/2$ and $W(x)=m - x \mbox{ }\forall x > m/2$. This weighting scheme approximates the probabilities of occurrences of real world events by giving low weights to instances which removes either too few or too many uncertain edges, since those events are less likely to occur. Instances which remove $m/2$ uncertain edges get the highest weight, since that event is most likely.

\item \textbf{PSINET-C} Given a ranking over actions from each instance, the Copeland rule makes pairwise comparisons among all actions, and picks the one preferred by a majority of instances over the highest number of other actions \cite{pomerol2000multicriterion}. Algorithm 3 is run $D$ times for each instance to generate a partial ranking. 
\end{itemize} 
 
\textbf{Belief State Update} Recall that every MC simulation samples a particle from the belief, after which UCT chooses an action. Upon taking this action, some random state (particle) is reached using the transition probability heuristic. This particle is stored, indexed by the action taken to reach it. Finally, when all simulations are done, corresponding to every action $\alpha$ that was tried during the simulations, there will be a set of particles that were encountered when we took action $\alpha$ in that belief. The particle set corresponding to the action that we finally choose, forms our next belief state.

\section{Experimental Evaluation}
We provide two sets of results. First, we show results on artificial networks to understand our algorithms' properties on abstract settings, and to gain insights on a range of networks. Next, we show results on the two real world homeless youth networks that we had access to. In all experiments, we select 2 nodes per round and average over 20 runs, unless otherwise stated. PSINET-(S and W) use 20 network instances and PSINET-C uses 5 network instances (each instance finds its best action 5 times) in all experiments, unless otherwise stated. The propagation and existence probability values were set to 0.5 in all experiments (based on findings by \cite{kelly1997randomised}), although we relax this assumption later in the section. In this section, a $\tuple{X, Y, Z}$ network refers to a network with $X$ nodes, $Y$ certain and $Z$ uncertain edges. We use a metric of ``indirect influence spread" (IIS) throughout this section, which is number of nodes ``indirectly" influenced by intervention participants. For example, on a 30 node network, by selecting 2 nodes each for 10 interventions (horizon), 20 nodes (a lower bound for any strategy) are influenced with certainty. However, the total number of influenced nodes might be 26 (say) and thus, the IIS is 6. \textit{All comparison results are statistically significant under bootstrap-t ($\alpha = 0.05$)}.

\begin{figure}[htp]
\subfloat[Solution Quality]{\includegraphics[height=1.5in,width=0.48\columnwidth]{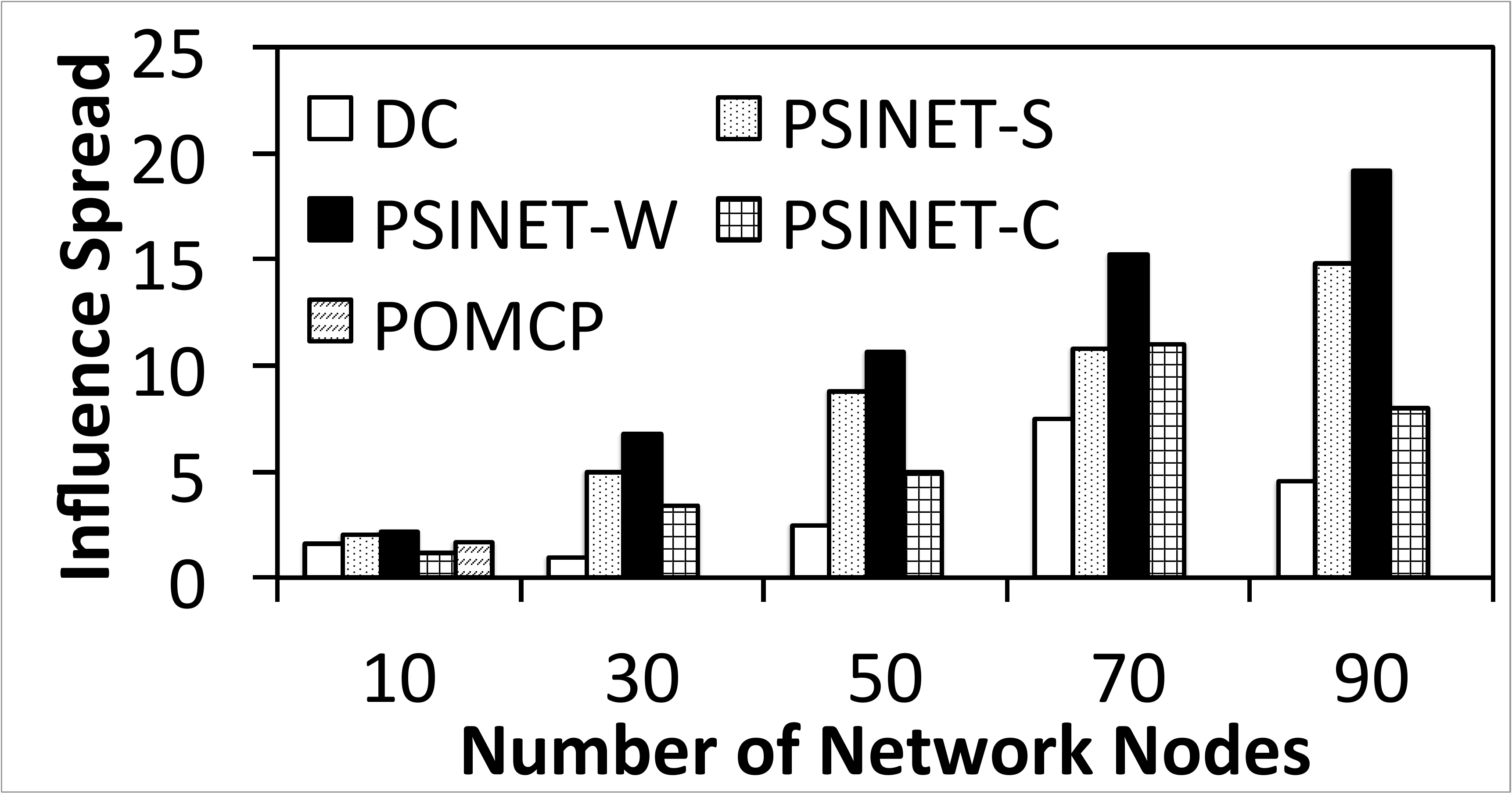}\label{fig:qmdpSolQual}}
\hspace{2mm}
\subfloat[Runtime]{\includegraphics[height=1.5in,width=0.48\columnwidth]{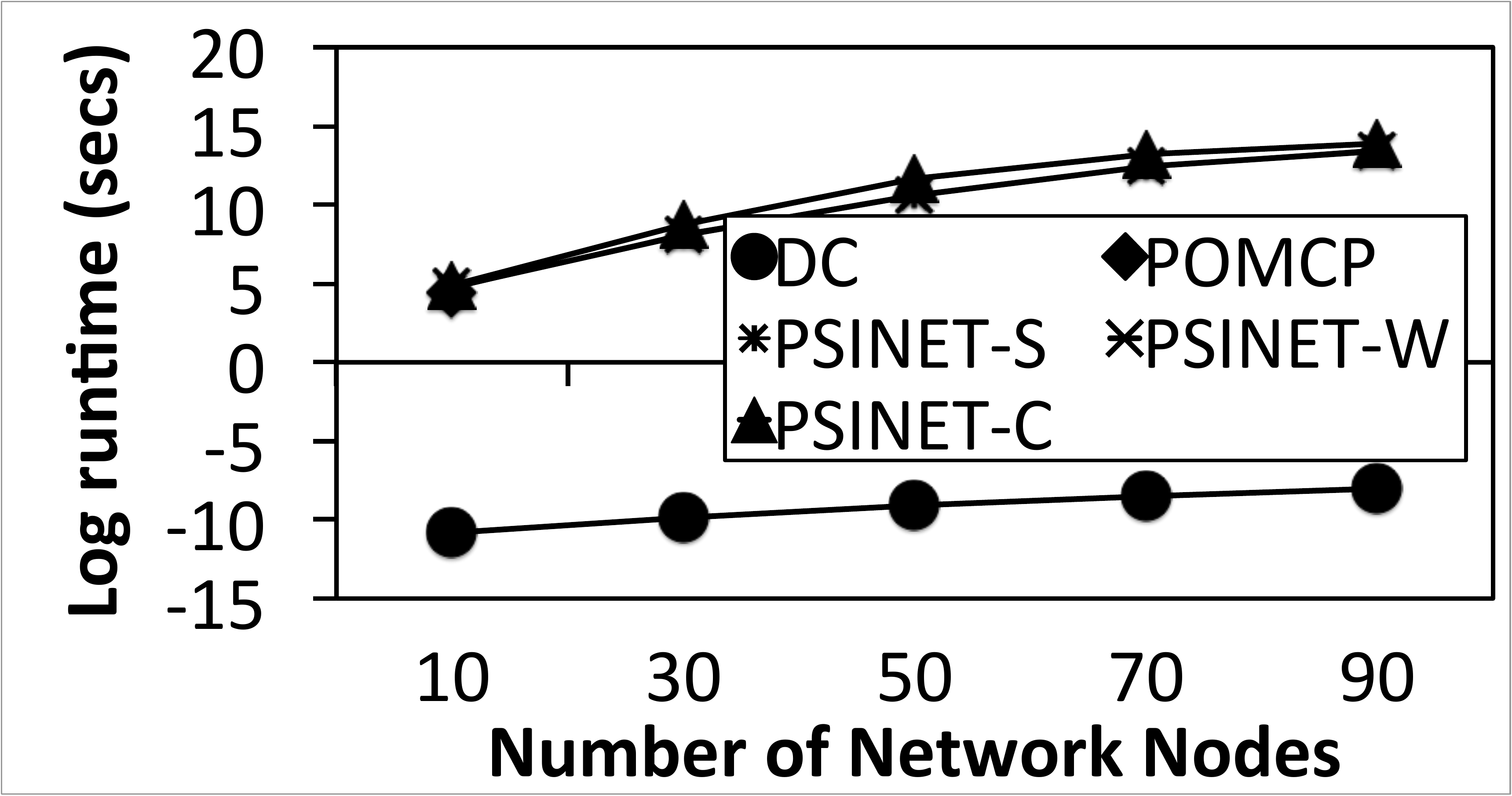}\label{fig:qmdpRuntime}}
\caption{\small Comparison on BTER graphs}
\end{figure}

\textbf{Artificial networks} First, we compare all algorithms on Block Two-Level Erdos-Renyi (BTER) networks (having degree distribution $X_d \propto d^{-1.2}$, where $X_d$ is number of nodes of degree $d$) of several sizes, as they accurately capture observable properties of real-world social networks \cite{seshadhri2012community}. Figures \ref{fig:qmdpSolQual} and \ref{fig:qmdpRuntime} show solution quality and runtimes (respectively) of Degree Centrality (DC) (which selects nodes based on their out-degrees, and $e \in E_u$ add $u(e)$ to node degrees), POMCP and PSINET-(S,W and C). We choose DC as our baseline as it is the current modus operandi of agencies working with homeless youth. X-axis is number of network nodes and Y-axis shows IIS across varying horizons (number of interventions) in Figure \ref{fig:qmdpSolQual} and log of runtime (in seconds) (Figure \ref{fig:qmdpRuntime}).

Figure \ref{fig:qmdpSolQual} shows that all POMDP based algorithms beat DC by $\sim$60\%, which shows the value of our POMDP model. Further, it shows that PSINET-W beats PSINET-(S and C). Also, \textit{POMCP runs out of memory on 30 node graphs}. Figure \ref{fig:qmdpRuntime} shows that DC runs quickest (as expected) and all PSINET variants run in almost the same time. Thus, Figures \ref{fig:qmdpSolQual} and \ref{fig:qmdpRuntime} tell us that while DC runs quickest, it provides the worst solutions. Amongst the POMDP based algorithms, PSINET-W is the best algorithm that can provide good solutions and can scale up as well. Surprisingly, PSINET-C performs worse than PSINET-(W and S) in terms of solution quality. Thus, we now focus on PSINET-W.

\begin{figure}[htp]
\subfloat[Solution Quality]{\includegraphics[height=1.5in,width=0.48\columnwidth]{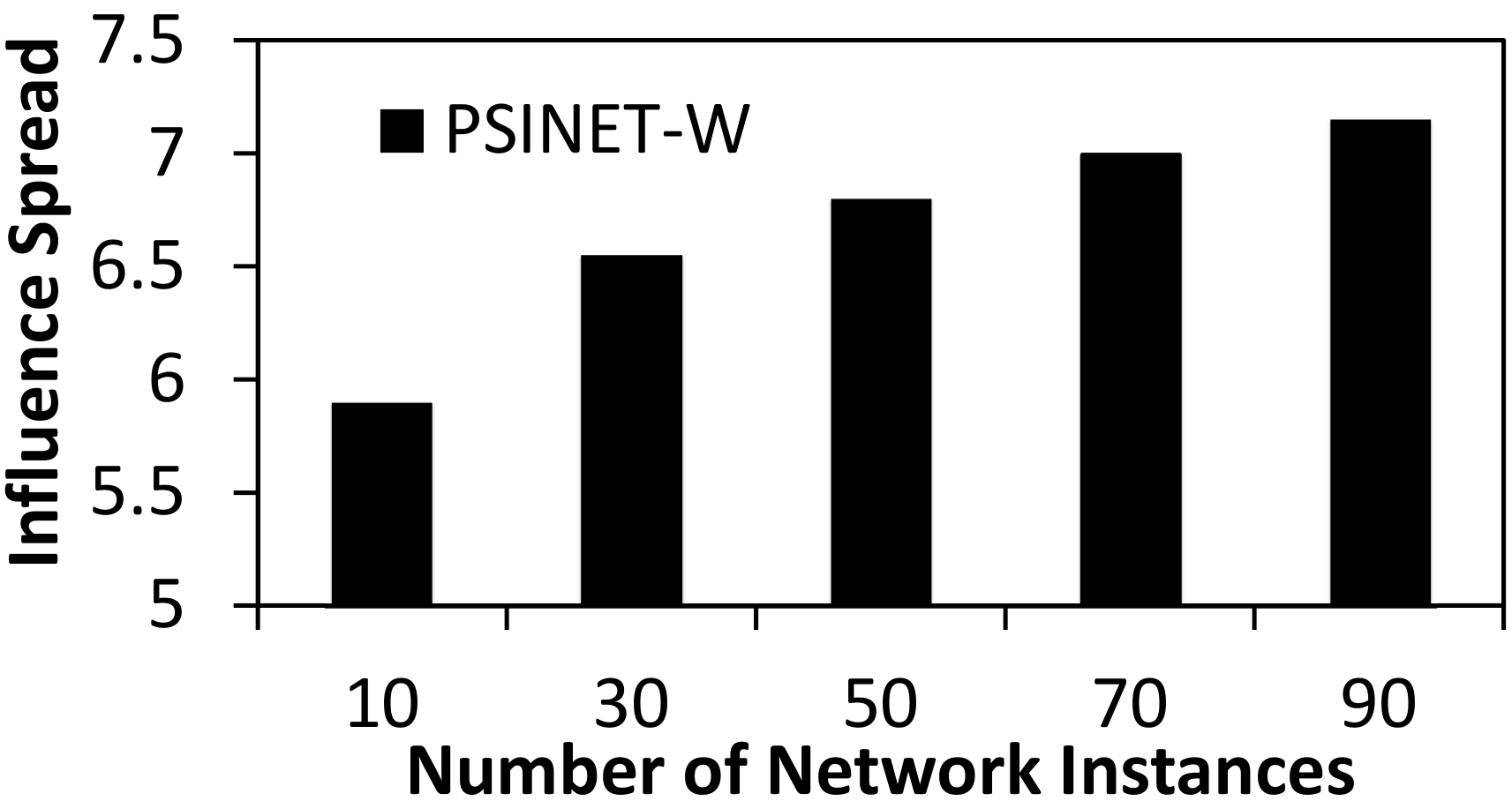}\label{fig:SolQual2}}
\hspace{2mm}
\subfloat[Runtime]{\includegraphics[height=1.5in,width=0.48\columnwidth]{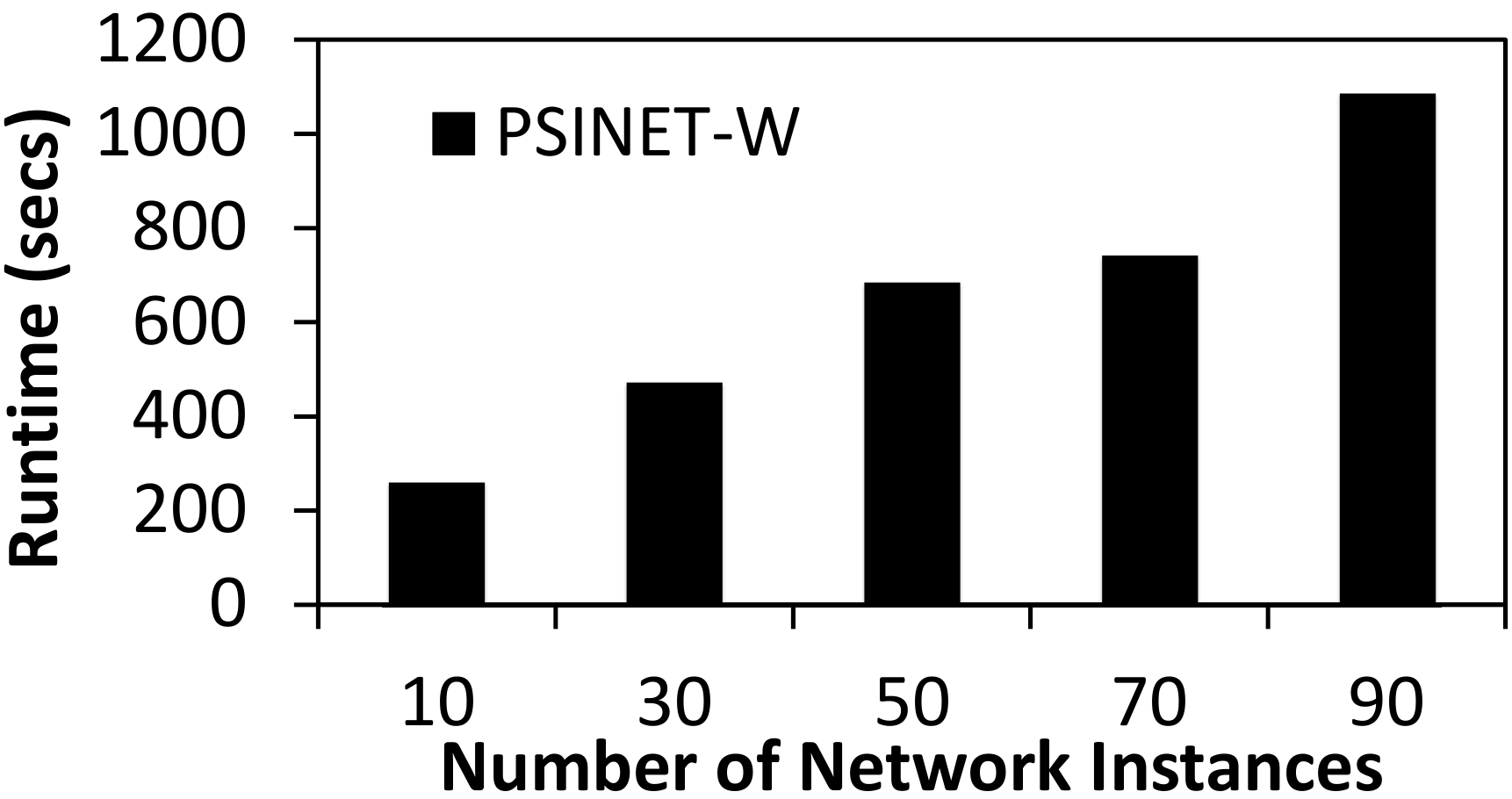}\label{fig:Runtime2}}
\caption{\small Increasing number of graph instances}
\end{figure}

Having shown the impact of POMDPs, we analyze the impact of increasing network instances (which implies increasing number of votes in our algorithm) on PSINET-W. Figures \ref{fig:SolQual2} and \ref{fig:Runtime2} show solution quality and runtime respectively of PSINET-W with increasing network instances, for a $\tuple{40,71,41}$ BTER network with a horizon of 10. X-axis is number of network instances and Y-axis shows IIS (Figure \ref{fig:SolQual2}) and runtime (in seconds) (Figure \ref{fig:Runtime2}). These figures show that increasing the number of instances increases IIS as well as runtime. Thus, a solution quality-runtime tradeoff exists, which depends on the number of network instances. Greater number of instances results in better solutions and slower runtimes and vice versa. However, for 30 vs 70 instances, the gain in solution quality is $<$5\% whereas the runtime is $\sim$2X, which shows that increasing instances beyond 30 yields marginal returns.

\begin{figure}[htp]
\subfloat[Varying $p(e)$]{\includegraphics[height=1.5in,width=0.48\columnwidth]{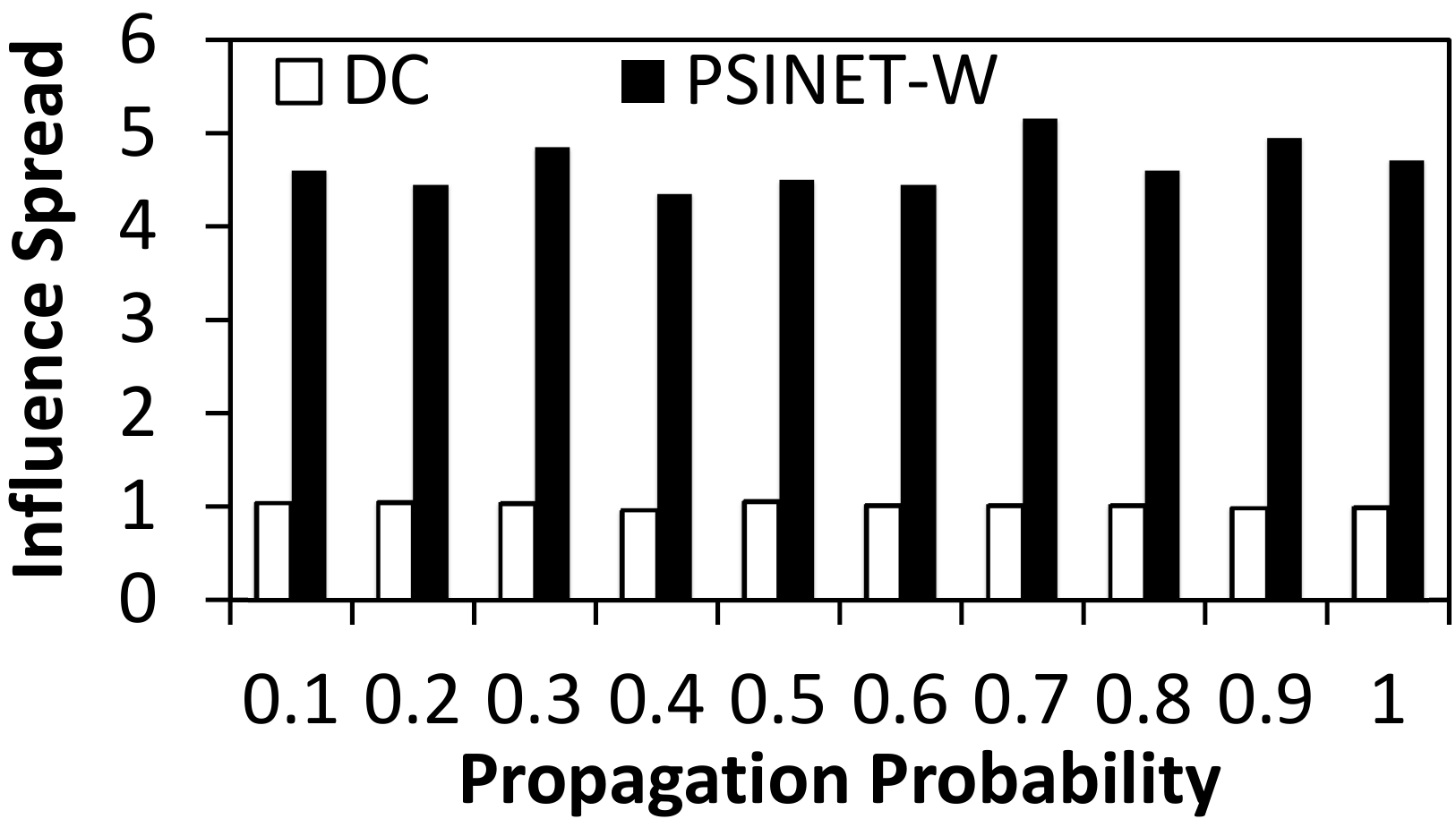}\label{fig:Propagation}}
\hspace{2mm}
\subfloat[ER networks]{\includegraphics[height=1.5in,width=0.48\columnwidth]{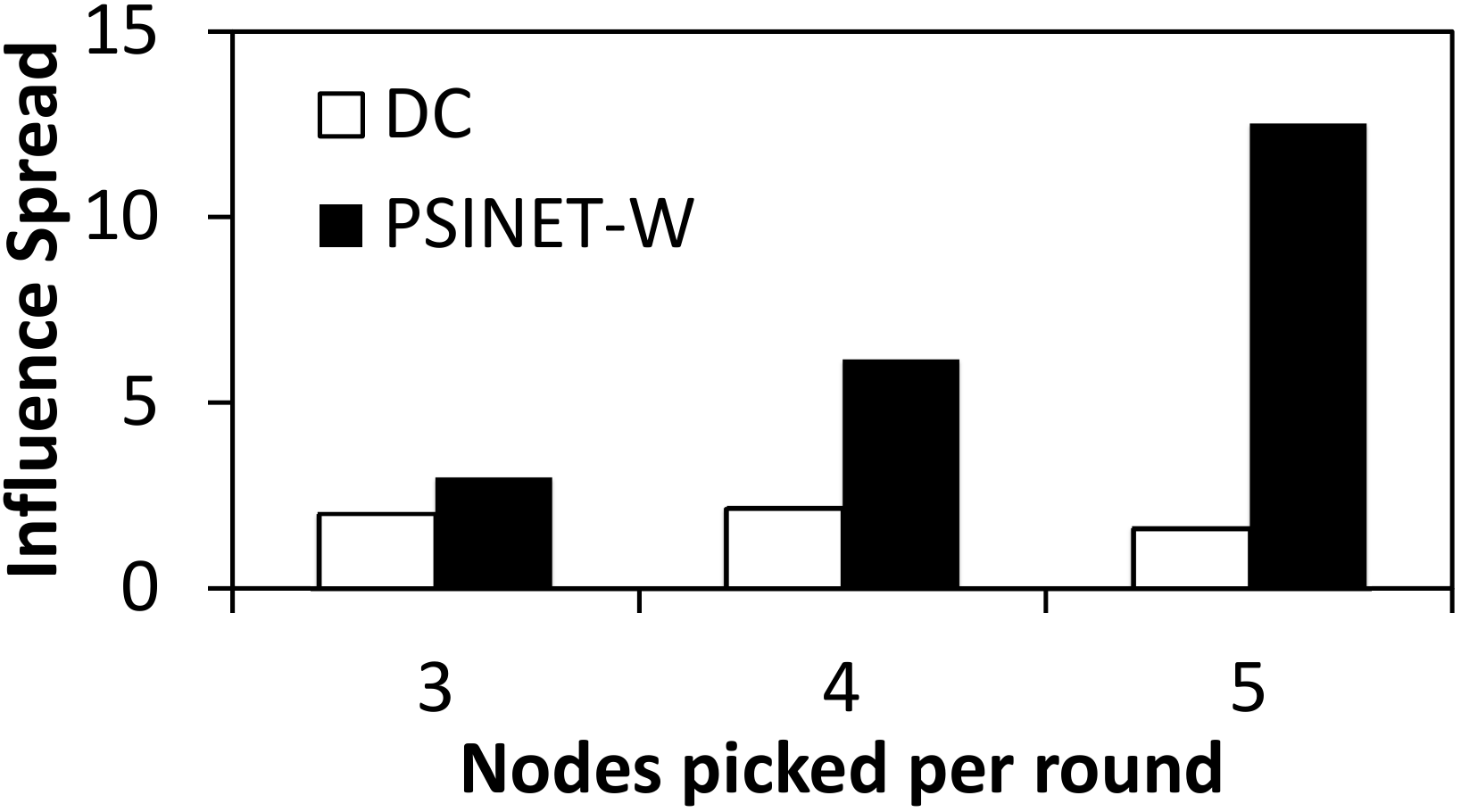}\label{fig:Scale}}
\caption{\small Comparison of DC with PSINET-W}
\end{figure}

Next, we relax our assumptions about propagation ($p(e)$) probabilities, which were set to 0.5 so far. Figure \ref{fig:Propagation} shows the solution quality, when PSINET-W and DC are solved with different $p(e)$ values respectively, for a $\tuple{40,71,41}$ BTER network with a horizon of 10. X-axis shows $p(e)$ and Y-axis shows IIS. This figure shows that varying $p(e)$ minimally impacts PSINET-W's improvement over DC, which shows our algorithms' robustness to these probability values (We get similar results upon changing $u(e)$). In Figure \ref{fig:Scale}, we show solution qualities of PSINET-W and DC on a $\tuple{30,31,27}$ BTER network (horizon=3) and vary number of nodes selected per round ($K$). X-axis shows increasing $K$, and Y-axis shows IIS. This figure shows that even for a small horizon of length 3, which does not give many chances for influence to spread, PSINET-W significantly beats DC with increasing $K$. 

\begin{figure}[htb]
\center{\includegraphics[height=2.8in,width=1\columnwidth]
{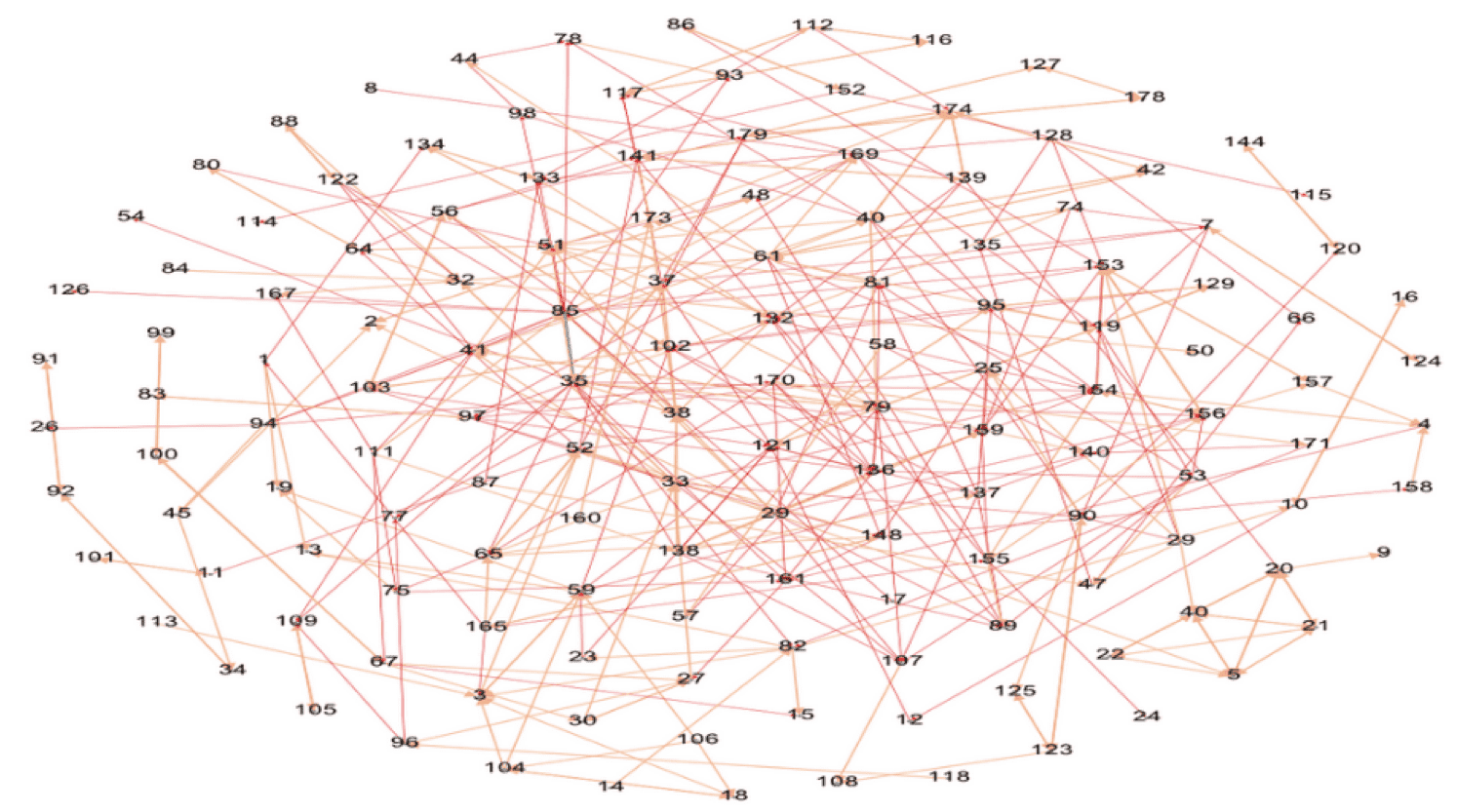}}
\caption{\label{fig:Figure17} A friendship based social network of homeless people visiting My Friend's Place}
\end{figure}

\begin{figure}[htp]
\subfloat[Solution Quality]{\includegraphics[height=1.5in,width=0.48\columnwidth]{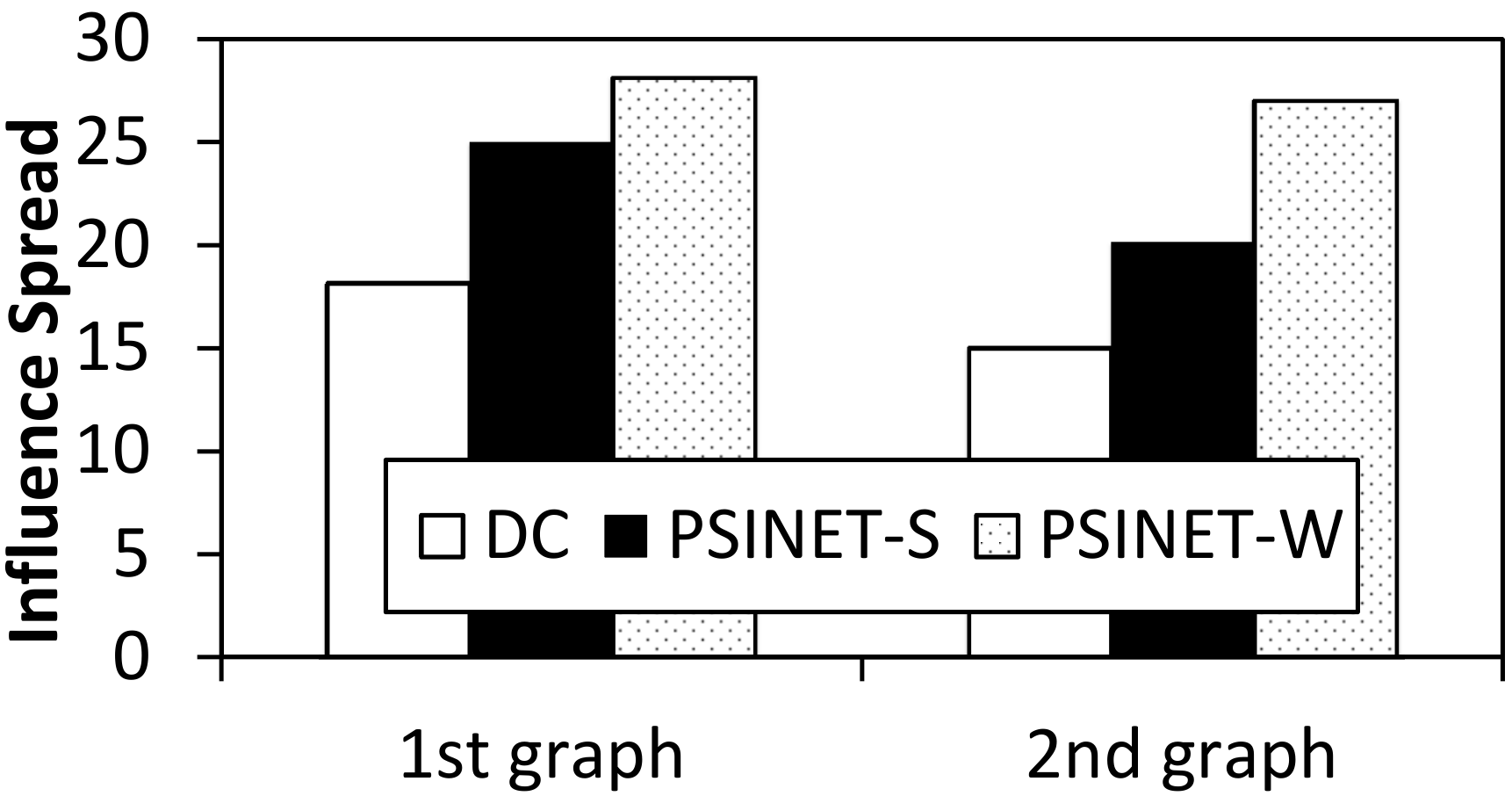}\label{fig:Figure18}}
\hspace{2mm}
\subfloat[Sample BTER graph]{\includegraphics[height=1.5in,width=0.48\columnwidth]{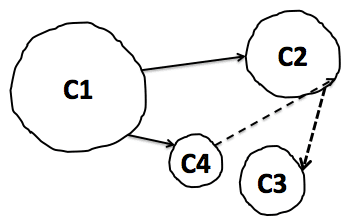}\label{fig:Figure19}}
\caption{\small Real world networks}
\end{figure}


\textbf{Real World Networks} Figure \ref{fig:Figure17} shows one of the two real-world friendship based social networks of homeless youth (created by our collaborators through surveys and interviews of homeless youth attending My Friend's Place), where each numbered node represents a homeless youth. Figure \ref{fig:Figure18} compares PSINET variants and DC (horizon = 30) on these two real-world social networks (each of size  around $\tuple{155,120,190}$).  The x-axis shows the two networks and the y-axis shows IIS. This figure clearly shows that all PSINET variants beat DC on both real world networks by around 60\%, which shows that PSINET works equally well on real-world networks. Also, PSINET-W beats PSINET-S, in accordance with previous results. Above all, this signifies that we could improve the quality and efficiency of HIV based interventions over the current modus operandi of agencies by around 60\%.


We now differentiate between the kinds of nodes selected by DC and PSINET-W for the sample BTER network in Figure \ref{fig:Figure19}, which contains nodes segregated into four clusters (C1 to C4), and node degrees in a cluster are almost equal. C1 is biggest, with slightly higher node degrees than other clusters, followed by C2, C3 and C4. DC would first select all nodes in cluster C1, then all nodes in C2 and so on. Selecting all nodes in a cluster is not ``smart", since selecting just a few cluster nodes influences all other nodes. PSINET-W realizes this by looking ahead and spreads more influence by picking nodes in different clusters each time. For example, assuming \textit{k=2}, PSINET-W picks one node in both C1 and C2, then one node in both C1 and C4, etc.

%
%
%

\section{Implementation Challenges}
Looking towards the future of testing the deployment of this procedure in agencies, there are a few implementation challenges that will need to be faced. First, collecting accurate social network data on homeless youth is a technical and financial burden beyond the capacity of most agencies working with these youth. Members of this team had a large three year grant from the National Institute of Mental Health to conduct such work in only two agencies. Our solution, moving forward (with other agencies) would be to use staff at agencies to delineate a first approximation of their homeless youth social network, based on their ongoing relationships with the youth. The POMDP procedure would subsequently be able to correct the network graph iteratively (by resolving uncertain edges via POMDP observations in each step). This is feasible because, as mentioned, homeless youth are more willing to discuss their social ties  in an intervention \cite{rice2012mobilizing}. We see this as one of the major strengths of this approach. 

Second, our prior research on homeless youth \cite{rice2013should} suggests that some structurally important youth may be highly anti-social and hence a poor choice for change agents in an intervention. We suggest that if such a youth is selected by the POMDP program, we then choose the next best action (subset of nodes) which does not include that ``anti-social" youth. Thus, the solution may require some ongoing management as certain individuals either refuse to participate as peer leaders or based on their anti-social behaviors are determined by staff to be inappropriate. 

Third, because of the history of neglect and abuse suffered by most of these youth, many are highly suspicious of adults. Including a computer-based selection procedure into the recruitment of peer leaders may raise suspicions about invasion of privacy for these youth. We suggest an ongoing public awareness campaign in the agencies working with this program to help overcome such fears and to encourage participation. Along with this issue, there is a secondary issue about protection of privacy for the individuals involved. Agencies collect information on their youth, but most of this information is not to be shared with researchers. We suggest working with agencies to create procedures which allow them to implement the POMDP program without having to provide identifying information to our team.

\section{Conclusion}
This paper presents PSINET, a POMDP based decision support system to select homeless youth for HIV based interventions. Previous work in strategic selection of intervention participants does not handle uncertainties in the social network's structure and evolving network state, potentially causing significant shortcomings in spread of information. PSINET has the following key novelties: (i) it handles uncertainties in network structure and evolving network state; (ii) it addresses these uncertainties by using POMDPs in influence maximization; and (iii) it provides algorithmic advances to allow high quality approximate solutions for such POMDPs. Simulations show that PSINET achieves around 60\% improvement over the current state-of-the-art. PSINET was developed in collaboration with My Friend's Place and has been reviewed by their officials.

Unfortunately, even though PSINET was able to scale up to real-world sized networks, it completely failed at scaling up in the number of nodes that get picked in every round (intervention). Thus, while PSINET was successful in scaling up to the required POMDP state space, it failed to deal with the explosion in action space that occurred with scale up in the number of nodes picked per round. To address this challenge, we designed HEALER, which we present next.

\chapter{HEALER}
\label{chapter:HEALER}
This chapter presents HEALER (or \textbf{H}ierarchical \textbf{E}nsembling based \textbf{A}gent which p\textbf{L}ans for \textbf{E}ffective \textbf{R}eduction in HIV Spread), an online POMDP algorithm which has a better scale-up performance than PSINET \cite{yadav2015preventing}. Internally, HEALER \cite{yadav2016using,yadav2017maximizing} is comprised of two different algorithms: HEAL and TASP. We now discuss these algorithms in detail.

\section{HEAL}
HEAL solves the \textit{original POMDP} using a novel \textit{hierarchical ensembling heuristic}: it creates ensembles of imperfect (and smaller) POMDPs at \textit{two} different layers, in a hierarchical manner (see Figure \ref{fig:Flow}). HEAL's \textit{top layer} creates an ensemble of smaller sized \textit{intermediate POMDPs} by subdividing the original \textit{uncertain network} into several smaller sized \textit{partitioned networks} by using graph partitioning techniques \cite{lasalle2013multi}. Each of these partitioned networks is then mapped onto a POMDP, and these \textit{intermediate POMDPs} form our \textit{top layer} ensemble of POMDP solvers.

\begin{figure}[t]
\center{\includegraphics[scale=.5]
{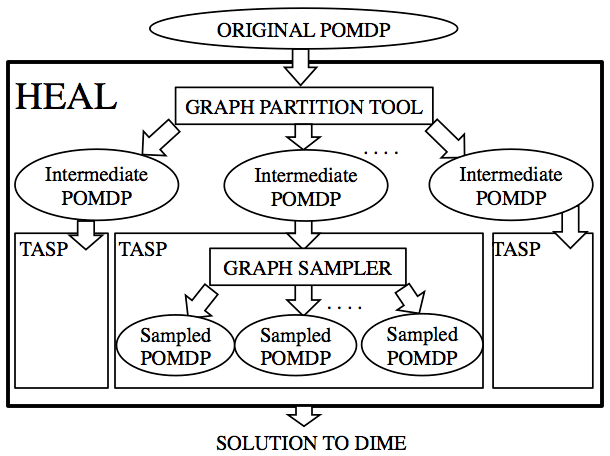}}
\caption{\label{fig:Flow} Hierarchical decomposition in HEAL}
\end{figure} 

In the bottom layer, each \textit{intermediate POMDP} is solved using TASP (\textbf{T}ree \textbf{A}ggregation for \textbf{S}equential \textbf{P}lanning), our novel POMDP planner, which subdivides the POMDP into another ensemble of smaller sized \textit{sampled POMDPs}. Each member of this \textit{bottom layer} ensemble is created by randomly sampling uncertain edges of the partitioned network to get a sampled network having no uncertain edges, and this sampled network is then mapped onto a \textit{sampled POMDP}. Finally, the solutions of POMDPs in both the \textit{bottom} and \textit{top layer} ensembles are aggregated using novel techniques to get the solution for HEAL's original POMDP. 

HEAL uses several novel heuristics. First, it uses a novel two-layered \textit{hierarchical ensembling heuristic}. Second, it uses graph partitioning techniques to partition the uncertain network, which generates partitions that minimize the edges going across partitions (while ensuring that partitions have similar sizes). Since these partitions are ``almost" disconnected, we solve each partition separately. Third, it solves the \textit{intermediate POMDP} for each partition by creating smaller-sized \textit{sampled POMDPs} (via sampling uncertain edges), each of which is solved using a novel tree search algorithm, which avoids the exponential branching factor seen in PSINET \cite{yadav2015preventing}. Fourth, it uses novel aggregation techniques to combine solutions to these smaller POMDPs rather than simple plurality voting techniques seen in previous ensemble techniques \cite{yadav2015preventing}.

These heuristics enable scale up to real-world sizes (at the expense of sacrificing performance guarantees), as instead of solving one huge problem, we now solve several smaller problems. However, these heuristics perform very well in practice. Our simulations show that even on smaller settings, HEAL achieves a 100X speed up over PSINET, while providing a 70\% improvement in solution quality; and on larger problems, \textit{where PSINET is unable to run at all}, HEAL continues to provide high solution quality. Now, we elaborate on these heuristics by first explaining the TASP solver.

\begin{algorithm}[h]
\label{alg:TASP}
\caption{TASP Solver}
\KwIn{Uncertain network $G$, Parameters $K$, $T$, $L$}
\KwOut{Best $K$ node action $\kappa$}
Create ensemble of $\Delta$ different POMDPs; \\\label{taspflow:1}
\For {$\delta \in \Delta$} {
	$\alpha^{\delta} = Evaluate(\delta)$;\\\label{taspflow:2}
	}
	$r = Expectation(\alpha)$;\\\label{taspflow:3}
	$\kappa = \argmax_j r_j$;\\\label{taspflow:4}
	return $\kappa$;\\
\end{algorithm}

\section{Bottom layer: TASP} \label{sec:TASP} We now explain TASP, our new POMDP solver that solves each \textit{intermediate POMDP} in HEAL's bottom layer. Given an \textit{intermediate POMDP} and the uncertain network it is defined on, as input, TASP goes through four steps (see Algorithm 4). 

First, Step \ref{taspflow:1} makes our \textit{intermediate POMDP} more tractable by creating an ensemble of smaller sized \textit{sampled POMDPs}. Each member of this ensemble is created by sampling uncertain edges of the input network to get an \textit{instantiated} network. Each uncertain edge in the input network is randomly kept with probability $u(e)$, or removed with probability $1-u(e)$, to get an \textit{instantiated} network with no uncertain edges. We repeat this sampling process to get $\Delta$ (a variable parameter) different \textit{instantiated} networks. These $\Delta$ different \textit{instantiated} networks are then mapped onto to $\Delta$ different POMDPs, which form our ensemble of \textit{sampled POMDPs}. Each \textit{sampled POMDP} shares the same action space (defined on the input partitioned network) as the different POMDPs only differ in the sampling of uncertain edges. Note that each member of our ensemble is a POMDP as even though sampling uncertain edges removes uncertainty in the $F$ portion of POMDP states, there is still partial observability in the $W$ portion of POMDP state.

In Step \ref{taspflow:2} (called the Evaluate Step), for each instantiated network $\delta \in [1,\Delta]$, we generate an $\alpha^{\delta}$ list of rewards. The $i^{th}$ element of $\alpha^{\delta}$ gives the long term reward achieved by taking the $i^{th}$ action in \textit{instantiated} network $\delta$. In Step \ref{taspflow:3}, we find the expected reward $r_i$ of taking the $i^{th}$ action, by taking a reward expectation across the $\alpha^{\delta}$ lists (for each $\delta \in [1,\Delta]$) generated in the previous step. For e.g., if $\alpha^{\delta_1}_1 = 10$ and $\alpha^{\delta_2}_1 = 20$, i.e., the rewards of taking the $1^{st}$ action in instantiated networks $\delta_1$ and $\delta_2$ (which occurs with probabilities $P(\delta_1)$ and $P(\delta_2)$) are 10 and 20 respectively, then the expected reward $r_1 = P(\delta_1)\times 10 + P(\delta_2)\times 20$. Note that $P(\delta_1)$ and $P(\delta_2)$ are found by multiplying existence probabilities $u(e)$ (or $1-u(e)$) for uncertain edges that were kept (or removed) in $\delta_1$ and $\delta_2$. Finally, in Step \ref{taspflow:4}, the action $\kappa = \argmax_j r_j$ is returned by TASP. Next, we discuss the Evaluate Step in detail (Step \ref{taspflow:2}).

\begin{algorithm}[t!]
\label{alg:Evaluate}
\caption{Evaluate Step}
\KwIn{Instantiated network $\delta$, Number of simulations $\bf{NSim}$}
\KwOut{Ranked Ordering of actions $\alpha^{\delta}$}
$tree = Initialize\_K\_Level\_Tree()$;\\\label{evalflow:0}
$counter = 0$;\\\label{evalflow:1}
\While {$counter++ < \bf{NSim}$} {
	$K\_Node\_Act =  FindStep(tree)$;\\\label{evalflow:2}
	$LT\_Reward = SimulateStep(K\_Node\_Act)$;\\\label{evalflow:3}
	$UpdateStep(tree, LT\_Reward, K\_Node\_Act)$;\\\label{evalflow:4}
	}
	$\alpha^{\delta} = Get\_All\_Leaf\_Values(tree)$;\\\label{evalflow:7}
	return $\alpha^{\delta}$;
\end{algorithm}

\subsection{Evaluate Step} Algorithm 5 generates the $\alpha^{\delta}$ list for a single instantiated network $\delta \in [1,\Delta]$. This algorithm works similarly for all instantiated networks. For each instantiated network, the Evaluate Step uses $\bf{NSim}$ (we use $2^{10}$) number of MC simulations to evaluate the long term reward achieved by taking actions in that network. Due to the combinatorial action space, the Evaluate Step uses a UCT \cite{kocsis2006bandit} driven approach to strategically choose the actions whose long term rewards should be calculated. UCT has been used to solve POMDPs in \cite{silver2010monte,yadav2015preventing}, but these algorithms suffer from a ${N \choose K}$  branching factor (where $K$ is number of nodes picked per round, $N$ is number of network nodes). We exploit the structure of our domain by creating a $K$-level UCT tree which has a branching factor of just $N$ (explained below). This $K$-level tree allows storing reward values for smaller sized node subsets as well (instead of just $K$ sized subsets), which helps in guiding the UCT search better.

Algorithm 5 takes an \textit{instantiated} network and creates the aforementioned $K$-level tree for that network. The first level of the tree has $N$ branches (one for each network node). For each branch $i$ in the first level, there are $N-1$ branches in the second tree level (one for each network node, except for node $i$, which was covered in the first level). Similarly, for every branch $j$ in the $m^{th}$ level ($m \in [2,K-1]$), there are $N-m$ branches in the $(m+1)^{th}$ level. Theoretically, this tree grows exponentially with $K$, however, the values of $K$ are usually small in practice (e.g., 4).

In this $K$ level tree, each leaf node represents a particular POMDP action of $K$ network nodes. Similarly, every non-leaf tree node $v$ represents a subset $S_v$ of network nodes. Each tree node $v$ maintains a value $R_v$, which represents the average long term reward achieved by taking our POMDP's actions (of size $K$) which contain $S_v$ as a subset. For example, in Figure \ref{fig:backgrnd_uncertainG}, if $K=5$, and for tree node $v$, $S_v = \{A,B,C,D\}$, then $R_v$ represents the average long term reward achieved by taking POMDP actions $A_1 = \{A,B,C,D,E\}$ and $A_2 = \{A,B,C,D,F\}$, since both $A_1$ and $A_2$ contain $S_v = \{A,B,C,D\}$ as a subset. To begin with, all nodes $v$ in the tree are initialized with $R_v=0$ (Step \ref{evalflow:0}). By running $\bf{NSim}$ number of MC simulations, we generate good estimates of $R_v$ values for each tree node $v$.

Each node in this $K$-level tree runs a UCB1 \cite{kocsis2006bandit} implementation of a multi-armed bandit. The arms of the multi-armed bandit running at tree node $v$ correspond to the child branches of node $v$ in the $K$-level tree. Recall that each child branch corresponds to a network node. The overall goal of all the multi-armed bandits running in the tree is to construct a POMDP action of size $K$ (by traversing a path from the root to a leaf), whose reward is then calculated in that MC simulation. Every MC simulation consists of three steps: Find Step (Step \ref{evalflow:2}), Simulate Step (Step \ref{evalflow:3}) and Update Step (Step \ref{evalflow:4}).

\subsection{Find Step} The Find Step takes a $K$-level tree for an instantiated network and  \textit{finds} a $K$ node action, which is used in the Simulate Step. Algorithm 6 details the process of \textit{finding} this $K$ node action, which is found by traversing a path from the root node to a leaf node, one edge/arm at a time. Initially, we begin at the root node with an empty action set of size 0 (Steps \ref{findflow:1} and \ref{findflow:2}). For each node that we visit on our way from the root to a leaf, we use its multi-armed bandit (denoted by $MAB_{node}$ in Step \ref{findflow:3}) to choose which tree node do we visit next (or, which network node do we add to our action set). We get a $K$ node action upon reaching a leaf.

\begin{algorithm}[t!]
\label{alg:FindStep}
\caption{FindStep}
\KwIn{$K$ level deep tree - $tree$}
\KwOut{Action set of size $K$ nodes - $Act$}
$Act = \Phi$;\\\label{findflow:1}
$tree\_node  = tree.Root$;\\\label{findflow:2}
\While {$is\_Leaf(tree\_node) == false$} {
	$MAB_{node} = Get\_UCB\_at\_Node(node)$;\\\label{findflow:3}
	$next\_node = Ask\_UCB(MAB_{node})$;\\\label{findflow:4}
	$Act = Act \cup next\_node$;\\
	$tree\_node = tree\_node.branch(next\_node)$;\\
	}
	return $Act$;\\
\end{algorithm}

\subsection{Simulate Step} The Simulate Step takes a $K$ node action from the Find Step, to \textit{evaluate} the long term reward of taking that action (called $Act$) in the instantiated network. Assuming that $T_0$ interventions remain (i.e., we have already conducted $T - T_0$ interventions), the Simulate Step first uses  action $Act$ in the generative model $\Lambda$ to generate a reward $r_0$. For all remaining $(T_0 - 1)$ interventions, Simulate Step uses a rollout policy to randomly select K node actions, which are then used in the generative model $\Lambda$ to generate future rewards $r_i \mbox{ }\forall \mbox{ } i \in [1, T_0 - 1]$ . Finally, the long term reward returned by Simulate Step is $r_0 + r_1 + ... + r_{T_0-1}$.

\subsection{Update Step} The Update Step uses the long term reward returned by Simulate Step to update relevant $R_v$ values in the $K$-level tree. It updates the $R_v$ values of all nodes $v$ that were traversed in order to find the $K$ node action in the Find Step. First, we get the tree's leaf node corresponding to the $K$ node action that was returned by the Find Step. Then, we go and update $R_v$ values for all ancestors (including the root) of that leaf node. 

After running the Find, Simulate and Evaluate for $\bf{NSim}$ simulations, we return the $R_v$ values of all leaf nodes as the $\alpha^{\delta}$ list. Recall that we then find the expected reward $r_i$ of taking the $i^{th}$ action, by taking an expectation of rewards across the $\alpha^{\delta}$ lists. Finally, TASP returns the action $\kappa = \argmax_j r_j$.

\section{Top layer: Using Graph Partitioning}\label{sec:parti} We now explain HEAL's top layer, in which we use METIS \cite{lasalle2013multi}, a state-of-the-art graph partitioning technique, to subdivide our original uncertain network into different partitioned networks. These partitioned networks form the ensemble of \textit{intermediate POMDPs} (in Figure \ref{fig:Flow}) in HEAL. Then, TASP is invoked on each intermediate POMDP independently, and their solutions are aggregated to get the final DIME solution. We try two different partitioning/aggregation techniques, which leads to two variants of HEAL:

\textbf{K Partition Variant (HEAL): } Given the \textit{uncertain} network $G$ and the parameters $K$, $L$ and $T$ as input, we first partition the uncertain network into $K$ partitions. In each round from 1 to $T$, we invoke the bottom layer TASP algorithm to select 1 node from each of the $K$ clusters. These singly selected nodes from the $K$ clusters give us an action of $K$ nodes, which is given to shelter officials to execute. Based on the \textit{observation} (about uncertain edges) that officials get while executing the action, we update the partition networks (which are input to the \textit{intermediate POMDPs}) by either replacing the \textit{observed} uncertain edges with certain edges (if the edge was \textit{observed} to exist in reality) or removing the uncertain edge altogether (if the edge was \textit{observed} to \textit{not exist} in reality). The list of $K$ node actions that Algorithm 4 generates serves as an online policy for use by the homeless shelter. 

\textbf{T Partition Variant (HEAL-T): } Given the \textit{uncertain} network $G$ and the parameters $K$, $L$ and $T$ as input, we first partition the uncertain network into $T$ partitions and TASP picks $K$ nodes from the $i^{th}$ partition ($i \in [1,T]$) in the $i^{th}$ round.

\section{Experimental Results}\label{sec:exp}
In this section, we analyze HEAL and HEAL-T's performance in a variety of settings. All our experiments are run on a 2.33 GHz 12-core Intel machine having 48 GB of RAM. All experiments are averaged over 100 runs. We use a metric of ``\textit{Indirect Influence}" throughout this section, which is number of nodes ``\textit{indirectly}" influenced by intervention participants. For example, on a 30 node network, by selecting 2 nodes each for 10 interventions (horizon), 20 nodes (a lower bound for any strategy) are influenced with certainty. However, the total number of influenced nodes might be 26 (say) and thus, the \textit{Indirect Influence} is $26-20 = 6$. In all experiments, the propagation and existence probability values on all network edges were uniformly set to $0.1$ and $0.6$, respectively. This was done based on findings in Kelly et. al.\cite{kelly1997randomised}. However, we relax these parameter settings later in the section. All experiments are statistically significant under bootstrap-t ($\alpha = 0.05$).

\textbf{Baselines: } We use two algorithms as baselines. We use PSINET-W as a benchmark as it is the most relevant previous algorithm, which was shown to outperform heuristics used in practice; however, we also need a point of comparison when PSINET-W does not scale. No previous algorithm in the influence maximization literature accounts for uncertain edges and uncertain network state in solving the problem of sequential selection of nodes; in-fact we show that even the standard Greedy algorithm \cite{kempe2003maximizing,golovin2011adaptive} has no approximation guarantees as our problem is not adaptive submodular. Thus, we modify Greedy by replacing our uncertain network with a certain network (in which each uncertain edge $e$ is replaced with a certain edge $e_0$ having propagation probability $p(e_0) = p(e) \times u(e)$), and then run the Greedy algorithm on this \textit{certain network}. We use the Greedy algorithm as a baseline as it is the best known algorithm known for influence maximization and has been analyzed in many previous papers \cite{cohen2014sketch,Borgs14,tang2014influence,kempe2003maximizing,leskovec2007cost,golovin2011adaptive}.

\textbf{Datasets: } We use \textit{four real world social networks} of homeless youth, provided to us by our collaborators. All four networks are friendship based social networks of homeless youth living in Los Angeles. The first and second networks are of homeless youth living in Venice Beach (VE) and Hollywood (HD), two large areas in Los Angeles, respectively. These two networks (each having $\sim$150-170 nodes, 400-450 edges) were created through surveys and interviews of homeless youth (conducted by our collaborators) living in these areas. The third and fourth networks are relatively small-sized online social networks of these youth created from their Facebook (34 nodes, 120 edges) and MySpace (107 nodes, 803 edges) contact lists, respectively. When HEALER is deployed, we anticipate even larger networks, (e.g., 250-300 nodes) than the ones we have in hand and we also show run-time results on artificial networks of these sizes.

\begin{figure}[h]
\subfloat[\small Solution Quality]{\includegraphics[height=1.5in,width=0.48\columnwidth]{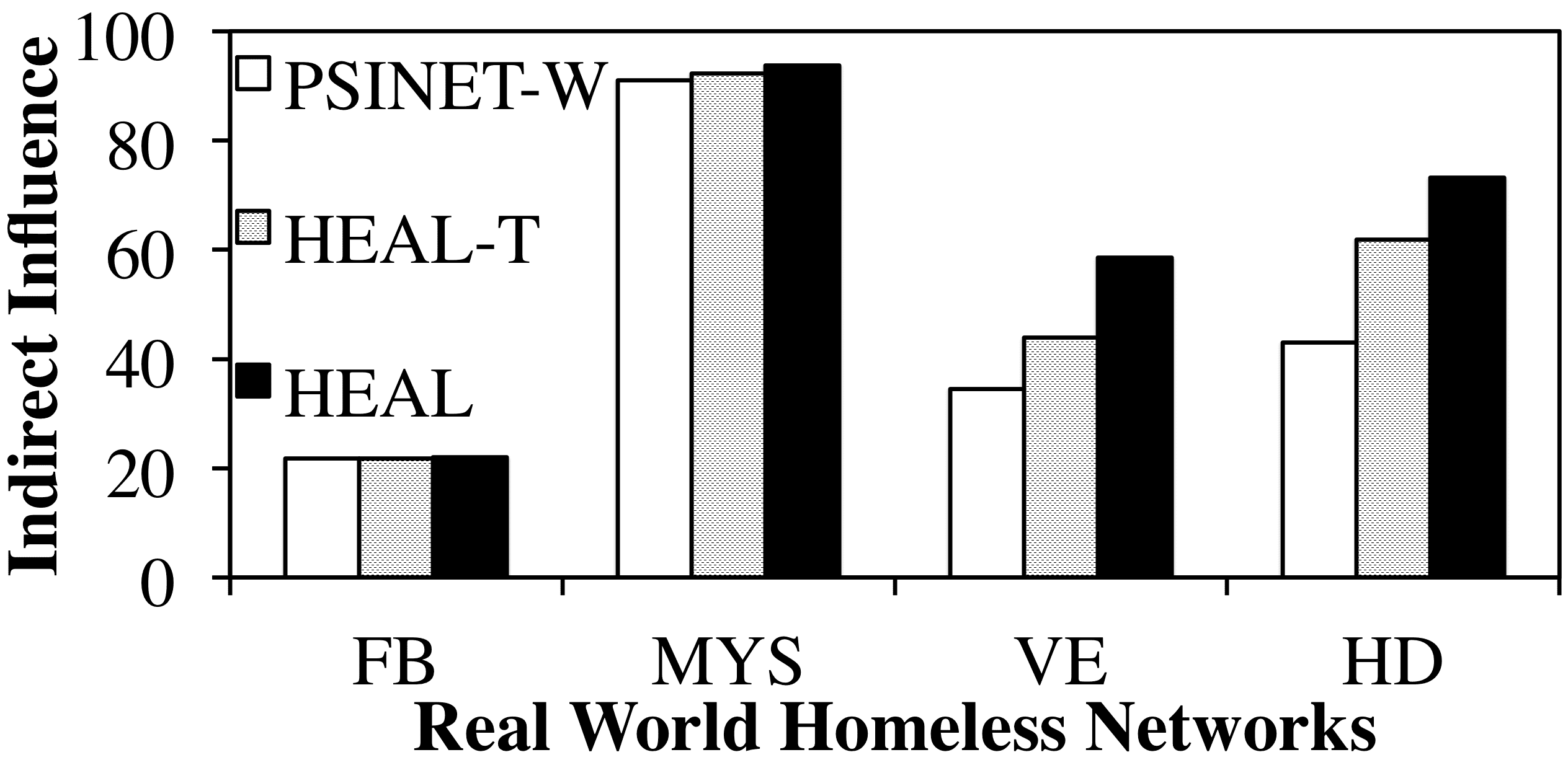}\label{fig:SolQual}}
\hspace{2mm}
\subfloat[\small Runtime]{\includegraphics[height=1.5in,width=0.48\columnwidth]{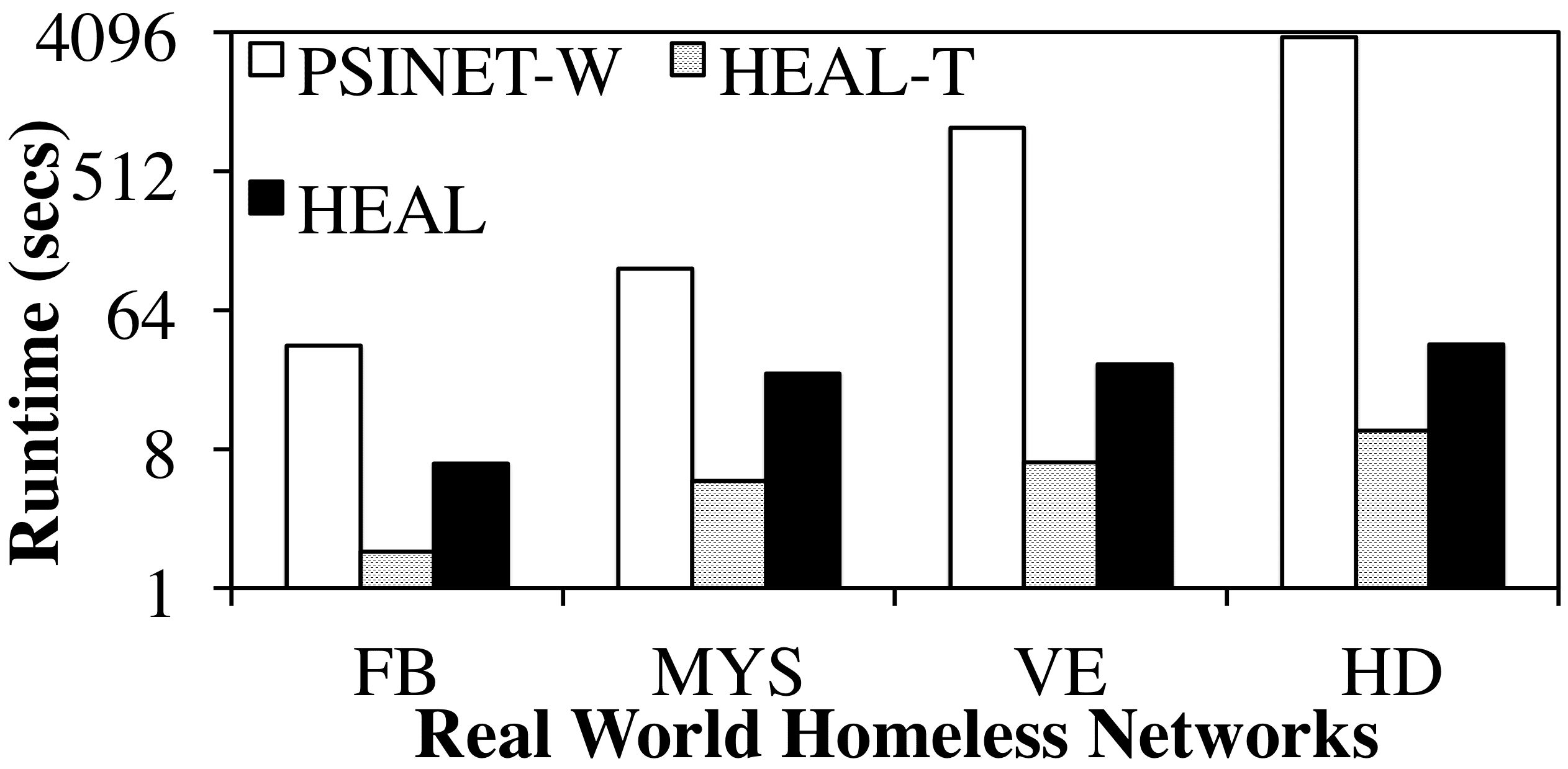}\label{fig:Runtime}}
\caption{Solution Quality and Runtime on Real World Networks}
\end{figure}  

\textbf{Solution Quality/Runtime Comparison. }We compare \textit{Indirect Influence} and run-times of HEAL, HEAL-T and PSINET-W on all four real-world networks. We set $T=5$ and $K=2$ (since PSINET-W fails to scale up beyond $K=2$ as shown later). Figure \ref{fig:SolQual} shows the \textit{Indirect Influence} of the different algorithms on the four networks. The X-axis shows the four networks and the Y-axis shows the \textit{Indirect Influence} achieved by the different algorithms. This figure shows that (i) HEAL outperforms all other algorithms on every network; (ii) \textit{it achieves $\sim$70\% improvement over PSINET-W} in VE and HD networks; (iii) it achieves $\sim$25\% improvement over HEAL-T. The difference between HEAL and other algorithms is not significant in the Facebook (FB) and MySpace (MYS) networks, as HEAL is already influencing almost all nodes in these two relatively small networks. Thus, in experiments to come, we focus more on the VE and HD networks.

Figure \ref{fig:Runtime} shows the run-time of all algorithms on the four networks. The X-axis shows the four networks and the Y-axis (in log scale) shows the run-time (in seconds). This figure shows that (i) \textit{HEAL achieves a 100X speed-up over PSINET-W}; (ii) PSINET-W's run-time increases exponentially with increasing network sizes; (iii) HEAL runs 3X slower than HEAL-T but achieves 25\% more \textit{Indirect Influence}. Hence, HEAL is our algorithm of choice.

Next, we check if PSINET-W's run-times become worse on larger networks. Because of lack of larger real-world datasets, we create relatively large artificial Watts-Strogatz networks (model parameters $p=0.1,k=7$). Figure \ref{fig:big-runtime} shows the run-time of all algorithms on Watts-Strogatz networks. The X-axis shows the size of networks and the Y-axis (in log scale) shows the run-time (in seconds). This figure shows that \textit{PSINET-W fails to scale beyond 180 nodes, whereas HEAL runs within 5 minutes}. Thus, PSINET-W fails to scale-up to network sizes that are of importance to us.

\begin{figure}[h]
\subfloat[\small VE Network]{\includegraphics[height=1.5in,width=0.48\columnwidth]{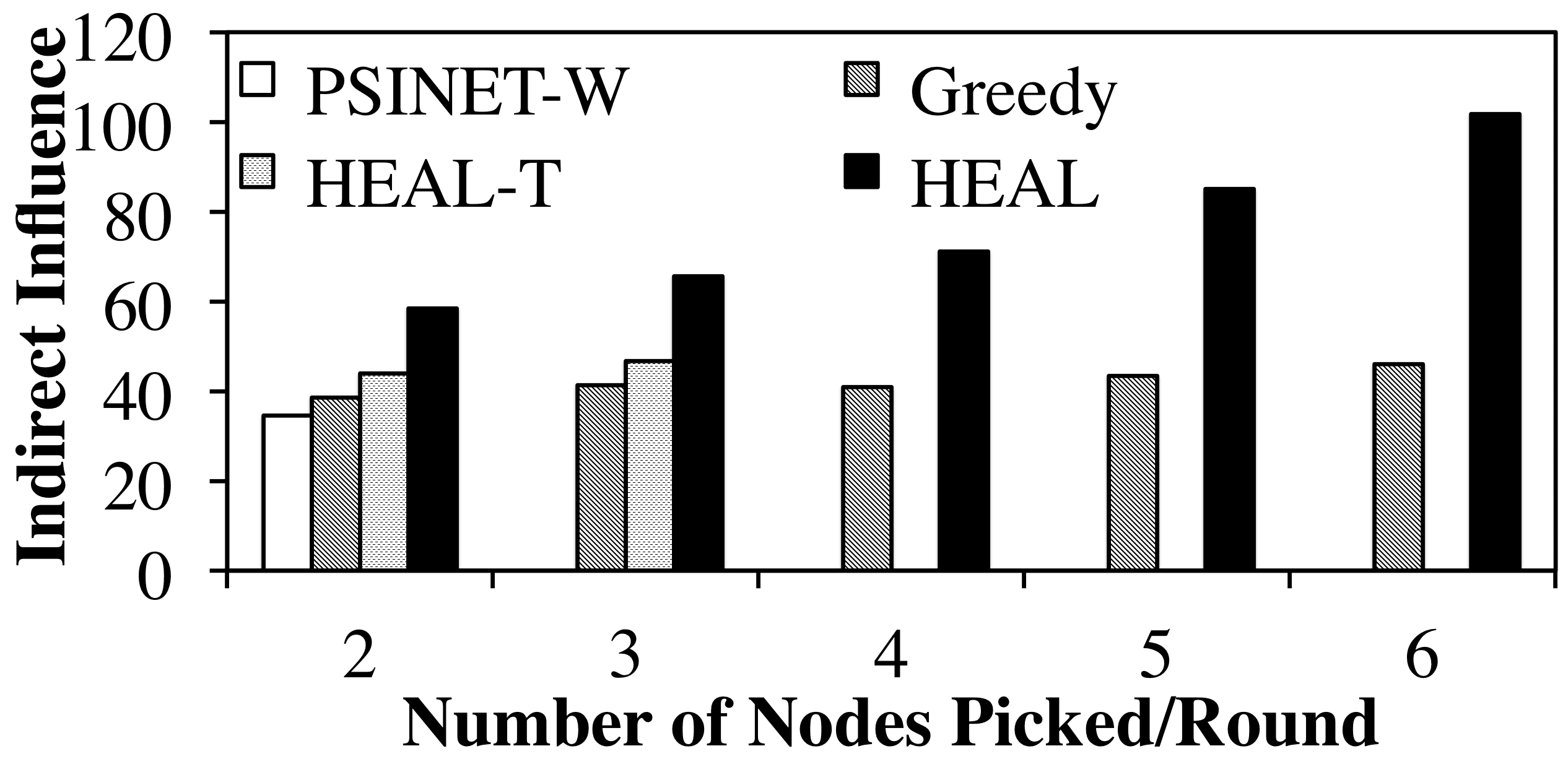}\label{fig:VeniceNet}}
\hspace{2mm}
\subfloat[\small HD Network]{\includegraphics[height=1.5in,width=0.48\columnwidth]{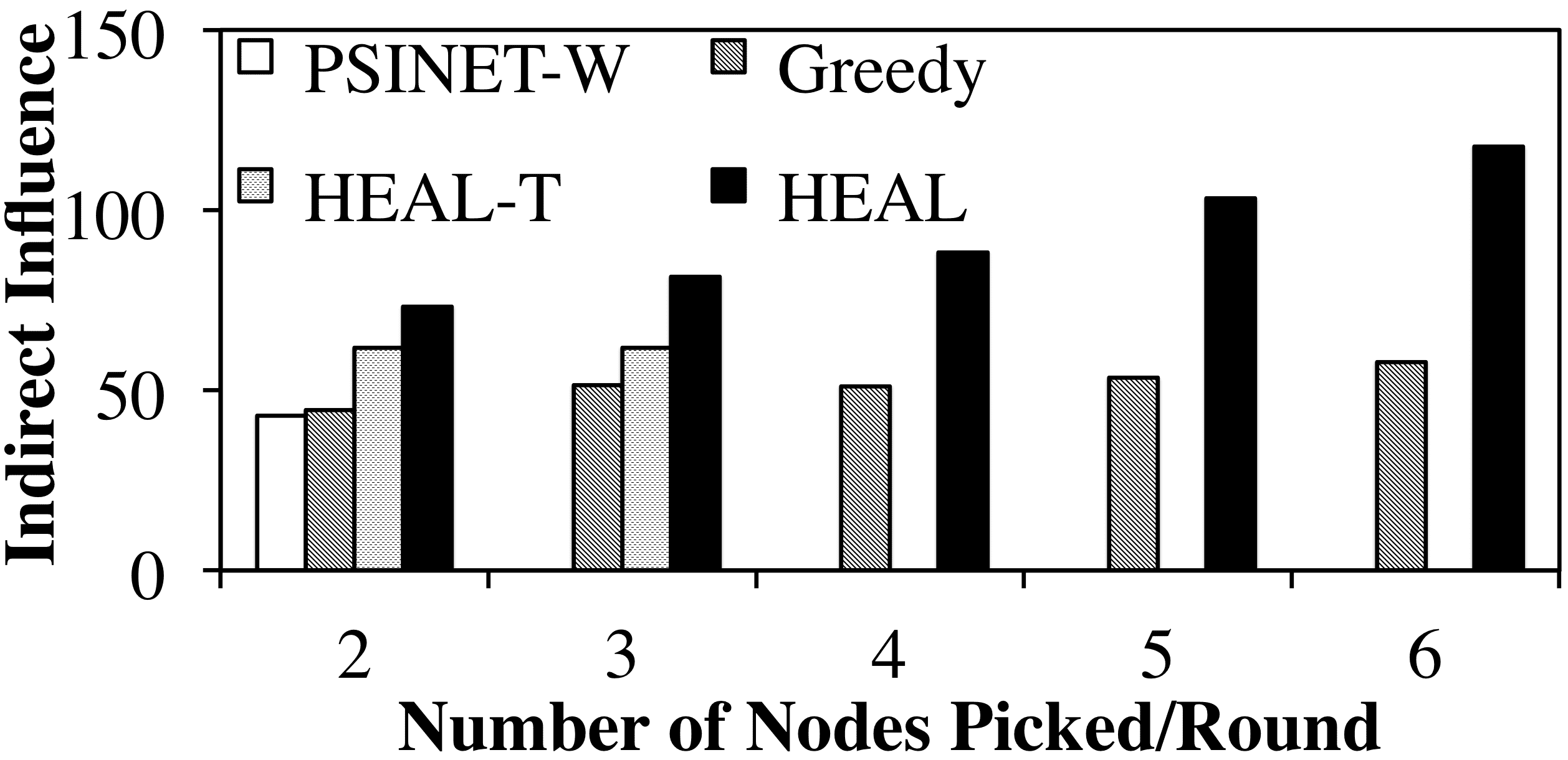}\label{fig:HollywoodNet}}
\caption{Scale up in number of nodes picked per round}
\end{figure}

\textbf{Scale Up Results. } Not only does PSINET-W fail in scaling up to larger network sizes, it even fails to scale-up with increasing number of nodes picked per round (or $K$), on our real-world networks. Figures \ref{fig:VeniceNet} and \ref{fig:HollywoodNet} show the \textit{Indirect Influence} achieved by HEAL, HEAL-T, Greedy and PSINET-W on the VE and HD networks respectively ($T=5$), as we scale up $K$ values. The X-axis shows increasing $K$ values, and the Y-axis shows the \textit{Indirect Influence}. These figures show that (i) PSINET-W and HEAL-T fail to scale up  -- they cannot handle more than $K=2$ and $K=3$ respectively (thereby not fulfilling real world demands); (ii) HEAL outperforms all other algorithms, and the difference between HEAL and Greedy increases linearly with increasing $K$ values. Also, \textit{in the case of $K=6$, HEAL runs in less than $40.12$ seconds on the HD network and $34.4$ seconds on the VE network}.

Thus, Figures \ref{fig:SolQual}, \ref{fig:Runtime}, \ref{fig:VeniceNet} and \ref{fig:HollywoodNet} show that PSINET-W (the best performing algorithm from previous work) fails to scale up with increasing network nodes, and with increasing $K$ values. Even for $K=2$ and moderate sized networks, it runs very slowly. Moreover, HEAL is the best performing algorithm that runs quickly, provides high-quality solutions, and can scale-up to real-world demands. Since only HEAL and Greedy scale up to $K=6$, we now analyze their performance in detail.

\begin{figure}[t]
\subfloat[\small Solution Quality]{\includegraphics[height=1.5in,width=0.48\columnwidth]{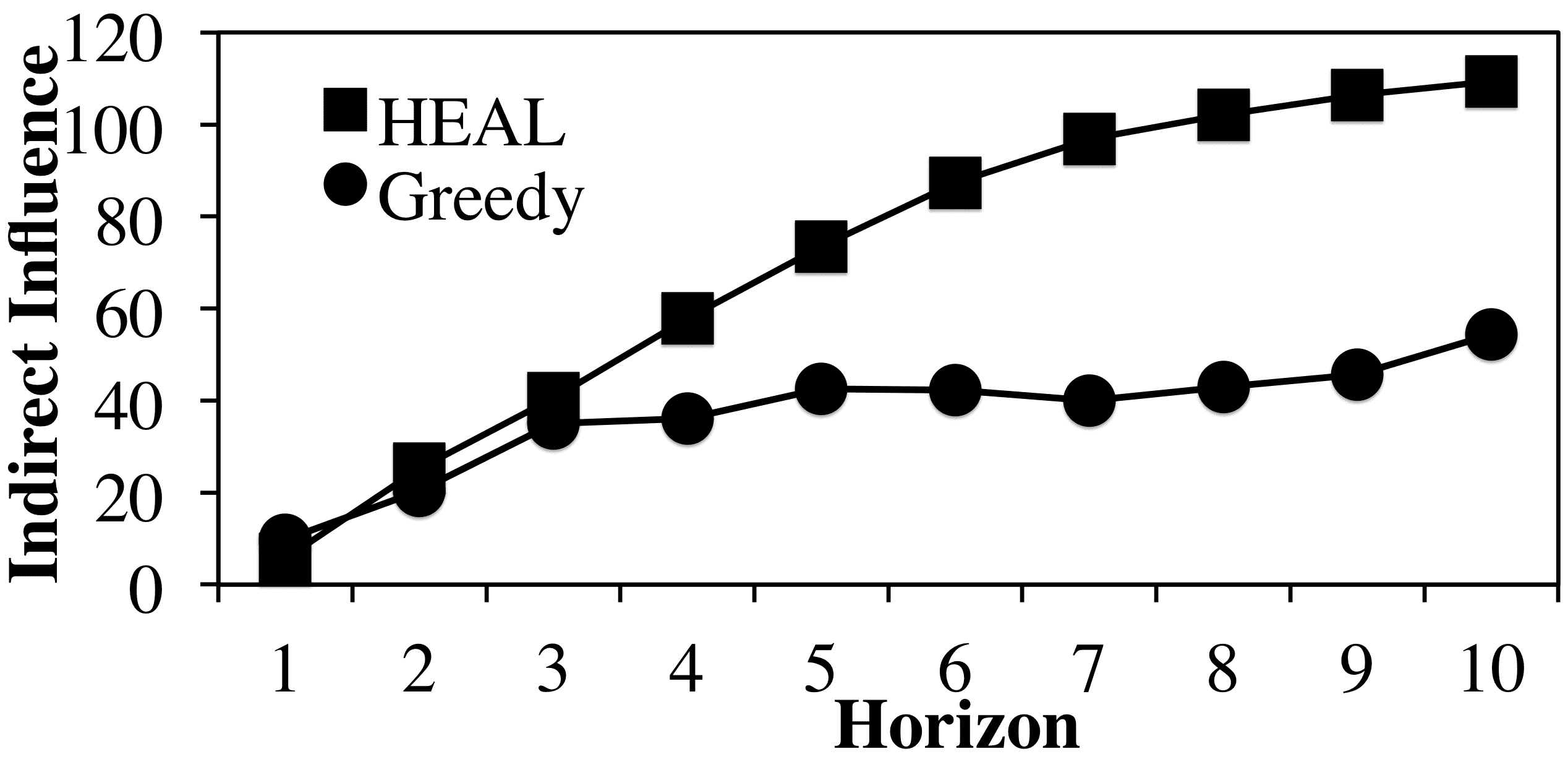}\label{fig:Horizon-HollySolQual}}
\hspace{2mm}
\subfloat[\small Maximum Relative Gain]{\includegraphics[height=1.5in,width=0.48\columnwidth]{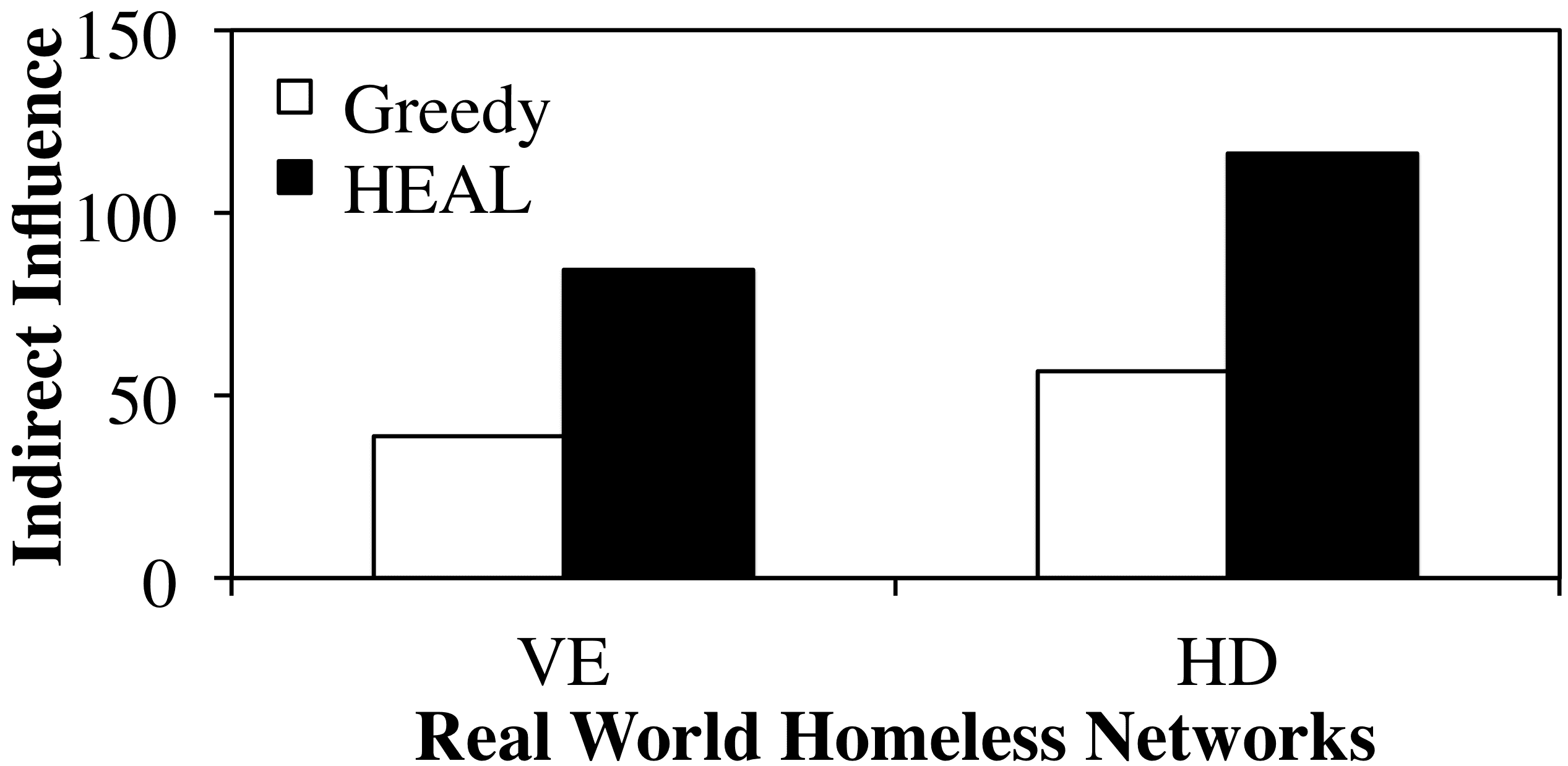}\label{fig:ws-SolQual}}
\caption{Horizon Scale up \& Maximum Gain on HD Network}
\end{figure}

\textbf{Scaling up Horizon. }Figure \ref{fig:Horizon-HollySolQual} shows HEAL and Greedy's \textit{Indirect Influence} in the HD network, with varying horizons (see appendix for VE network results). The X-axis shows increasing horizon values and the Y-axis shows the \textit{Indirect Influence} ($K=2$). This figure shows that the relative difference between HEAL and Greedy increases significantly with increasing $T$ values.

Next, we scale up $K$ values with increased horizon settings to find the maximum attainable solution quality difference between HEAL and Greedy. Figure \ref{fig:ws-SolQual} shows the \textit{Indirect Influence} achieved by HEAL and Greedy (with $K=4$ and $T=10$) on the VE and HD networks. The X-axis shows the two networks and the Y-axis shows the \textit{Indirect Influence}. This figure shows that with these settings, \textit{HEAL achieves $\sim$110\% more Indirect Influence than Greedy (i.e., more than a 2-fold improvement) in the two real-world networks}.


\textbf{HEAL vs Greedy. }Figure \ref{tab:inc-over-greedy} shows the percentage increase (in \textit{Indirect Influence}) achieved by HEAL over Greedy with varying $u(e)$/$p(e)$ values. The columns and rows of Figure \ref{tab:inc-over-greedy} show varying $u(e)$ and $p(e)$ values respectively. The values inside the table cells show the percentage increase (in \textit{Indirect Influence}) achieved by HEAL over Greedy when both algorithms plan using the same $u(e)$/$p(e)$ values. For example, with $p(e)=0.7$ and $u(e)=0.1$, HEAL achieves 45.62\% more \textit{Indirect Influence} than Greedy. This figure shows that \textit{HEAL continues to outperform Greedy across varying $u(e)$/$p(e)$ values}. Thus, on a variety of settings, HEAL dominates Greedy in terms of both \textit{Indirect Influence} and run-time.

\begin{figure}
\begin{center}
\begin{tabular}{|l|c|c|c|}\hline
\theadfont\diagbox[width=4em]{$p(e)$}{$u(e)$}&
\thead{$0.1$}&\thead{$0.2$}&\thead{$0.3$}\\    \hline
$0.7$ & $45.62$ & $44.37$ & $30.85$ \\    \hline
$0.6$ & $48.95$ & $24.56$ & $30$ \\    \hline
$0.5$ & $29.5$ & $55.18$ & $28.21$ \\    \hline
\end{tabular}
\end{center}
\caption{\label{tab:inc-over-greedy} Percentage Increase in HEALER Solution over Greedy}
\end{figure}

\begin{figure}[t]
\subfloat[\small Deviation Tolerance]{{\includegraphics[height=1.5in,width=0.48\columnwidth]{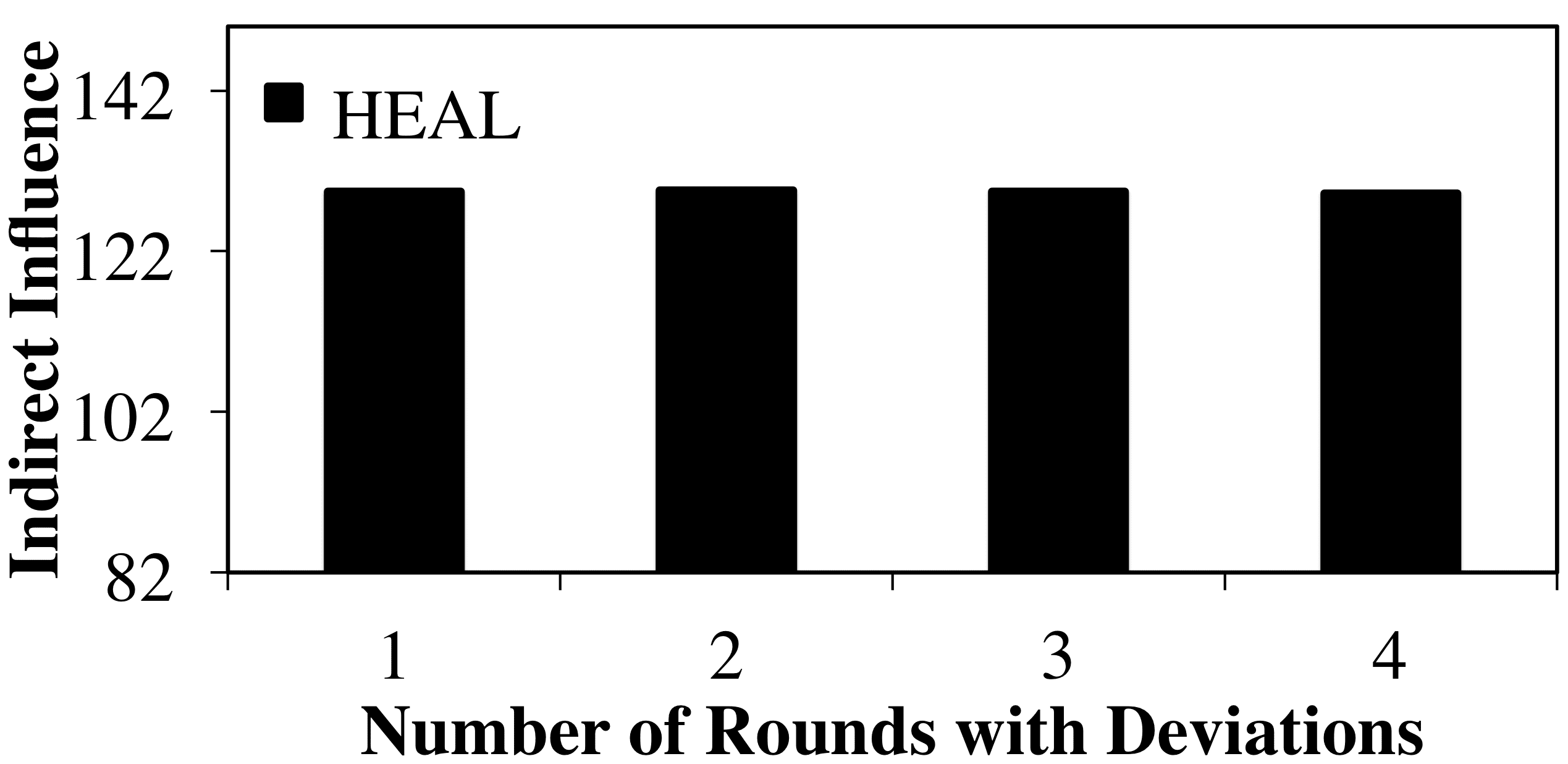}\label{fig:ws-deviation}}}
\hspace{2mm}
\subfloat[\small Artificial Networks]{{\includegraphics[height=1.5in,width=0.48\columnwidth]{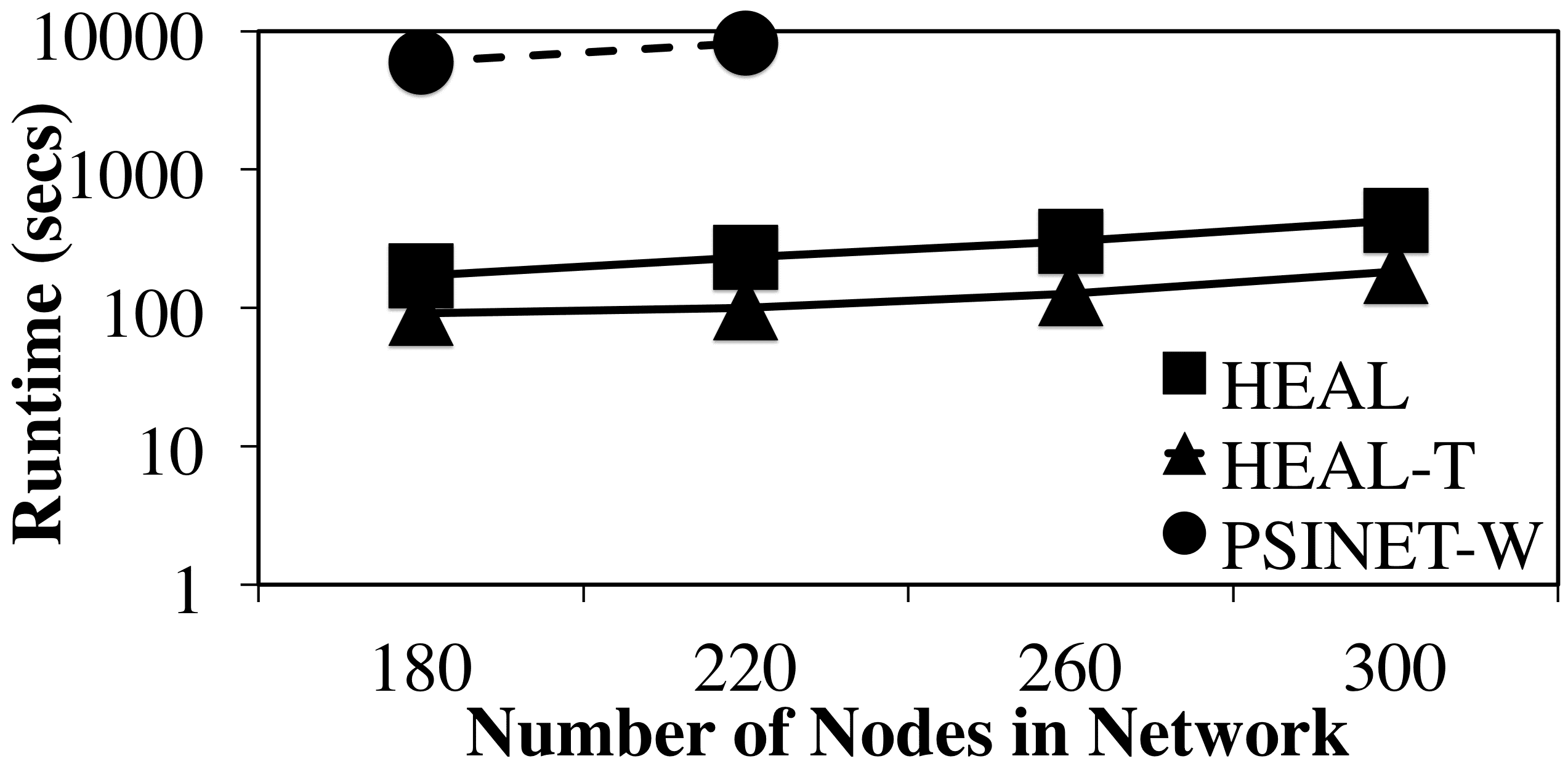}\label{fig:big-runtime}}}
\caption{ Deviation Tolerance \& Results on Artificially Generated Networks}
\end{figure}

\textbf{Deviation Tolerance. } We show HEAL's tolerance to deviation by replacing a fixed number of actions recommended by HEAL with randomly selected actions. Figure \ref{fig:ws-deviation} shows the variation in \textit{Indirect Influence} achieved by HEAL ($K=4$,$T=10$) with increasing number of random deviations from the recommended actions. The X-axis shows increasing number of deviations and the Y-axis shows the \textit{Indirect Influence}. For example, when there were 2 random deviations (i.e., two recommended actions were replaced with random actions), HEAL achieves 100.23 \textit{Indirect Influence}. This figure shows that HEAL is highly deviation-tolerant.

\begin{figure}
\begin{center}
\begin{tabular}{|l|c|c|c|}\hline
\theadfont\diagbox[width=4em]{$p(e)$}{$u(e)$}&
\thead{$0.1$}&\thead{$0.2$}&\thead{$0.3$}\\    \hline
$0.7$ & $24.42$ & $21.02$ & $16.85$ \\    \hline
$0.6$ & $0.0$ & $18.26$ & $12.46$ \\    \hline
$0.5$ & $11.58$ & $10.53$ & $8.11$ \\    \hline
\end{tabular}
\end{center}
\caption{\label{tab:sensitive} Percentage Loss in HEAL Solution on HD Network}
\end{figure}


\textbf{Sensitivity Analysis. } Finally, we test the robustness of HEAL's solutions in the HD network, by allowing for error in HEAL's understanding of $u(e)$/$p(e)$ values. We consider the case that $u(e)=0.1$ and $p(e)=0.6$ values that HEAL plans on, are wrong. Thus, HEAL plans its solutions using $u(e)=0.1$ and $p(e)=0.6$, but those solutions are evaluated on different (correct) $u(e)$/$p(e)$ values to get \textit{estimated solutions}. These \textit{estimated solutions} are compared to \textit{true solutions} achieved by HEAL if it planned on the correct $u(e)$/$p(e)$ values. Figure \ref{tab:sensitive} shows the percentage difference (in \textit{Indirect Influence}) between the \textit{true} and \textit{estimated} solutions, with varying $u(e)$ and $p(e)$ values. For example, when HEAL plans its solutions with wrong $u(e)=0.1$/$p(e)=0.6$ values (instead of correct $u(e)=0.3$/$p(e)=0.5$ values), it suffers a 8.11\% loss. This figure shows that HEAL is relatively robust to errors in its understanding of $u(e)$/$p(e)$ values, as it only suffers an average-case loss of $\sim$ 15\%.

\section{Conclusion}
This chapter presents HEALER, a software agent that recommends sequential intervention plans for use by homeless shelters, who organize these interventions to raise awareness about HIV among homeless youth. HEALER's sequential plans (built using knowledge of social networks of homeless youth) choose intervention participants strategically to maximize influence spread, while reasoning about uncertainties in the network. While previous work presents influence maximizing techniques to choose intervention participants, they do not address three real-world issues: (i) they \textit{completely fail} to scale up to real-world sizes; (ii) they do not handle deviations in execution of intervention plans; (iii) constructing real-world social networks is an expensive process. HEALER handles these issues via four major contributions: (i) HEALER casts this influence maximization problem as a POMDP and solves it using a novel planner which scales up to previously unsolvable real-world sizes; (ii) HEALER allows shelter officials to modify its recommendations, and updates its future plans in a deviation-tolerant manner; (iii) HEALER constructs social networks of homeless youth at low cost, using a Facebook application.

\chapter{Real World Deployment of Influence Maximization Algorithms}
\label{chapter:PILOT}
This chapter focuses on a topic that is insufficiently addressed in the literature, i.e., challenges faced in transitioning agents from an emerging phase in the lab, to a deployed application in the field. Specifically, we focus on challenges faced in transitioning HEALER and DOSIM \cite{wilder2017uncharted}, two agents for social influence maximization, which assist service providers in maximizing HIV awareness in real-world homeless-youth social networks.  While prior chapters have shown that these agents/algorithms have promising performance in simulation, this chapter illustrates that transitioning these algorithms from the lab into the real-world is not straightforward, and outlines three major lessons. First, it is important to conduct real-world pilot tests; indeed, due to the health-critical nature of the domain and complex influence spread models used by these algorithms, it is important to conduct field tests to ensure the real-world usability and effectiveness of these algorithms. We present results from three real-world pilot studies, involving 173 homeless youth in Los Angeles. These are the \textit{first such pilot studies} which provide head-to-head comparison of different algorithms for social influence maximization, including a comparison with a baseline approach. Second, we present analyses of these real-world results, illustrating the strengths and weaknesses of different influence maximization approaches we compare. Third, we present research and deployment challenges revealed in conducting these pilot tests, and propose solutions to address them.

\section{Pilot Study Pipeline}\label{sec:4}
Starting in Spring 2016, we conducted three different pilot studies \cite{yadav2018ijcai,yadav2017influence} at two service providers (see Figures \ref{fig:shelter1} and \ref{fig:shelter2}) in Los Angeles, over a seven month period. Each pilot study recruited a unique network of youth. Recall that these pilot studies serve three purposes. First, they help in justifying our assumptions about whether peer leaders actually spread HIV information in their social network, and whether they provide meaningful information about the social network structure (i.e., observations) during the intervention training. Second, these pilots help in exposing unforeseen challenges, which need to be solved convincingly before these agents can be deployed in the field. Third, they provide a head-to-head comparison of two different software agent approaches for social influence maximization, including a comparison with a baseline approach. 

\begin{figure}[htp]
\subfloat[\small Safe Place for Youth]{\includegraphics[height=1.5in,width=0.47\columnwidth]{Figures/1.jpg}\label{fig:shelter1}}
\hspace{2mm}
\subfloat[\small Emergency Resource Shelf]{\includegraphics[height=1.5in,width=0.47\columnwidth]{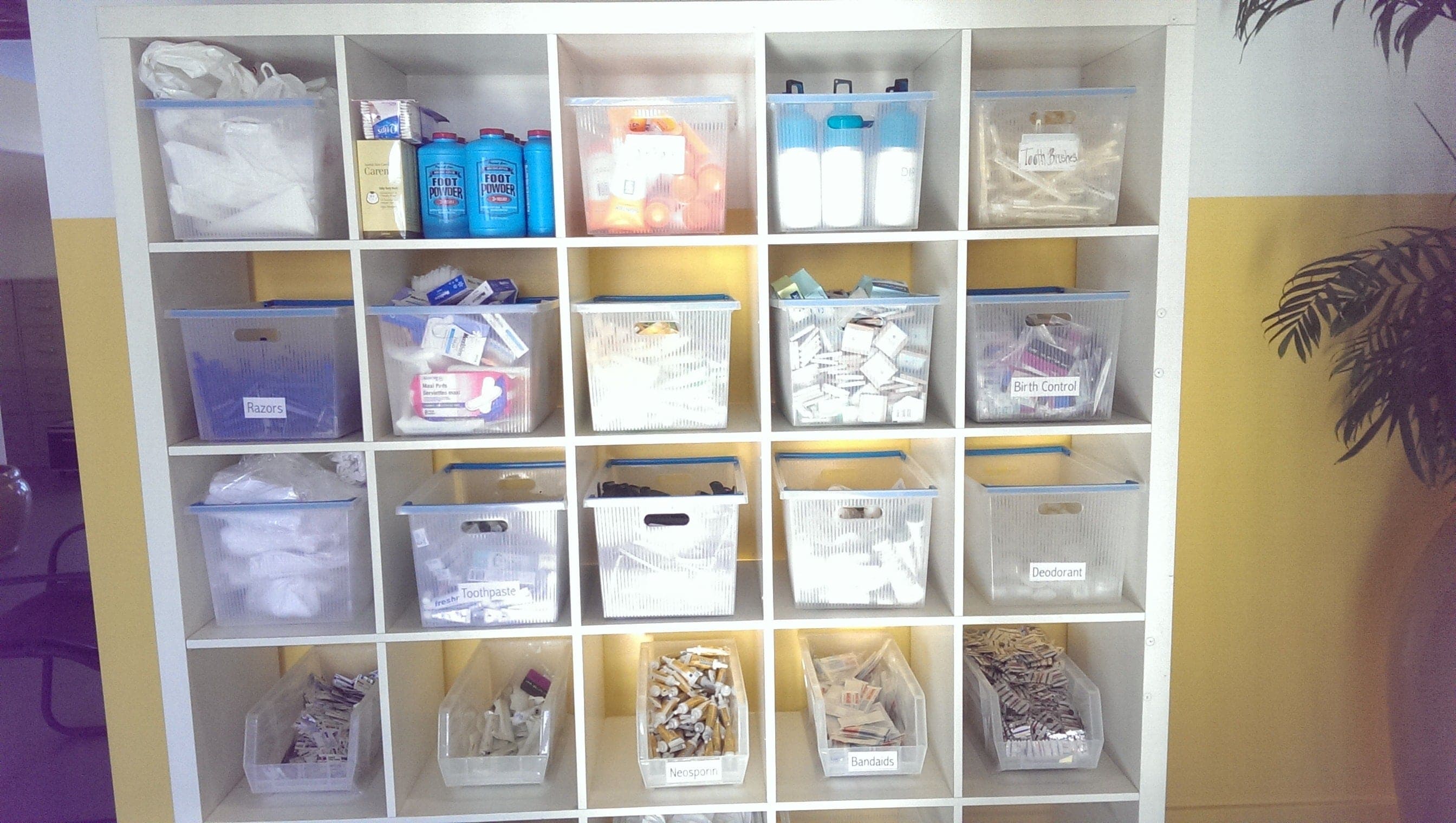}\label{fig:shelter2}}\\
\subfloat[\small Desks for Intervention Training]{\includegraphics[height=1.5in,width=0.47\columnwidth]{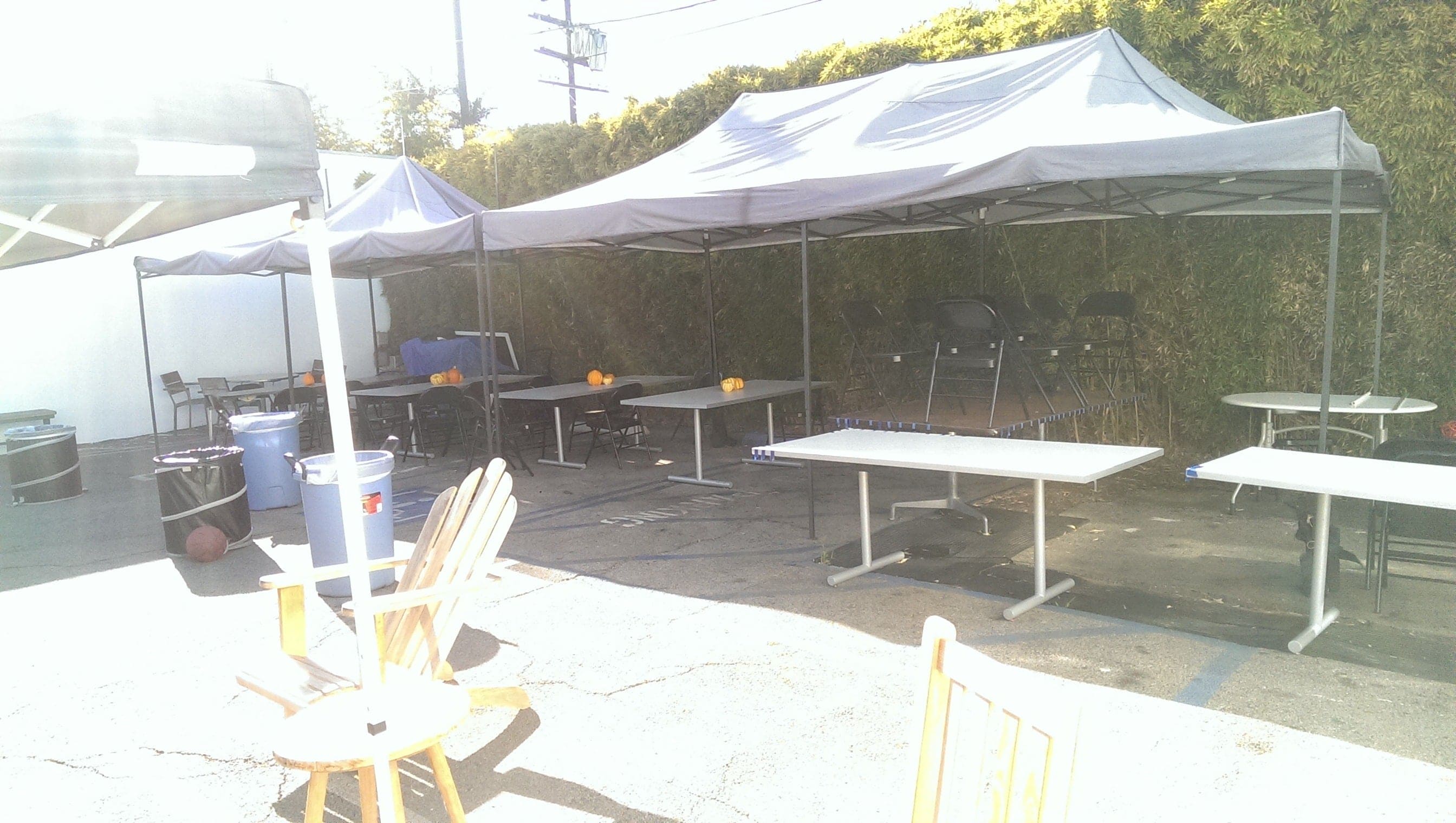}\label{fig:shelter3}}
\hspace{2mm}
\subfloat[\small Computer Kiosks at Homeless Shelter]{\includegraphics[height=1.5in,width=0.47\columnwidth]{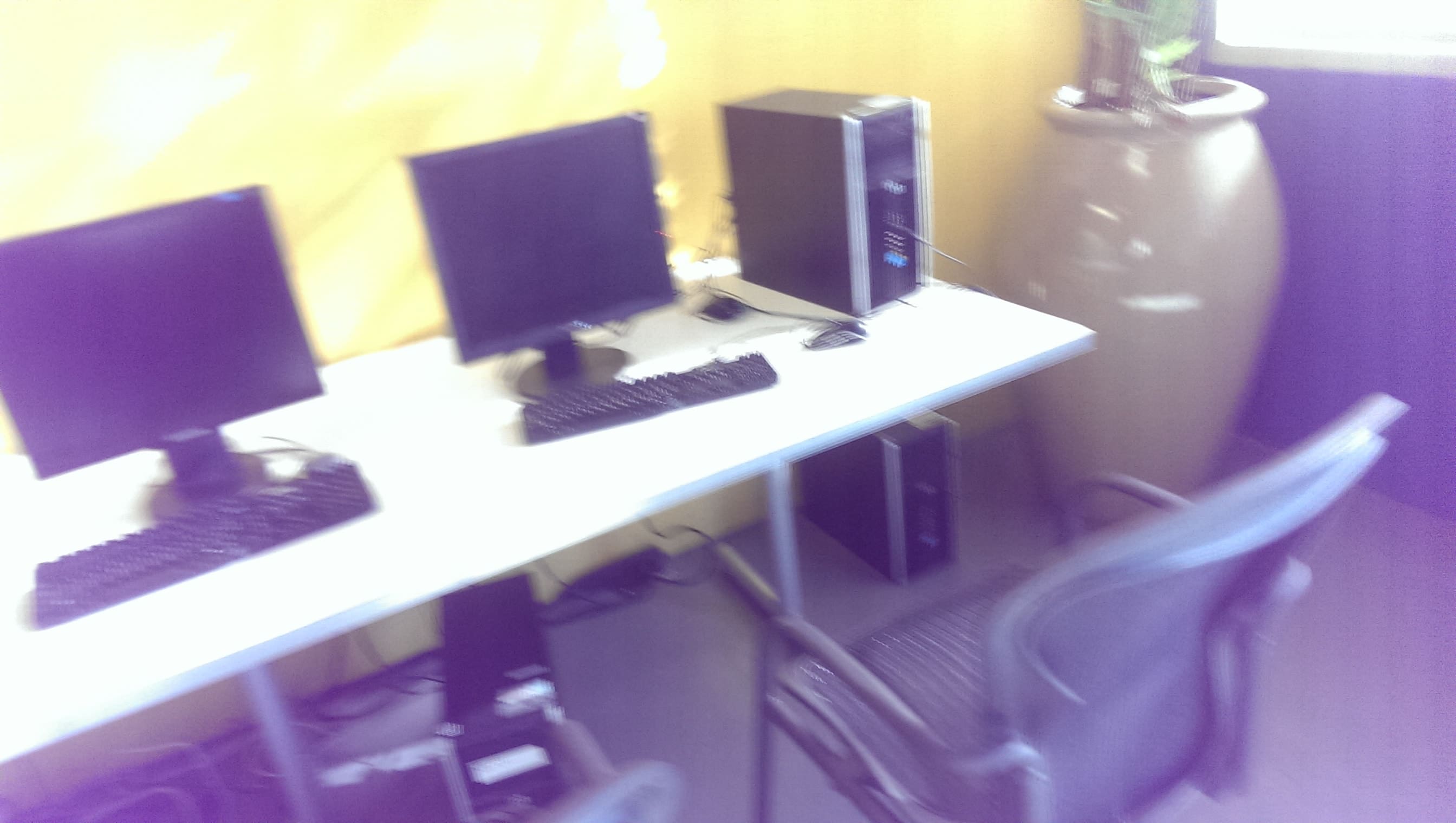}\label{fig:shelter4}}\\

\caption{\small Facilities at our Collaborating Service Providers}
\end{figure}

Each of these pilot studies had a different intervention mechanism, i.e., a different way of selecting actions (or a set of $K$ peer leaders). The first and second studies used HEALER and DOSIM (respectively) to select actions, whereas the third study served as the control group, where actions were selected using Degree Centrality (i.e., picking $K$ nodes in order of decreasing degrees). We chose Degree Centrality (DC) as the control group mechanism, because this is the current modus operandi of service providers in conducting these network based interventions \cite{valente2012network}.

\begin{figure}[t]
\center{\includegraphics[scale=.2]
{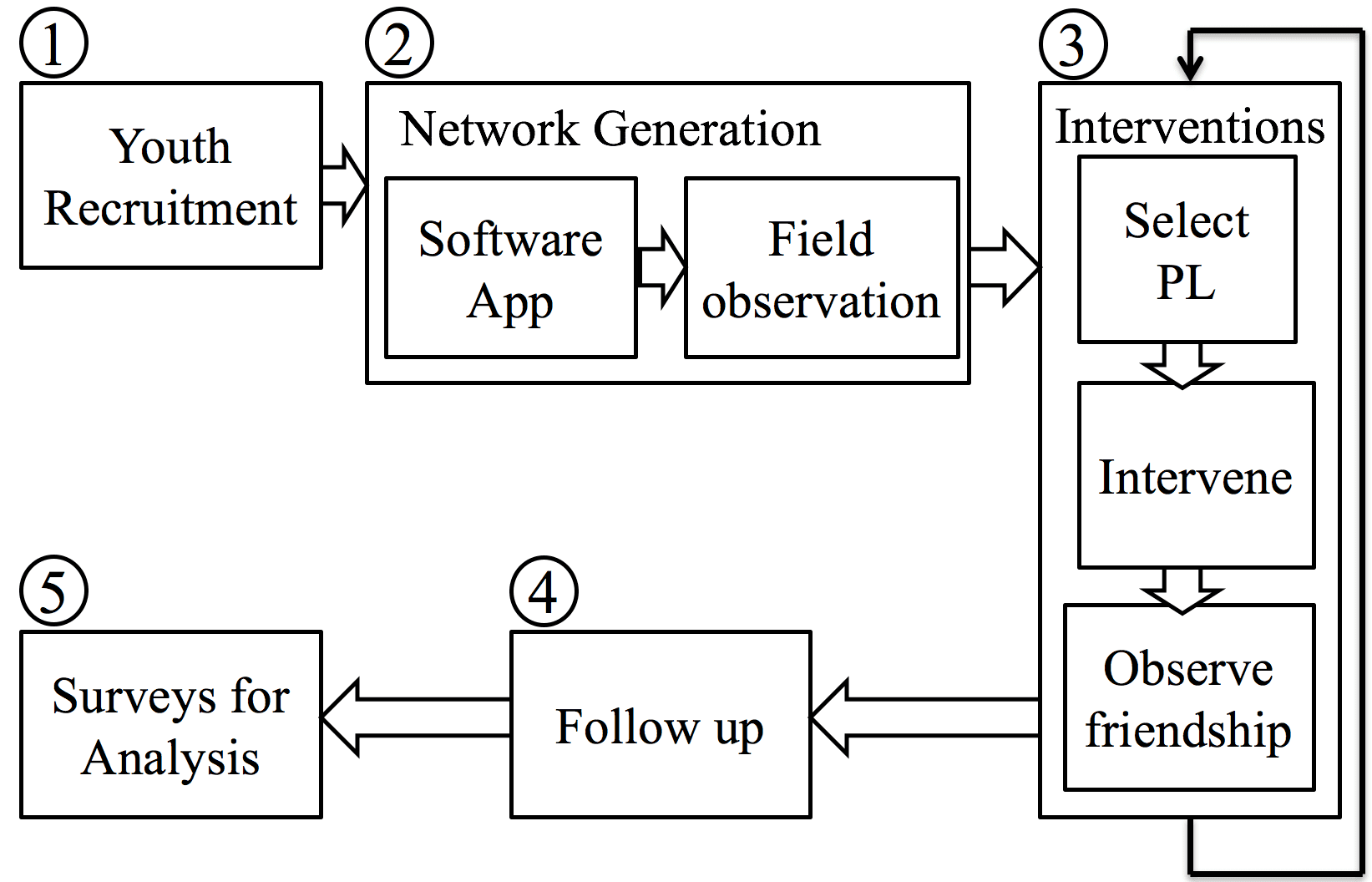}}
\caption{\label{fig:pipe} Real World Pilot Study Pipeline}
\end{figure} 

\subsection{Pilot Study Process} The pilot study process consists of five sequential steps. Figure \ref{fig:pipe} illustrates these five steps.

\begin{enumerate}

\item \textbf{Recruitment}: First, we recruit homeless youth from a service provider into our study. We provide youth with general information about our study, and our expectations from them (i.e., if selected as a peer leader, they will be expected to spread information among their peers). The youth take a 20 minute baseline survey, which enables us to determine their current risk-taking behaviors (e.g., they are asked about the last time they got an HIV test, etc.).  Every youth is given a 20 USD gift card as compensation for being a part of the pilot study. All study procedures were approved by our Institutional Review Board.

\item \textbf{Network Generation}: After recruitment, the friendship based social network that connects these homeless youth is generated. We rely on two information sources to generate this network: (i) online contacts of homeless youth; and (ii) field observations made by the authors and service providers. To expedite the network generation phase, online contacts of homeless youth are used (via a software application that the youth are asked to use) to build a first approximation of the real-world social network of homeless youth. Then, this network is refined using field observations (about additional real-world friendships) made by the authors and the service providers. All edges inferred in this manner are assumed to be \textit{certain edges}. More information on uncertain edges is provided later.

\item \textbf{Interventions}: Next, the generated network is used by the software agents to select actions (i.e., $K$ peer leaders) for $T$ stages. In each stage, an action is selected using the pilot's intervention strategy. The $K$ peer leaders of this chosen action are then trained as peer leaders (i.e., informed about HIV) by pilot study staff during the intervention. These peer leaders also reveal more information (i.e., provide observation) about newer friendships which we did not know about. These friendships are incorporated into the network, so that the agents can select better actions in the next stage of interventions. Every peer leader is given a 60 USD gift card.

\item \textbf{Follow Up}: The follow up phase consists of meetings, where the peer leaders are asked about any difficulties they faced in talking to their friends about HIV. They are given further encouragement to keep spreading HIV awareness among their peers. These follow-up meeting occur on a weekly basis, for a period of one month after Step 3 ends.

\item \textbf{Analysis}: For analysis, we conduct in-person surveys, one month after all interventions have ended. Every youth in our study is given a 25 USD gift card to show up for these surveys. During the surveys, they are asked if some youth from within the pilot study talked to them about HIV prevention methods, after the pilot study began. Their answer helps determine if information about HIV reached them in the social network or not. Thus, these surveys are used to find out the number of youth who got informed about HIV as a result of our interventions. Moreover, they are asked to take the same survey about HIV risk that they took during recruitment. These post-intervention surveys enable us to compare HEALER, DOSIM and DC in terms of information spread (i.e., how successful were the agents in spreading HIV information through the social network) and behavior change (i.e., how successful were the agents in causing homeless youth to test for HIV), the two major metrics that we use in our evaluation section. 
\end{enumerate}

We provide these behavior change results in order to quantify the true impact of these social influence maximization agents in the homeless youth domain. In these results, we measure behavior change by asking youth if they have taken an HIV test at baseline and repeating this question during the follow up surveys.  If the youth reported taking an HIV test at one month (after interventions) but not at baseline and that youth also reported getting informed about HIV, we attribute this behavior change to our intervention. This allows us to measure whether our interventions led to a reduction in risk attitudes.

\textbf{Uncertain network parameters} While there exist many link prediction techniques \cite{kim2011network} to infer \textit{uncertain edges} in social networks, the efficacy of these techniques is untested on homeless youth social networks. Therefore, we took a simpler, less "risky" approach -- each edge \textit{not created} during the network generation phase (i.e., Step 2 above) was added to the network as an \textit{uncertain edge}.  Thus, after adding these uncertain edges, the social network in each pilot study became a completely connected network, consisting of \textit{certain edges} (inferred from Step 2), and \textit{uncertain edges}. The existence probability on each uncertain edge was set to $u=0.01$. Our approach to adding uncertain edges ensures that no potential friendship is missed in the social network because of our lack of accurate knowledge.

\begin{figure}[t]
\center{\includegraphics[scale=.1]
{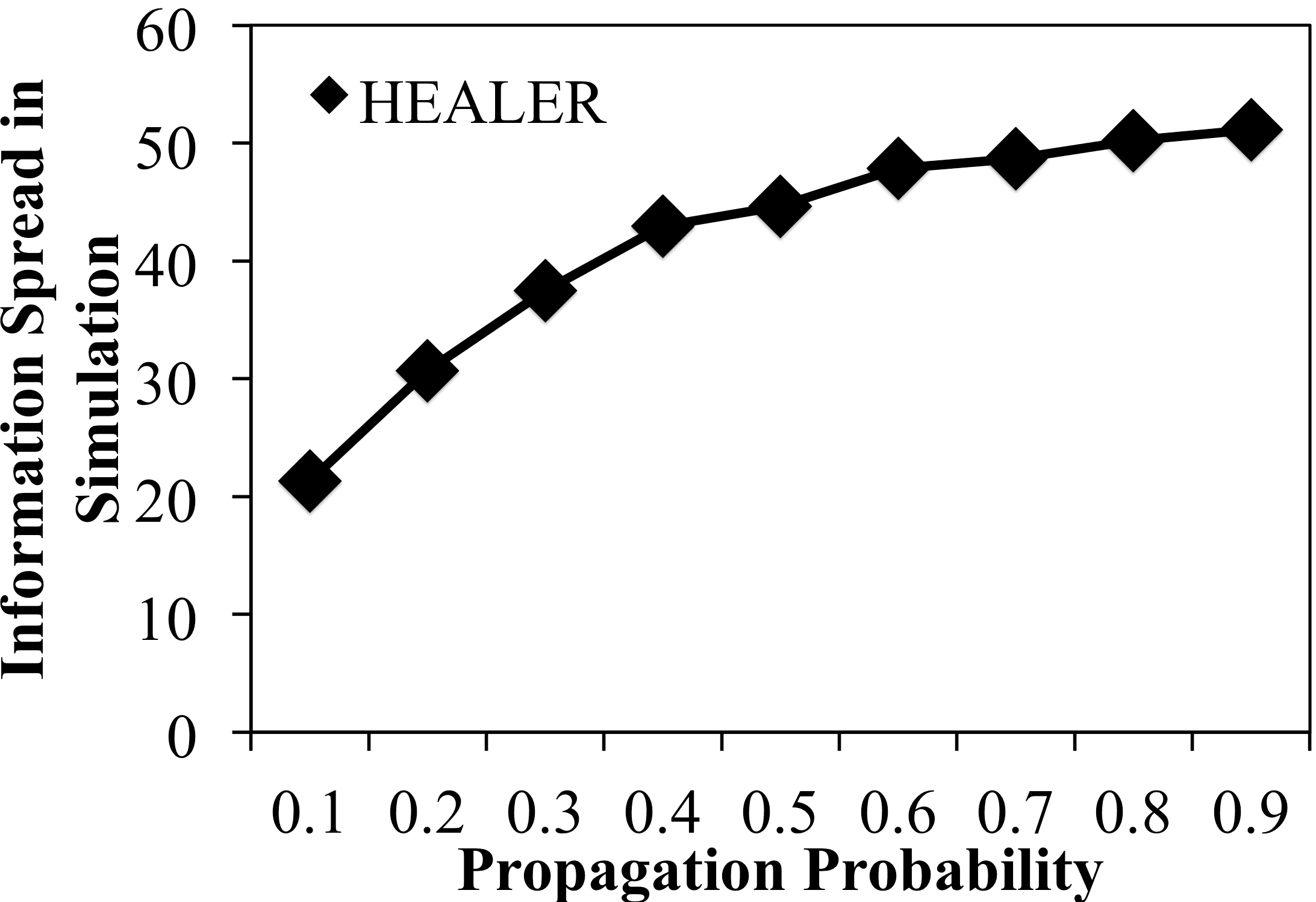}}
\caption{\label{fig:pilot6} Information Spread with $p_e$ on HEALER's Pilot Network}
\end{figure}

Getting propagation probabilities ($p_e$) values was also challenging. In HEALER's pilot, service providers estimated that the true $p_e$ value would be somewhere around 0.5. Since the exact value was unknown, we assumed an interval of $[0.4, 0.8]$ and simulated HEALER's performance with $p_e$ values in this range.  Figure \ref{fig:pilot6} shows how information spread achieved by HEALER on its pilot study network is relatively stable in simulation for $p_e$ values around 0.5. The Y-axis shows the information spread in simulation and the X-axis shows increasing $p_e$ values. This figure shows that information spread achieved by HEALER varied by $\sim$11.6\% with $p_e$ in the range $[0.4, 0.8]$. Since influence spread is relatively stable in this range, we selected $p_e=0.6$ (the mid point of $[0.4,0.8]$) on all network edges. Later, we provide ex-post justification for why $p_e=0.6$ was a good choice, atleast for this pilot study.

In DOSIM's pilot, we did not have to deal with the issue of assigning accurate $p_e$ values to edges in the network. This is because DOSIM can work with intervals in which the exact $p_e$ is assumed to lie. For the pilot study, we used the same interval of $[0.4, 0.8]$ to run DOSIM. Finally, the control group pilot study did not require finding $p_e$ values, as peer leaders were selected using Degree Centrality, which does not require knowledge of $p_e$.

\section{Results from the Field}
We now provide results from all three pilot studies \cite{rice2018piloting}. In each study, three interventions were conducted (or, $T=3$), i.e., Step 3 of the pilot study process (Figure \ref{fig:pipe}) was repeated three times. The actions (i.e., set of $K$ peer leaders) were chosen using intervention strategies (policies) provided by HEALER \cite{yadav2016using,yadav2017ibm,yadav2017explanation,soriano2016simultaneous}, DOSIM \cite{wilder2017uncharted}, and Degree Centrality (DC) in the first, second and third pilot studies, respectively. Recall that we provide comparison results on two different metrics. First, we provide results on information spread, i.e., how well different software agents were able to spread information about HIV through the social network. Second, even though HEALER and DOSIM do not explicitly model behavior change in their objective function (both maximize the information spread in the network), we provide results on behavior change among homeless youth, i.e., how successful were the agents in inducing behavior change among homeless youth.

\begin{figure}[ht]
\center{\includegraphics[scale=.15]
{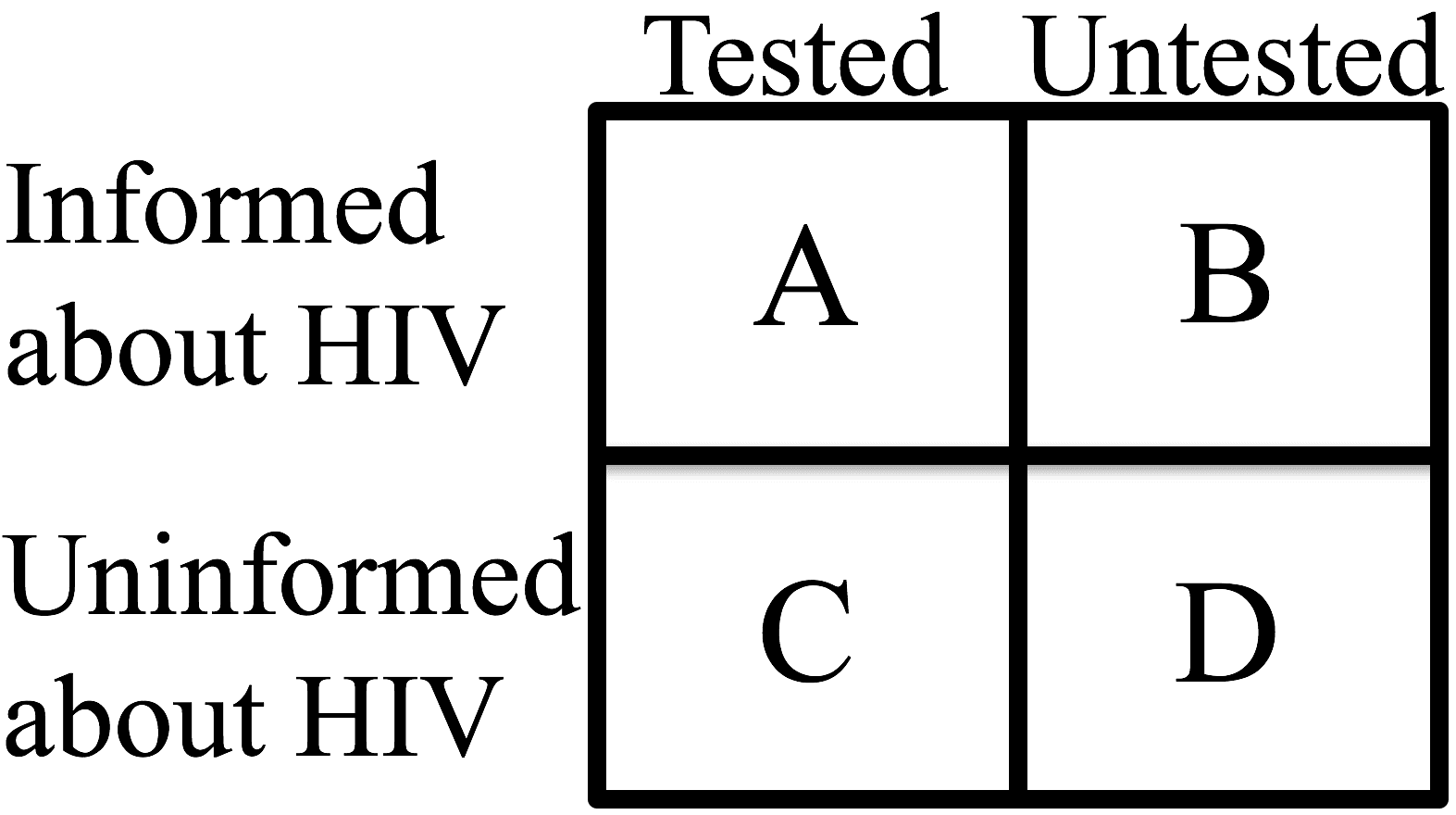}}
\caption{\label{fig:venn} Set of Surveyed Non Peer-Leaders}
\end{figure}

Figure \ref{fig:venn} shows a Venn diagram that explains the results that we collect from the pilot studies. To begin with, we exclude peer leaders from all our results, and focus only on non peer-leaders. This is done because peer leaders cannot be used to differentiate the information spread (and behavior change) achieved by HEALER, DOSIM and DC. In terms of information spread, all peer leaders are informed about HIV directly by study staff in the intervention trainings. In terms of behavior change, the proportion of peer leaders who change their behavior does not depend on the strategies recommended by HEALER, DOSIM and DC. Thus, Figure \ref{fig:venn} shows a Venn diagram of the set of all non peer-leaders (who were surveyed at the end of one month). This set of non peer-leaders can be divided into four quadrants based on (i) whether they were informed about HIV or not (by the end of one-month surveys in Step 5 of Figure \ref{fig:pipe}); and (ii) whether they were already tested for HIV at baseline (i.e., during recruitment, they reported that they had got tested for HIV in the last six months) or not.

For information spread results, we report on the percentage of youth in this big rectangle, who were informed about HIV by the end of one month (i.e., boxes A+B as a fraction of the big box). For behavior change results, we exclude youth who were already tested at baseline (as they do not need to undergo ``\textit{behavior change}", because they are already exhibiting desired behavior of testing). Thus, we only report on the percentage of \textit{untested informed youth}, (i.e., box B), who now tested for HIV (i.e., changed behavior) by the end of one month (which is a fraction of youth in box B). We do this because we can only attribute conversions (to testers) among youth in box B (Figure \ref{fig:venn}) to strategies recommended by HEALER and DOSIM (or the DC baseline). For example, non peer-leaders in box D who convert to testers (due to some exogenous reasons) cannot be attributed to HEALER or DOSIM's strategies (as they converted to testers \textit{without} getting HIV information).

\begin{figure}
\begin{center}
\begin{tabular}{|l|c|c|c|}
    \hline
     & HEALER & DOSIM & DC\\
    \hline
    Youth Recruited & 62 & 56 & 55\\
    \hline
    PL Trained & 17.7\% & 17.85\% & 20\%\\
    \hline
    Retention \% & 73\% & 73\% & 65\%\\
    \hline
    Avg. Observation Size & 16 & 8 & 15\\
    \hline
\end{tabular}
\end{center}
\caption{\label{tab:detail} Logistic Details of Different Pilot Studies}
\end{figure}

\textbf{Study Details} Figure \ref{tab:detail} shows details of the pilot studies. This figure shows that the three pilots had fairly similar conditions as (i) all three pilots recruited $\sim$60 homeless youth; (ii) peer leader training was done on 15-20\% of these youth, which is recommended in social sciences literature \cite{rice2010positive}; and (iii) retention rates of youth (i.e., percentage of youth showing up for post-intervention surveys)  were fairly similar ($\sim$70\%) in all three pilots. This figure also shows that peer leaders provided information about 13 uncertain friendships on average in every intervention stage (across all three pilot studies), which validates HEALER and DOSIM's assumption that peer leaders provide \textit{observations} about friendships \cite{wilder2017uncharted,yadav2016using}.

\begin{figure}[ht]
\subfloat[\small Comparison of Information Spread Among Non Peer-Leaders]{\includegraphics[height=1.7in,width=0.48\columnwidth]{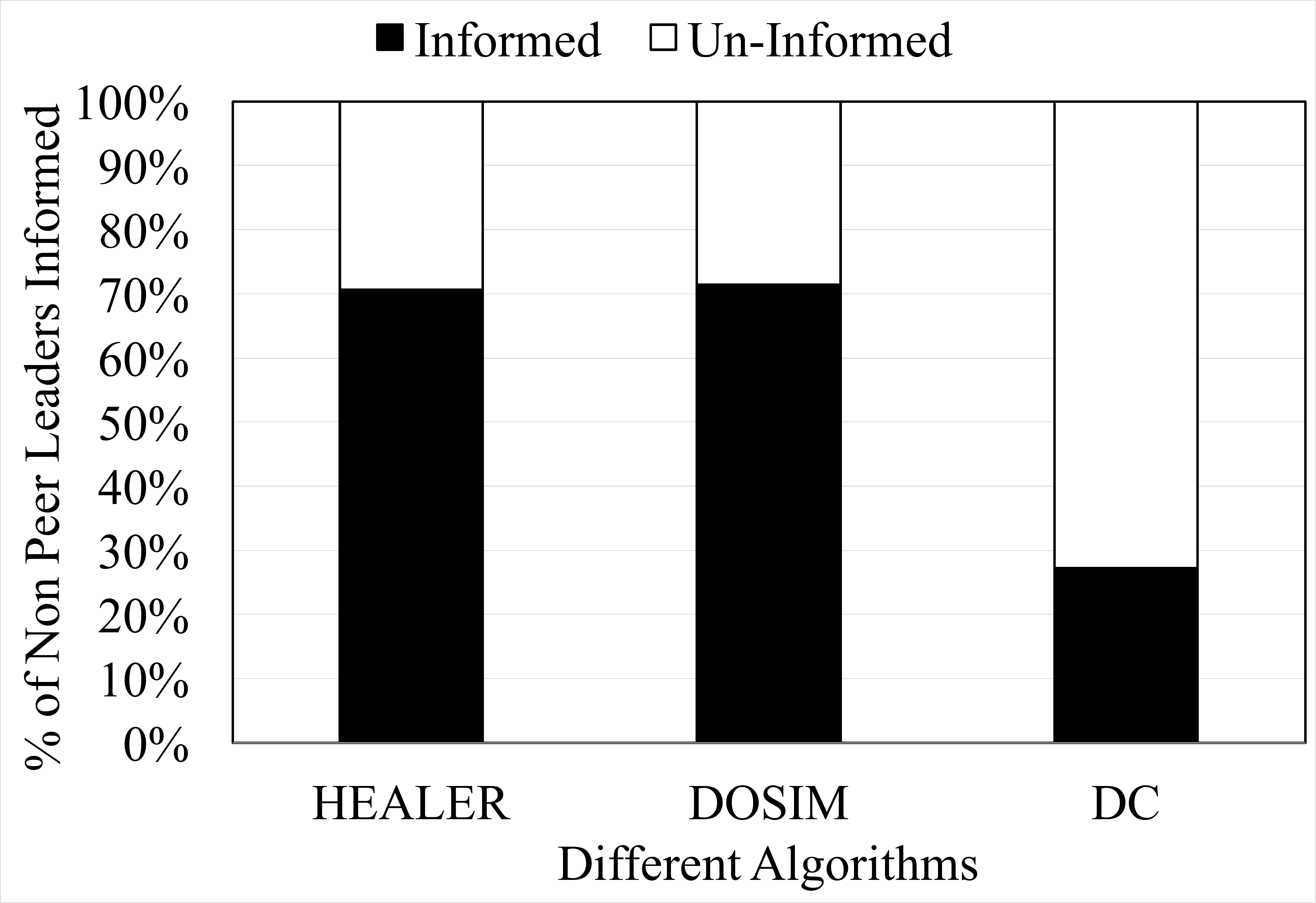}\label{fig:inf-spread}}
\hspace{2mm}
\subfloat[\small \% of edges between PL]{\includegraphics[height=1.7in,width=0.48\columnwidth]{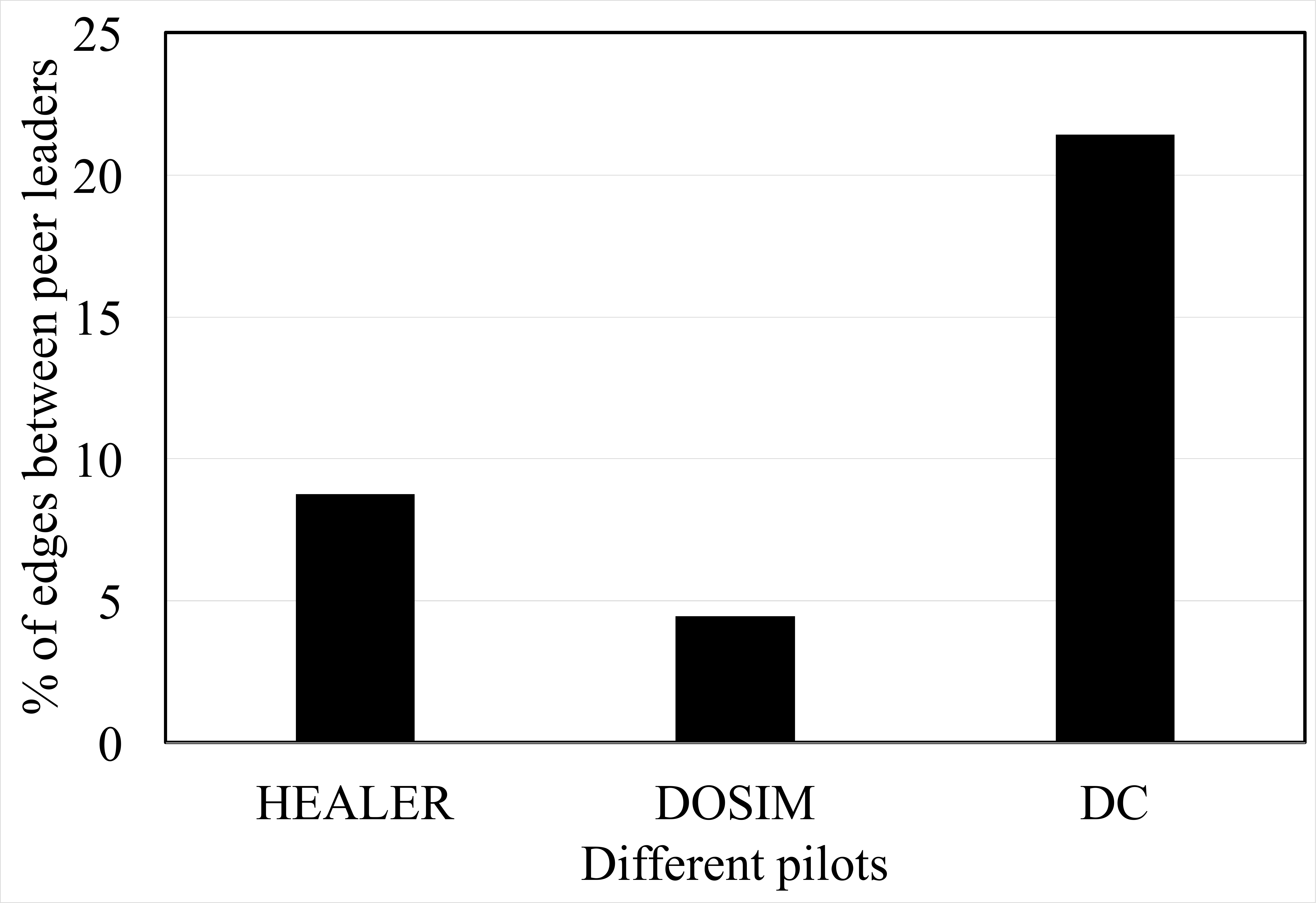}\label{fig:edge}}
\caption{Information Spread Comparison \& Analysis}
\end{figure}

\textbf{Information Spread} Figure \ref{fig:inf-spread} compares the information spread achieved by HEALER, DOSIM and DC in the pilot studies. The X-axis shows the three different intervention strategies and the Y-axis shows the percentage of non-peer-leaders to whom information spread (box A+B as a percentage of total number of non-peer leaders in Figure \ref{fig:venn}). This figure shows that PL chosen by HEALER (and DOSIM) are able to spread information among $\sim$70\% of the non peer-leaders in the social network by the end of one month. Surprisingly, PL chosen by DC were only able to inform $\sim$27\% of the non peer-leaders. This result is surprising, as it means that \textit{HEALER and DOSIM's strategies were able to improve over DC's information spread by over 160\%}. We now explain reasons behind this significant improvement in information spread achieved by HEALER and DOSIM (over DC).

Figure \ref{fig:edge} illustrates a big reason behind DC's poor performance. The X-axis shows different pilots and the Y-axis shows what percentage of network edges were \textit{redundant}, i.e., they connected two peer leaders. Such edges are \textit{redundant}, as both its nodes (peer leaders) already have the information. This figure shows that redundant edges accounted for only 8\% (and 4\%) of the total edges in HEALER (and DOSIM's) pilot study. On the other hand, 21\% of the edges in DC's pilot study were redundant. Thus, DC's strategies picks PL in a way which creates a lot of redundant edges, whereas HEALER picks PL which create only 1/3 times the number of redundant edges. DOSIM performs best in this regard, by selecting nodes which creates the fewest redundant edges ($\sim5$X less than DC, and even 2X less than HEALER), and is the key reason behind its good performance in Figure \ref{fig:inf-spread}. Concomitantly to the presence of redundant edges, HEALER also spreads out its PL selection across different communities within the homeless youth network, that also aids in information spreading, as discussed below.

\begin{figure}[t]
\subfloat[\small Community Structure]{\includegraphics[height=1.7in,width=0.48\columnwidth]{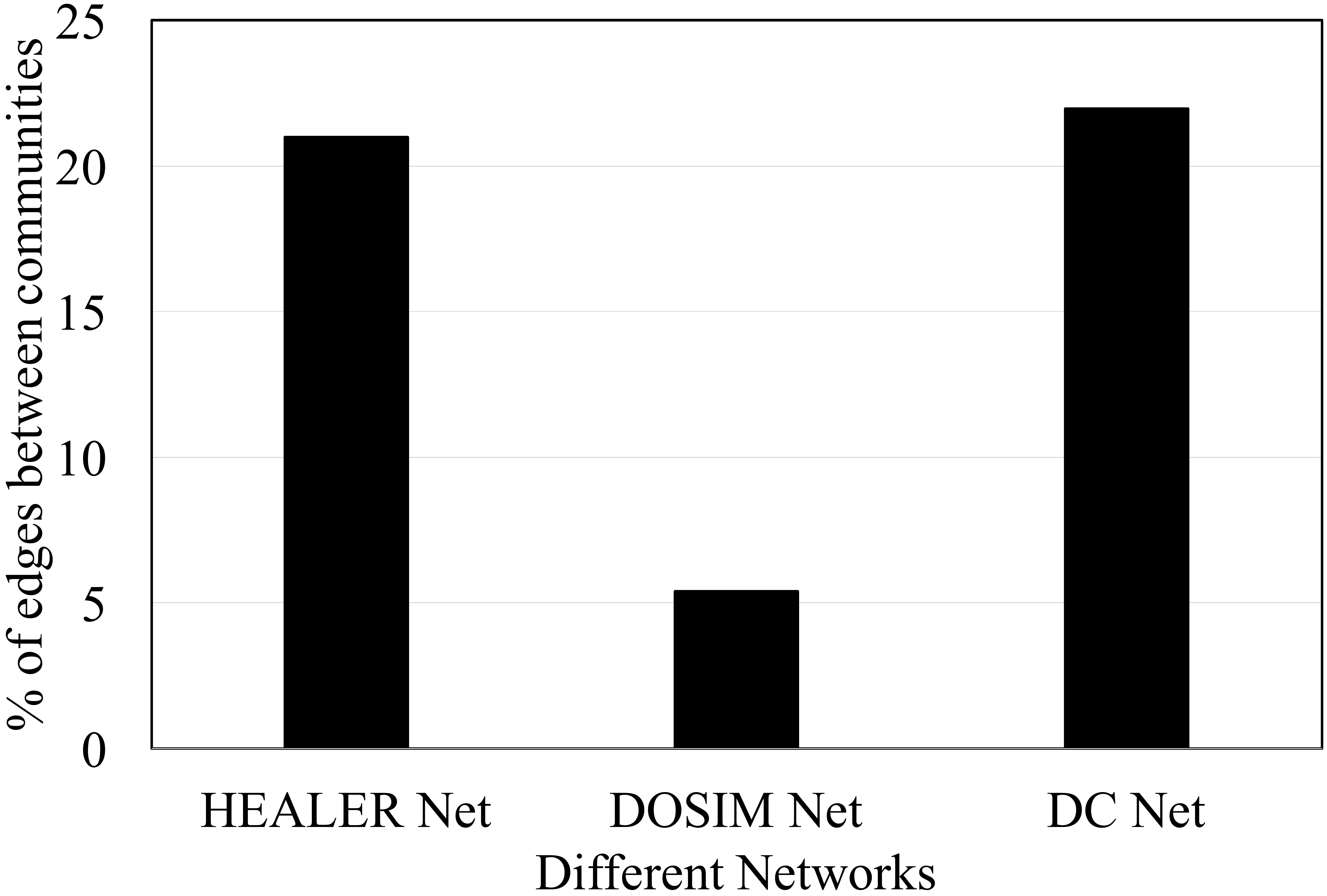}\label{fig:com}}
\hspace{2mm}
\subfloat[\small Coverage of communities by agents]{\includegraphics[height=1.7in,width=0.48\columnwidth]{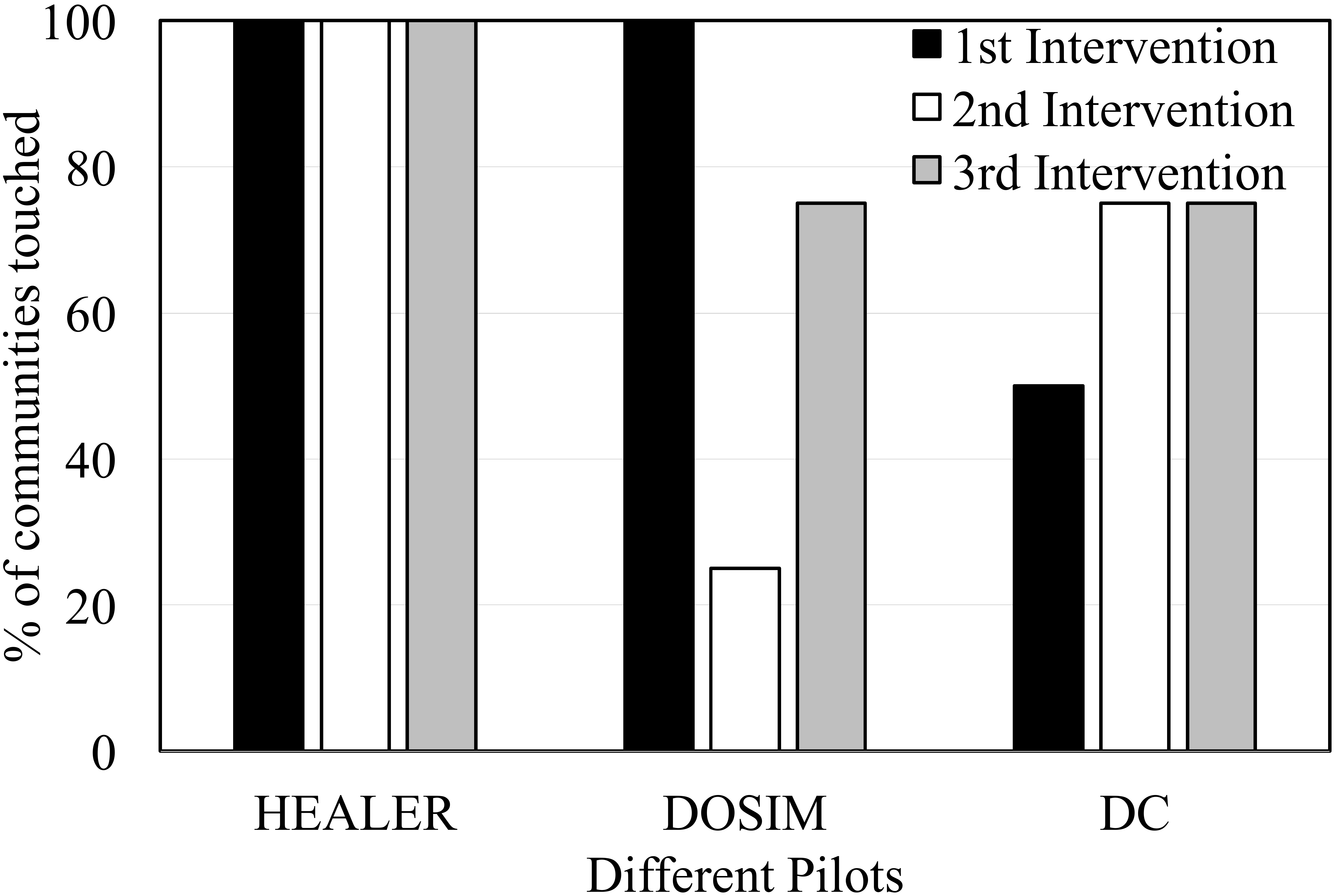}\label{fig:comcov}}
\caption{Exploiting community structure of real-world networks}
\end{figure} 

\begin{figure}[t]
\center{\includegraphics[scale=0.13]
{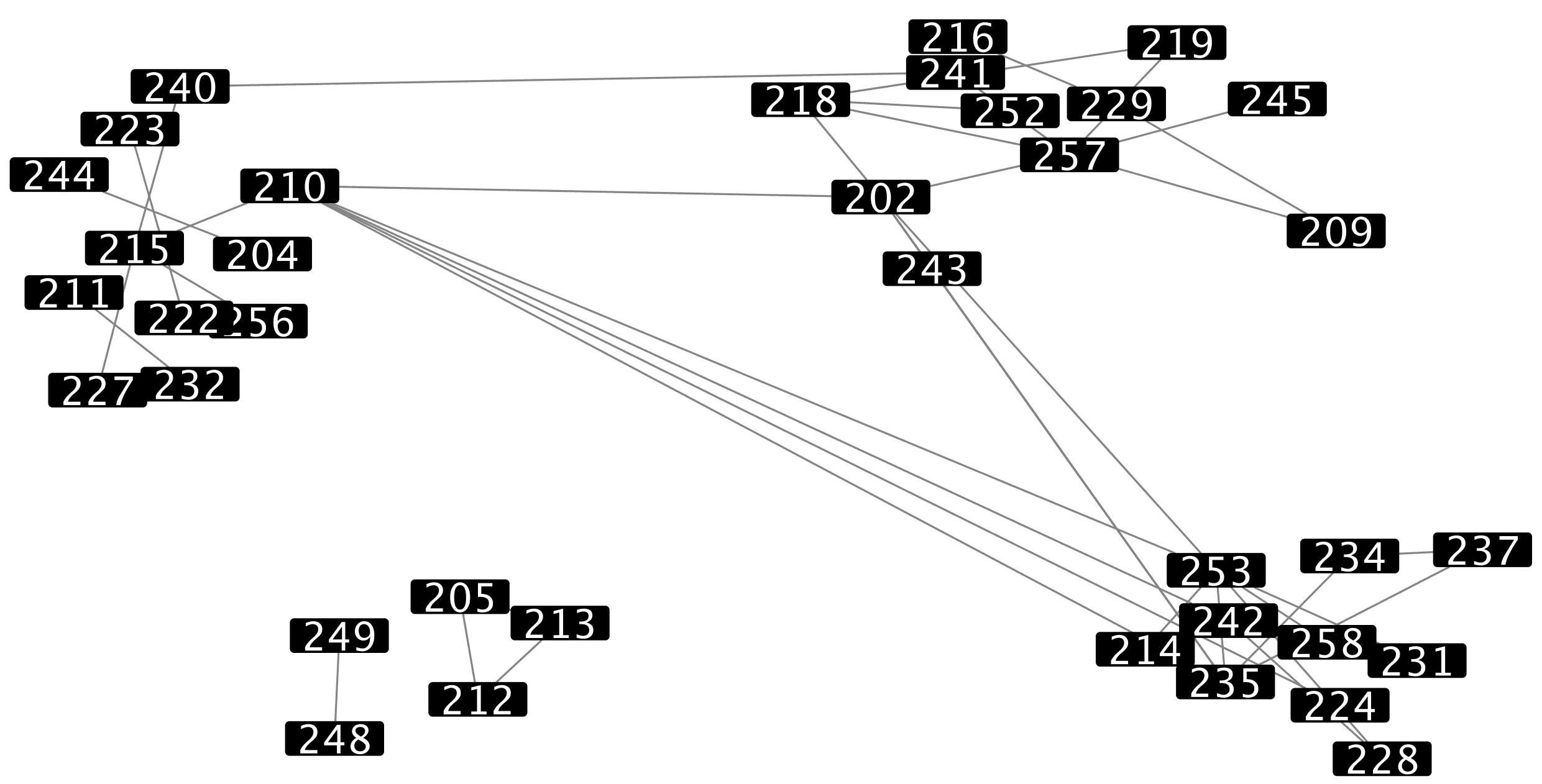}}
\caption{\label{fig:commun} Four Partitions of DC's Pilot Network}
\end{figure}

Figure \ref{fig:com} shows the community structure of the three pilot study social networks. To generate this figure, the three networks were partitioned into communities using METIS \cite{lasalle2013multi}, an off-the-shelf graph partitioning tool. We partitioned each network into four different communities (as shown in Figure \ref{fig:commun}) to match the number of PL (i.e., $K=4$) chosen in every stage. The X-axis shows the three pilot study networks and the Y-axis shows the percentage of edges that go across these four communities. This figure shows that all three networks can be fairly well represented as a set of reasonably disjointed communities, as only 15\% of edges (averaged across all three networks) went across the communities. Next, we show how HEALER and DOSIM exploit this community structure by balancing their efforts across these communities simultaneously to achieve greater information spread as compared to DC.

Figure \ref{fig:comcov} illustrates patterns of PL selection (for each stage of intervention) by HEALER, DOSIM and DC across the four different communities uncovered in Figure \ref{fig:com}. Recall that each pilot study comprised of three stages of intervention (each with four selected PL). The X-axis shows the three different pilots. The Y-axis shows what percentage of communities had a PL chosen from within them. For example, in DC's pilot, the chosen PL covered 50\%  (i.e., two out of four) communities in the $1^{st}$ stage, 75\% (i.e., three out of four) communities in the $2^{nd}$ stage, and so on. This figure shows that HEALER's chosen peer leaders cover all possible communities (i.e., 100\% communities touched) in the social network in all three stages. On the other hand, DC concentrates its efforts on just a few clusters in the network, leaving $\sim$50\% communities untouched (on average). Therefore, while HEALER ensures that its chosen PL covered most real-world communities \textit{in every intervention}, the PL chosen by DC focused on a single (or a few) communities in each intervention. This further explains why HEALER is able to achieve greater information spread, as it spreads its efforts across communities unlike DC. While DOSIM's coverage of communities is similar to DC, it outperforms DC because of $\sim$5X less redundant edges than DC (Figure \ref{fig:edge}).

\begin{figure}[t]
\subfloat[\small Behavior Change]{\includegraphics[height=1.7in,width=0.48\columnwidth]{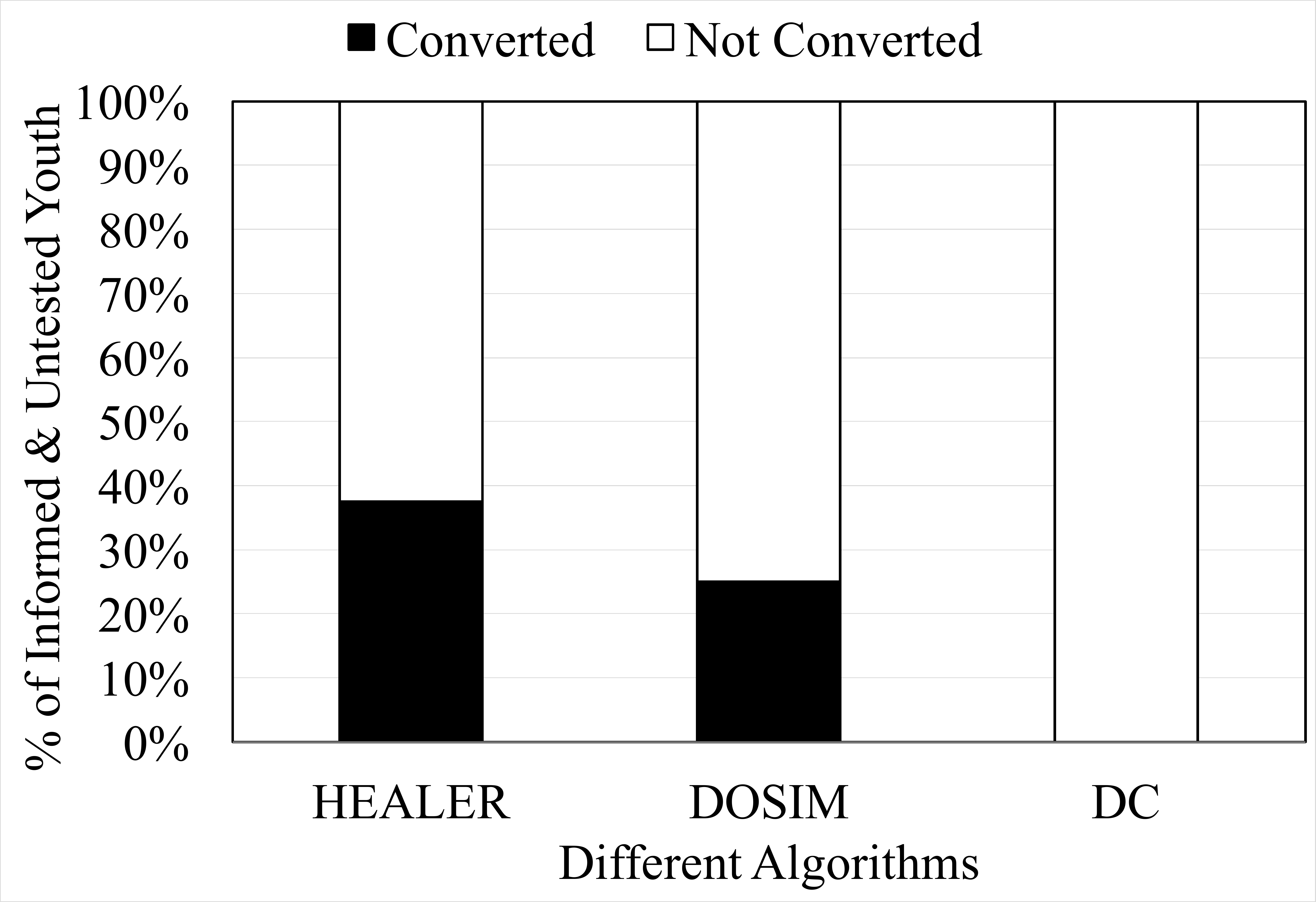}\label{fig:behav}}
\hspace{2mm}
\subfloat[\small Simulation of information spread]{\includegraphics[height=1.7in,width=0.48\columnwidth]{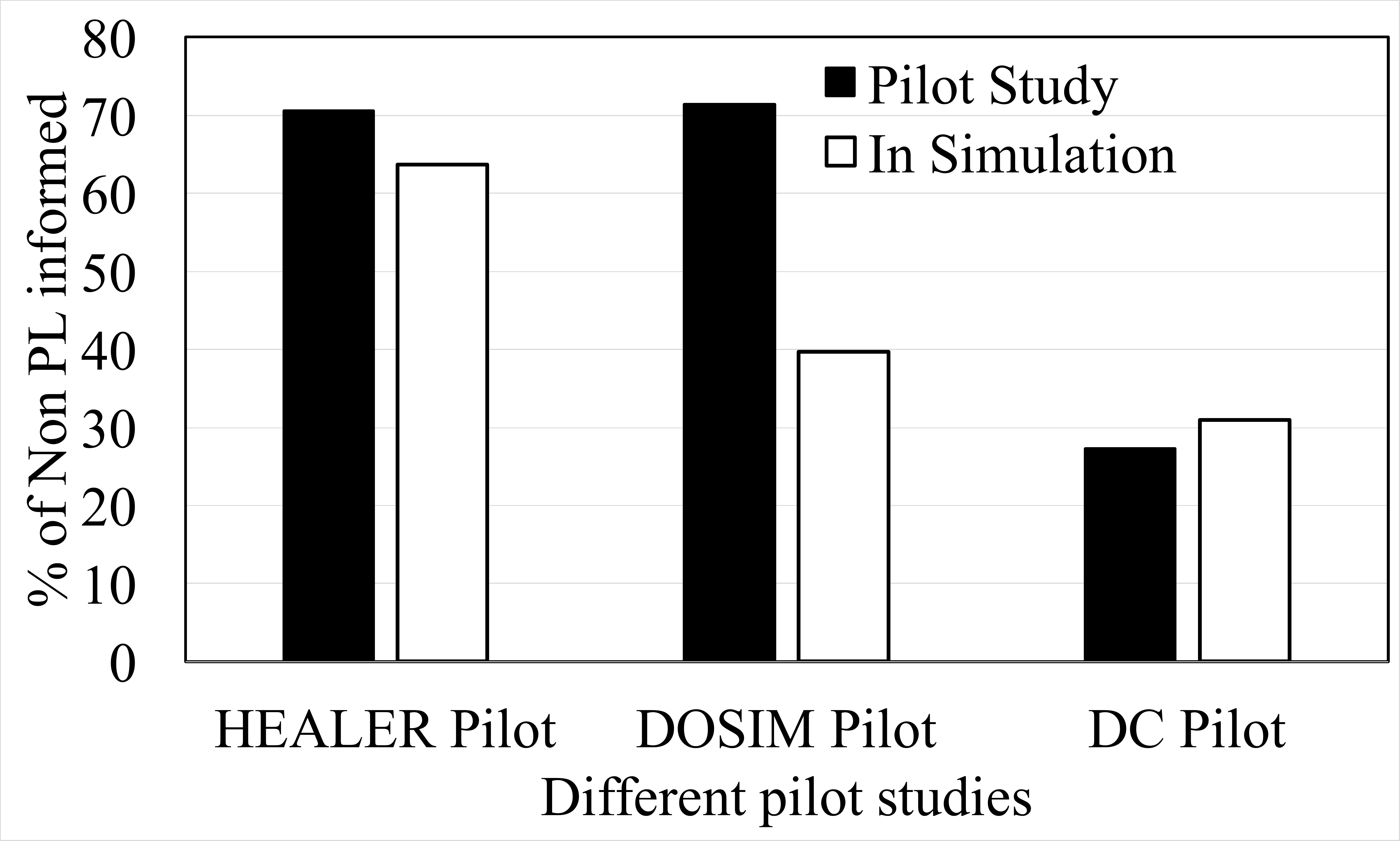}\label{fig:stat1}}
\caption{Behavior Change \& Information Spread in Simulation}
\end{figure}

\textbf{Behavior Change} Figure \ref{fig:behav} compares behavior change observed in homeless youth in the three pilot studies. The X-axis shows different intervention strategies, and the Y-axis shows the percentage of non peer-leaders who were untested for HIV at baseline and were informed about HIV during the pilots (i.e. youth in box B in Figure \ref{fig:venn}). This figure shows that PL chosen by HEALER (and DOSIM) converted 37\% (and 25\%) of the youth in box B to HIV testers. In contrast, \textit{PL chosen by DC did not convert any youth in box B to testers}. DC's information spread reached a far smaller fraction of youth (Figure \ref{fig:inf-spread}), and therefore it is unsurprising that DC did not get adequate opportunity to convert anyone of them to testing. This shows that even though HEALER and DOSIM do not explicitly model behavior change in their objective function, the agents strategies still end up outperforming DC significantly in terms of behavior change.

\section{Challenges Uncovered} \label{sec:6}
This section highlights research and methodological challenges that we uncovered while deploying these agent based interventions in the field. While handling these challenges in a principled manner is a subject for future research, we explain some heuristic solutions used to tackle these challenges in the three pilot studies (which may help in addressing the longer term research challenges).

\textbf{Research Challenges} While conducting interventions, we often encounter an inability to execute actions (i.e., conduct intervention with chosen peer leaders), because a subset of the chosen peer leaders may fail to show up for the intervention (because they may get incarcerated, or find temporary accommodation). Handling this inability to execute actions in a principled manner is a research challenge. Therefore, it is necessary that algorithms and techniques developed for this problem are robust to these errors in execution of intervention strategy. Specifically, we require our algorithms to be able to come up with alternate recommendations for peer leaders, when some homeless youth in their original recommendation are not found. We now explain how HEALER, DOSIM and DC handle this challenge by using \textit{heuristic solutions}. 

Recall that for the first pilot, HEALER's intervention strategies were found by using online planning techniques for POMDPs \cite{yadav2016using}. Instead of offline computation of the entire policy (strategy), online planning only finds the best POMDP action (i.e., selection of $K$ network nodes) for the current belief state (i.e., probability distribution over state of influence of nodes). Upon reaching a new belief state, online planning again plans for this new belief. This interleaving of planning and execution works to our advantage in this domain, as every time we have a failure which was not anticipated in the POMDP model  (i.e., a peer leader which was chosen in the current POMDP action did not show up), we can recompute a policy quickly by marking these unavailable nodes, so that they are ineligible for future peer leader selection. After recomputing the plan, the new peer leader recommendation is again given to the service providers to conduct the intervention.

For the second pilot study, we augmented DOSIM to account for unavailable nodes by using its computed policy to produce a list of alternates for each peer leader. This alternate list ensures that unlike HEALER, DOSIM does not require rerunning in the event of a failure. Thus, if a given peer leader does not show up, then study staff work down the list of alternates to find a replacement. DOSIM computes these alternates by maintaining a parameter $q_v$ (for each node $v$), which gives the probability that node $v$ will show up for the intervention. This $q_v$ parameter enables DOSIM to reason about the inability to execute actions, thereby making DOSIM's policies robust to such failures. To compute the alternate for $v$, we \textit{condition} on the following event $\sigma_v$: node $v$ fails to show up (i.e., set $q_v = 0$), while every other peer leader $u$ shows up with probability $q_u$.  Conditioned on this event $\sigma_v$, we find the node which maximizes the \textit{conditional} marginal gain in influence spread, and use it as the alternate for node $v$. Hence, each alternate is selected in a manner which is robust with respect to possible failures on other peer leader nodes. Finally, in the DC pilot, in case of a failure, the node with the next highest degree is chosen as a peer leader.

\textbf{Methodological Challenges} A methodological challenge was to ensure a fair comparison of the performance of different agents in the field. In the real-world, HEALER, DOSIM and DC could not be tested on the same network, as once we spread HIV messages in one network as part of one pilot study, fewer youth are unaware about HIV (or uninfluenced) for the remaining pilots. Therefore, each agent (HEALER, DOSIM or DC) is tested in a different pilot study with a different social network (with possibly different structure). Since HEALER, DOSIM and DC's performance is not compared on the same network, it is important to ensure that HEALER and DOSIM's superior performance (observed in Figure \ref{fig:inf-spread}) is not due to differences in network structure or any extraneous factors. 

\begin{figure}
\begin{center} 
\begin{tabular}{|l|c|c|c|}
    \hline
     & HEALER & DOSIM & DC\\
    \hline
    Network Diameter & 8 & 8 & 7\\
    \hline
    Network Density & 0.079 & 0.059 & 0.062\\
    \hline
    Avg. Clustering Coefficient & 0.397 & 0.195 & 0.229\\
    \hline
    Avg. Path Length & 3.38 & 3.15 & 3.03\\
    \hline
    Modularity & 0.568 & 0.568 & 0.602\\
    \hline
\end{tabular}
\end{center}
\caption{\label{tab:netsimilar} Similarity of social networks in different pilot studies}
\end{figure}

First, we compare several well-known graph metrics for the three distinct pilot study social networks. Figure \ref{tab:netsimilar} shows that most metrics are similar on all three networks, which establishes that the social networks generated in the three pilot studies were structurally similar. This suggests that comparison results would not have been very different, had all three algorithms been tested on the same network. Next, we attempt to show that HEALER and DOSIM's superior performance (Figure \ref{fig:inf-spread}) was not due to extraneous factors.

\begin{figure}[ht]
\subfloat[\small Comparison on DC's network]{\includegraphics[height=1.7in,width=0.48\columnwidth]{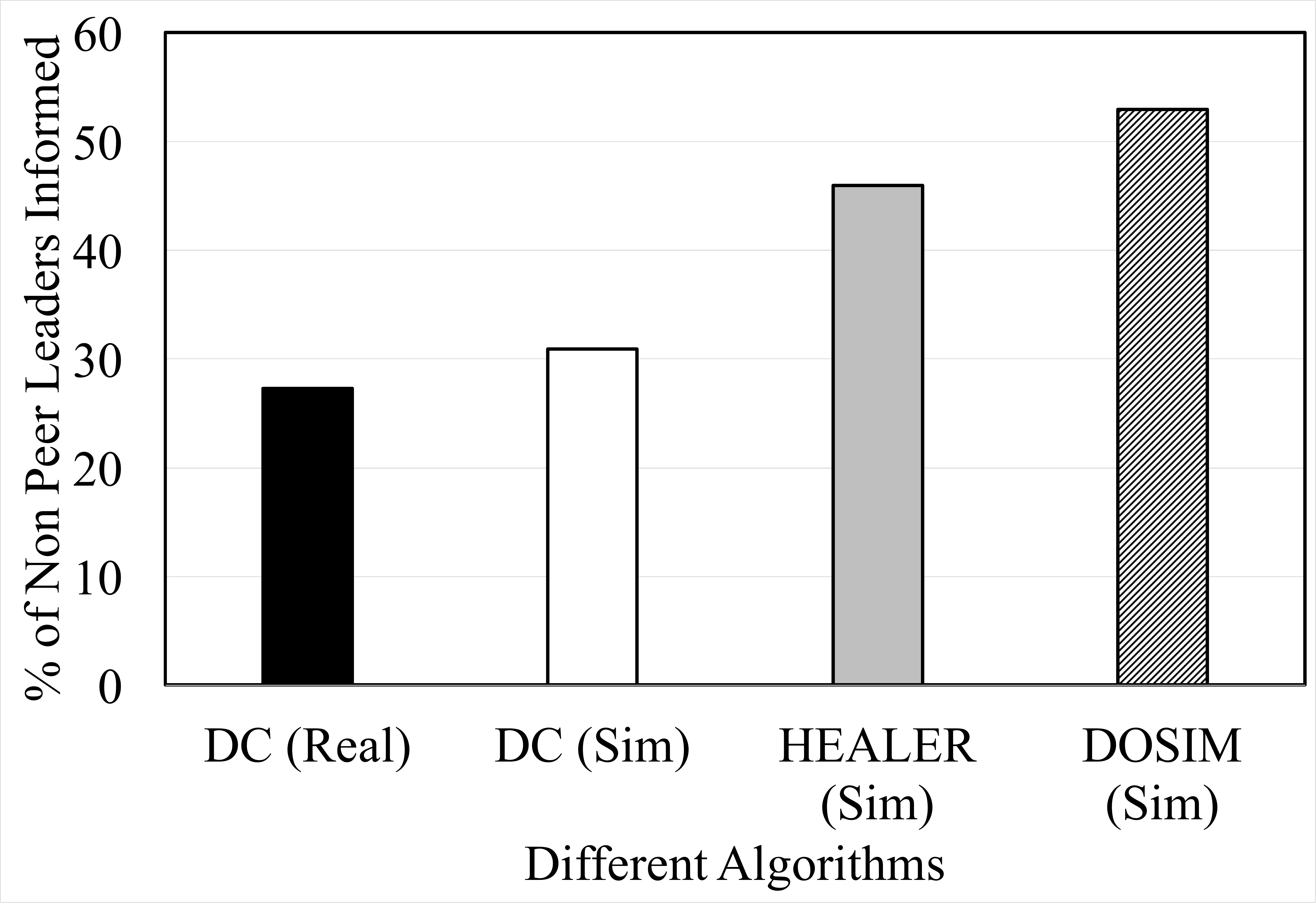}\label{fig:stat2}}
\hspace{2mm}
\subfloat[\small Comparison on perturbed networks]{\includegraphics[height=1.7in,width=0.48\columnwidth]{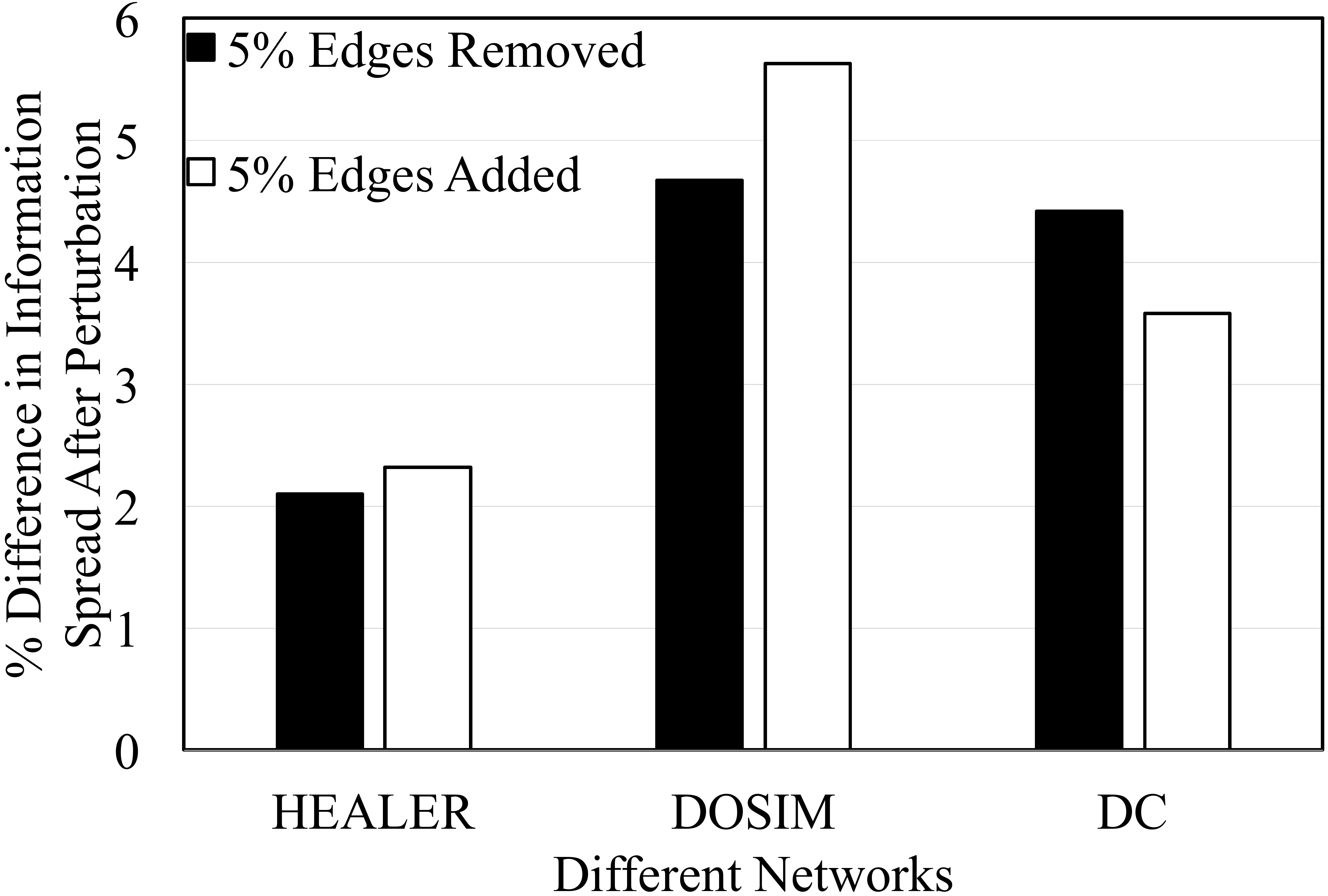}\label{fig:stat3}}
\caption{Investigation of peculiarities in network structure}
\end{figure} 

Figure \ref{fig:stat1} compares information spread achieved by peer leaders in the actual pilot studies with that achieved by the same peer leaders in simulation. The simulation (averaged over 50 runs)  was done with propagation probability set to $p_e=0.6$ in our influence model. The X-axis shows the different pilots and the Y-axis shows the percentage of non peer-leaders informed in the pilot study networks.  First, this figure shows that information spread in simulation closely mirrors pilot study results in HEALER and DC's pilot ($\sim$10\% difference), whereas it differs greatly in DOSIM's pilot. This shows that using $p_e=0.6$ as the propagation probability modeled the real-world process of influence spread in HEALER and DC's pilot study network fairly well, whereas it was not a good model for DOSIM's pilot network. This further suggests that information spread achieved in the real world (atleast in HEALER and DC's pilot) was indeed due to the respective strategies used, and not some extraneous factors. In other words, DC's poor performance may not be attributed to some real-world external factors at play, since its poor performance is mimicked in simulation results (which are insulated from real-world external factors) as well. Similarly, HEALER's superior performance may not be attributed to external factors working in its favor, for the same reason.

On the other hand, since DOSIM's performance in the pilot study does not mirror simulation results in Figure \ref{fig:stat1}, it suggests the role of some external factors, which were not considered in our models. However, the comparison of simulation results in this figure is statistically significant ($p-value = 9.43E-12$), which shows that even if DOSIM's performance in the pilot study matched its simulation results, i.e., even if DOSIM achieved only $\sim$40\% information spread in its pilot study (as opposed to the 70\% spread that it actually achieved), it would still outperform DC by $\sim$33\%.

Having established that DC's poor performance was not due to any external factors, we now show that DC's poor performance in the field was also not tied to some peculiar property/structure of the network used in its pilot study. Figure \ref{fig:stat2} compares information spread achieved by different agents (in simulation over 50 runs), when each agent was run on DC's pilot study network. Again, the simulation was done using $p_e=0.6$ as propagation probability, which was found to be a reasonable model for real-world influence spread in DC's network (see Figure \ref{fig:stat1}). The X-axis in Figure \ref{fig:stat2} shows different algorithms being run on DC's pilot study network (in simulation). The Y-axis shows the percentage of non peer-leaders informed. This figure shows that even on DC's pilot study network, HEALER (and DOSIM) outperform DC in simulation by $\sim$53\% (and 76\%) ($p-value = 9.842E-31$), thereby establishing that HEALER and DOSIM's improvement over DC was not due to specific properties of the networks in their pilot studies, i.e., HEALER and DOSIM's superior performance may not be attributed to specific properties of networks (in their pilot studies) working in their favor. In other words, this shows that DC's poor performance may not be attributed to peculiarities in its network structure working against it, as otherwise, this peculiarity should have affected HEALER and DOSIM's performance as well, when they are run on DC's pilot study network (which does not happen as shown in Figure \ref{fig:stat2}). 

Figure \ref{fig:stat3} shows information spread achieved by peer leaders (chosen in the pilot studies) in simulation (50 runs), averaged across 30 different networks which were generated by perturbation of the three pilot study networks. The X-axis shows the networks which were perturbed. The Y-axis shows the percentage difference in information spread achieved on the perturbed networks, in comparison with the unperturbed network. For example, adding 5\% edges randomly to HEALER's pilot study network results in only $\sim$2\% difference ($p-value = 1.16E-08$) in information spread (averaged across 30 perturbed networks). These results support the view that HEALER, DOSIM and DC's performance are not due to their pilot study networks being on the knife's edge in terms of specific peculiarities. Thus, HEALER and DOSIM outperform DC on a variety of slightly perturbed networks as well.

\section{Conclusion \& Lessons Learned}
This paper illustrates challenges faced in transitioning agents from an emerging phase in the lab, to a deployed application in the field. It presents first-of-its-kind results from three real-world pilot studies, involving 173 homeless youth in Los Angeles. Conducting these pilot studies underlined their importance in this transition process -- \textit{they are crucial milestones in the arduous journey of an agent from an emerging phase in the lab, to a deployed application in the field}. The pilot studies helped in answering several questions that were raised in Section \ref{sec:intro}. First, we learnt that peer-leader based interventions are indeed successful in spreading information about HIV through a homeless youth social network (as seen in Figures \ref{fig:inf-spread}). Moreover, we learnt that peer leaders are very adept at providing lots of information about newer friendships in the social network (Figure \ref{tab:detail}), which helps software agents to refine its future strategies. 

These pilot studies also helped to establish the superiority (and hence, their need) of HEALER and DOSIM -- we are using complex agents (involving POMDPs and robust optimization), and they outperform DC (the modus operandi of conducting peer-led interventions) by 160\% (Figures \ref{fig:inf-spread}, \ref{fig:behav}). The pilot studies also helped us gain a deeper understanding of how HEALER and DOSIM beat DC (shown in Figures \ref{fig:edge}, \ref{fig:comcov}, \ref{fig:com}) --  by minimizing redundant edges and exploiting community structure of real-world networks. Out of HEALER and DOSIM, the pilot tests do not reveal a significant difference in terms of either information spread or behavior change (Figures \ref{fig:inf-spread}, \ref{fig:behav}). Thus, carrying either of them forward would lead to significant improvement over the current state-of-the-art techniques for conducting peer-leader based interventions. However, DOSIM runs significantly faster than HEALER ($\sim40\times$), thus, it is more beneficial in time-constrained settings \cite{wilder2017uncharted}. 

These pilot studies also helped uncover several key challenges (e.g., inability to execute actions, estimating propagation probabilities, etc.), which were tackled in the pilot studies using heuristic solutions. However, handling these challenges in a principled manner is a subject for future research. Thus, while these pilot studies open the door to future deployment of these agents in the field (by providing positive results about the performance of HEALER and DOSIM), they also revealed some challenges which need to be resolved convincingly before these agents can be deployed. 

To the best of our knowledge, this is the first deployment of an influence maximization algorithm in the real world. Further, it is the first time that influence maximization has been applied for social good. The success of this pilot study illustrates one way (among many others) in which AI and influence maximization can be harnessed for benefiting low-resource communities.

\chapter{CAIMS}
\label{chapter:CAIMS}
Both PSINET, HEALER and DOSIM \cite{wilder2017uncharted} rely on the following key assumption: seed nodes can be influenced with certainty. Unfortunately, in most public health domains, this assumption does not hold as ``influencing" seed nodes entails training them to be ``\textit{peer leaders}" \cite{valente2007identifying}. For example, seed nodes promoting HIV awareness among homeless youth need to be trained so that they can communicate information about supposedly private issues in a safe manner \cite{schneider2015new}. This issue of training seed nodes leads to two practical challenges. First, it may be difficult to contact seed nodes in a timely manner (e.g., contacting homeless youth is challenging since they rarely have fixed phone numbers, etc). Second, these seed nodes may decline to be influencers (e.g., they may decline to show up for training sessions). In this chapter, we refer to these two events as contingencies in the influence maximization process. 

Unsurprisingly, these contingencies result in a wastage of valuable time/money spent in unsuccessfully contacting/convincing the seed nodes to attend the training. Moreover, the resulting influence spread achieved is highly sub-optimal, as very few seed nodes actually attend the training session, which defeats the purpose of conducting these interventions. Clearly, contingencies in the influence maximization process need to be considered very carefully.


This chapter discusses a principled approach to handle these inevitable contingencies via the following contributions. First, we introduce the \textbf{C}ontingency \textbf{A}ware \textbf{I}nfluence \textbf{M}aximization (or CAIM) problem to handle cases when seed nodes may be unavailable, and analyze it theoretically. The principled selection of alternate seed nodes in CAIM (when the most preferred seed nodes are not available) sets it apart from any other previous work in influence maximization, which mostly assumes that seed nodes are always available for activation. Second, we cast the CAIM problem as a Partially Observable Markov Decision Process (POMDP) and solve it using CAIMS (\textbf{CAIM} \textbf{S}olver), a novel POMDP planner which provides an adaptive policy which explicitly accounts for contingency occurrences. CAIMS is able to scale up to real-world network sizes by leveraging the community structure (present in most real-world networks) to factorize the action space of our original POMDP into several smaller community-sized action spaces. Further, it utilizes insights from social network literature to represent belief states in our POMDP in a compact, yet accurate manner using Markov networks. Our simulations show that CAIMS outperforms state-of-the-art influence maximization algorithms by $\sim$60\%. Finally, we evaluate CAIMS's usability in the real-world by using it to train a small set of homeless youth (the seed nodes) to spread awareness about HIV among their peers. This domain is an excellent testbed for CAIMS, as the transient nature of homeless youth increases the likelihood of the occurrence of contingencies \cite{rice2013should}.

\section{CAIM Model \& Problem}
In practice, the officials from the homeless youth service providers typically only have 4-5 days to locate/invite the desired youth to be trained as peer leaders. However, the transient nature of homeless youth (i.e., no fixed postal address, phone number, etc) makes contacting the chosen peer leaders difficult for homeless shelters. Further, most youth are distrustful of adults, and thus, they may decline to be trained as peer leaders \cite{milburn2009adolescents}. As a result of these ``\textit{contingencies}", the shelter officials are often forced to conduct their intervention with very few peer leaders in attendance, despite each official spending 4-5 days worth of man hours in trying to find the chosen peer leaders \cite{yadav2017influence}. Moreover, the peer leaders who finally attend the intervention are usually not influential seed nodes. This has been the state of operations even though peer-led interventions have been conducted by social workers for almost a decade now.

%

\begin{figure}[t]
\subfloat[Social Network 1]{\includegraphics[width=0.5\columnwidth]{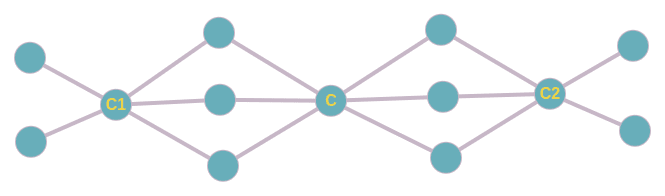}\label{fig:ex1}}
\subfloat[Social Network 2]{\includegraphics[width=0.5\columnwidth]{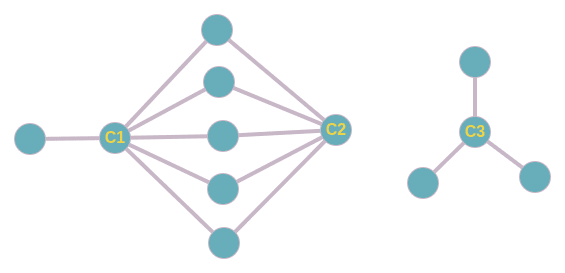}\label{fig:ex2}}
\caption{Examples illustrating harm in overprovisioning}
\end{figure}

To avoid this outcome, ad-hoc measures have been proposed \cite{yadav2017influence}, e.g., contacting many more homeless youth than they can safely manage in an intervention. However, one then runs the risk that lots of youth may agree to be peer leaders, and shelter officials would have to conduct the intervention with all these youth (since it's unethical to invite a youth first and then ask him/her not to come to the intervention), even if the total number of such participants exceeds their maximum capacity \cite{rice2012mobilizing}. This results in interventions where the peer leaders may not be well trained, as insufficient attention is given to any one youth in the training. Note that if contingencies occurred infrequently, then inviting a few extra nodes (over the maximum capacity) may be a reasonable solution. However, as we show in the real-world feasibility trial conducted by us, contingencies are very common (\textit{$\sim$80\%, or 14 out of 18 invitations in the real-world study resulted in contingencies}), and thus, overprovisioning by a small number of nodes is not an option. An ad-hoc fix for this over-attendance, is to first select (say) twice the desired number of homeless youth, invite them one at a time, and stop as soon as the desired number of homeless youth have accepted the invitation. However, we will show that this intuitive ad-hoc overprovisioning based solution performs poorly.


\subsection{Overprovisioning May Backfire} Let $K$ denote the number of nodes (or homeless youth) we want at the intervention. Now, suppose we overprovision by a factor of $2$ and use the algorithm mentioned before. This means that instead of searching for the optimal set of $K$ seed nodes, the algorithm finds the optimal set of $2K$ seed nodes and then influences the \textit{first} $K$ of these nodes that accept the invitation. Naturally, this algorithm should perform better (under contingencies) than the algorithm without overprovisioning. Surprisingly, we show that overprovisioning may make things worse. This happens because of two key ideas: (i) No $K$-sized subset of the optimal set of $2K$ nodes may be as good as the optimal set of $K$ nodes (this indicates that we may not be looking for the right nodes when we search for the optimal set of $2K$ nodes), and (ii) An arbitrary $K$-sized subset of the optimal set of $2K$ nodes (obtained because we stick to the \textit{first} $K$ nodes that accept the invitation) may perform arbitrarily bad.


We now provide two examples that concretize these facts. For simplicity of the examples, we assume that influence spreads only for one round, number of nodes required for the intervention is $K = 1$ and the propagation probability $p(e)$ is 0.5 for every edge. We use $\I(S)$ to denote the expected influence in the network when nodes of set $S$ are influenced. Firstly, consider the example social network graph in Figure~\ref{fig:ex1}. Suppose $C$ and $C1$ are nodes that are regularly available, and are likely to accept the invitation. Now, let's find the best single node to influence for maximum influence spread. We don't need to consider nodes other than $\{C1, C, C2\}$ since they're obviously suboptimal. For the remaining nodes, we have $\I(C1) = 5*0.5 = 2.5$, $\I(C) = 6*0.5 = 3$ and $\I(C2)=5*0.5=2.5$, and so the best single node to influence is $C$. Now, suppose we overprovision by a factor of $2$, and try to find the optimal set of $2$ nodes for maximum influence spread. The influence values are $\I(\{C1, C\}) = \I(\{C2, C\}) = 5*0.5 + 3*0.75 = 4.75$ and $\I(\{C1, C2\})=10*0.5=5$. So, the optimal set of $2$ nodes to influence is $\{C1, C2\}$. But, since we need only one node, we would eventually be influencing either $C1$ or $C2$, giving us an expected influence of $2.5$. On the other hand, if we did not overprovision, we would go for node $C$ (the best single node to influence) and have an expected influence of $3$. This example demonstrates that no $K$-sized subset of the optimal set of $2K$ nodes may be as good as the optimal set of $K$ nodes. Note that, for clarity, the example considered here was small and made simple, and hence the difference between $3$ and $2.5$ may seem small. But, the example can be extended such that the difference is arbitrarily larger.


On a different note, suppose in this second example, node $C1$ is unavailable (because say it declines the invitation). In this case, the overprovisioning algorithm would have to go for $C3$ (the only other node in the optimal set of $2$ nodes), leading to an expected influence of $1.5$. However, an adaptive solution, would look for node $C1$ and after finding that its unavailable, would go for the next best node which is node $C2$. This gives an adaptive solution an expected influence of $2.5$.


\begin{figure}[t]
\subfloat[SBM Networks]{\includegraphics[height=1.7in,width=0.47\columnwidth]{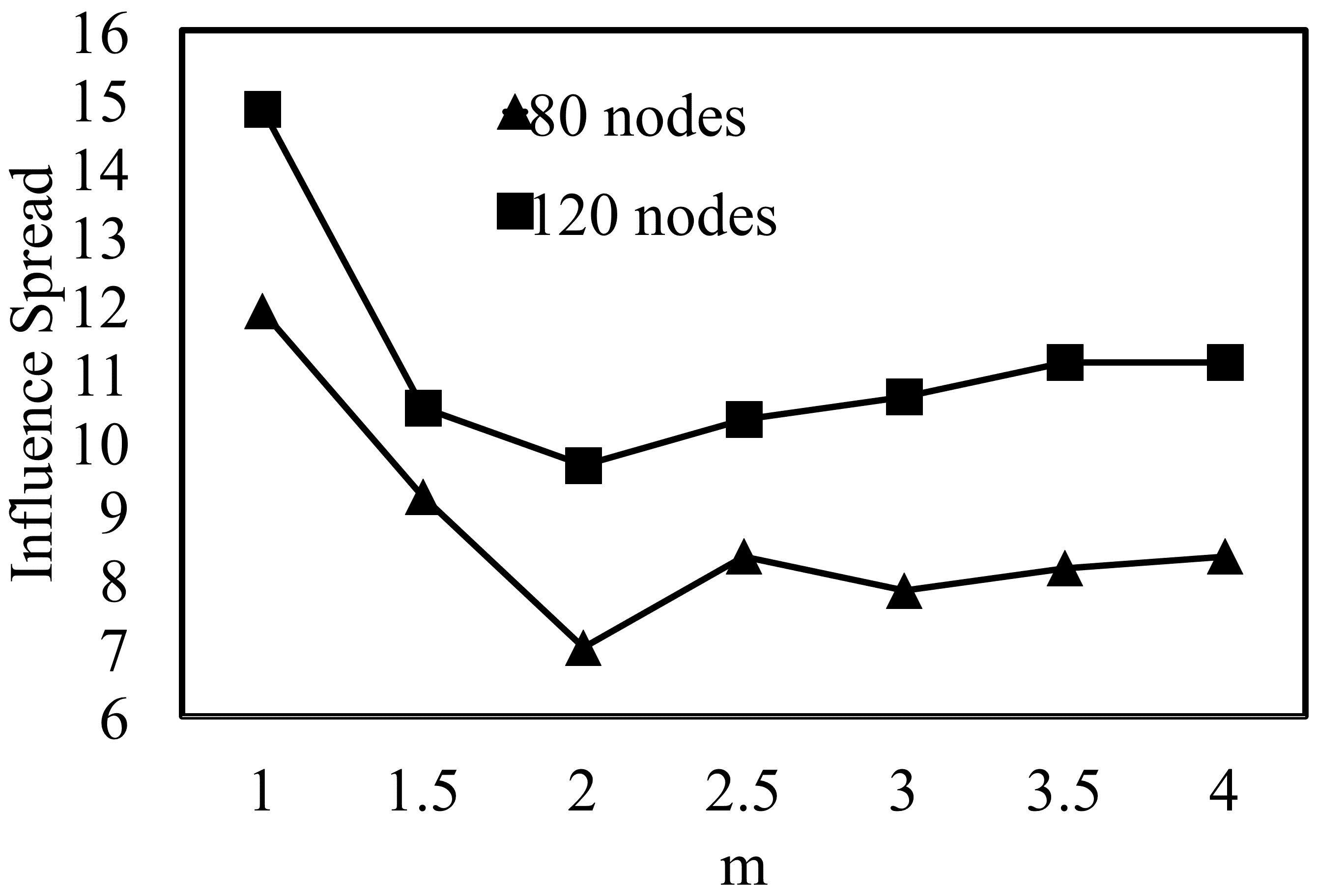}\label{fig:9}}
\hspace{2mm}
\subfloat[PA Networks]{\includegraphics[height=1.7in,width=0.47\columnwidth]{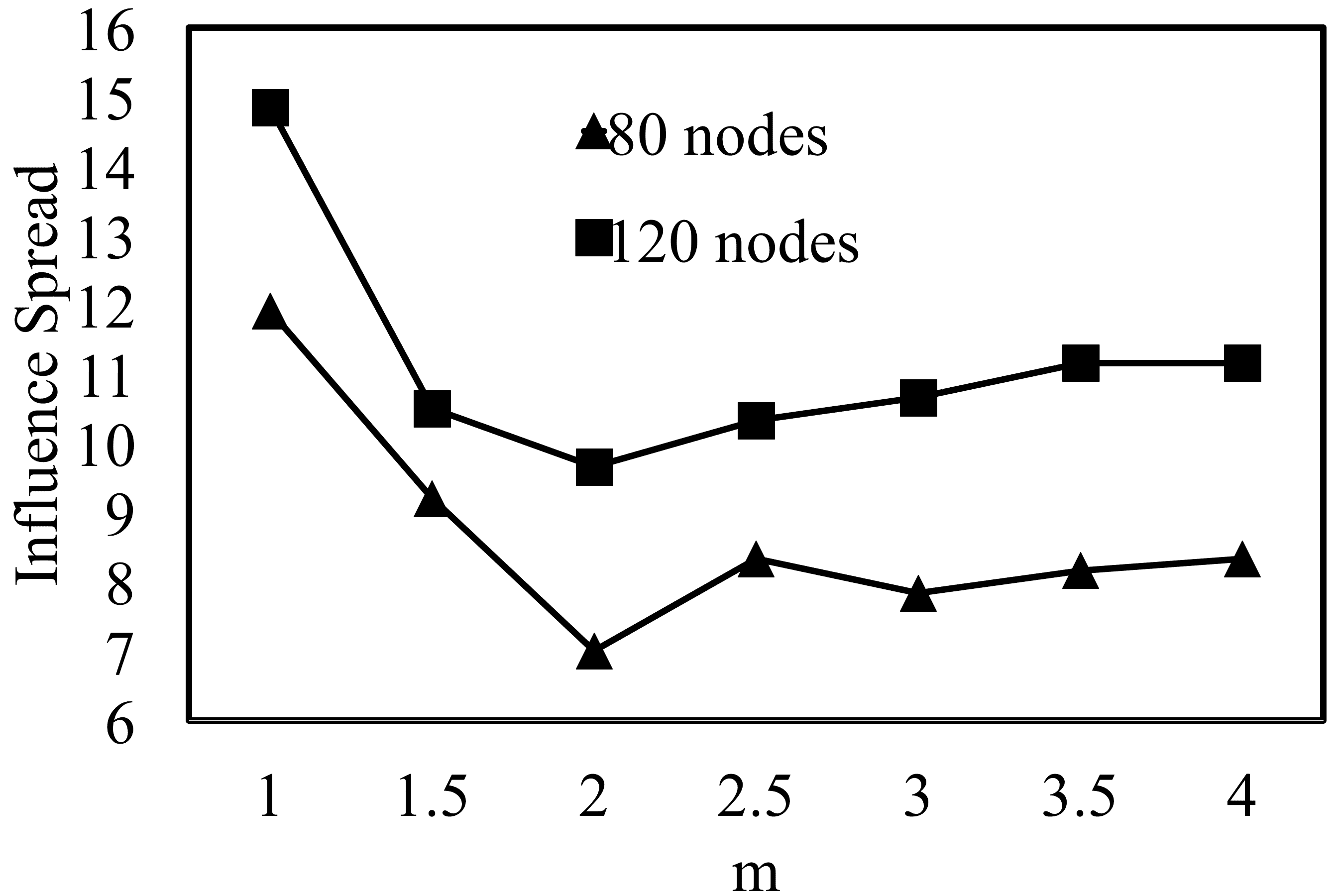}\label{fig:10}}
\caption{The Harm in Overprovisioning}
\end{figure}

Having provided examples which provide intuition as to why simple ad-hoc overprovisioning based algorithms may backfire, we now provide empirical support for this intuition by measuring the performance of the Greedy algorithm \cite{kempe2003maximizing} (the gold standard in influence maximization) under varying levels of overprovisioning. Figures \ref{fig:9} and \ref{fig:10} compare influence spread achieved by Greedy on stochastic block model (SBM) and preferential attachment (PA) networks \cite{seshadhri2012community}, respectively, as it finds the optimal set of $m*K$ nodes ($K=2$) to invite (i.e., overprovision by factor $m$) and influence the first $K$ nodes that accept the invitation (the order in which nodes are invited is picked uniformly at random). The x-axis shows increasing $m$ values and the y-axis shows influence spread. This figure shows that in both SBM and PA networks of different sizes, overprovisioning hurts, i.e., optimizing for larger seed sets in anticipation of contingencies actually hurts influence spread, which confirms our intuition outlined above. Overprovisioning's poor performance reveals that simple solutions do not work, thereby necessitating careful modeling of contingencies, as we do in CAIM.

\subsection{Problem Setup} Given a \textit{friendship based social network}, the goal in CAIM is to invite several network nodes for the intervention until we get $K$ nodes who agree to attend the intervention. The problem proceeds in $T$ sequential sessions, where $T$ represents the number of days that are spent in trying to invite network nodes for the intervention. In each session, we assume that nodes are either available or unavailable for invitation. This is because on any given day (session), homeless youth may either be present at the shelter (i.e., available) or not (i.e., unavailable). We assume that only nodes which are available in a given session can accept invitations in that session. This is because homeless youth frequently visit shelters, hence we utilize this opportunity to issue invitations to them if we see them at the shelter.


Let $\phi^t \in \{0,1\}^N$ (called a realization) be a binary vector which denotes the availability or unavailability (for invitation) of each network node in session $t \in [1,T]$. We take a Bayesian approach and assume that there is a known prior probability distribution $\bm{\Phi}$ over realizations $\phi^t$ such that $p(\phi^t) := \mathcal{P}[\bm{\Phi} = \phi^t]$. In our domain, this prior distribution is represented using a Markov Network. We assume that the realization $\phi^t$ for each session $t \in [1,T]$ is drawn i.i.d. from the prior distribution $\bm{\Phi}$, i.e., the presence/absence of homeless youth at the shelter in every session $t \in [1,T]$ is assumed to be an i.i.d. sample from $\bm{\Phi}$. We further assume that while the prior distribution $\bm{\Phi}$ is provided to the CAIM problem as input, the complete i.i.d. draws from this distribution (i.e., the realizations $\phi^t\mbox{ }\forall t \in [1,T]$) are not observable. This is because while querying the availability of a small number of nodes ($\sim$3-4) is feasible, querying each node in the social network (which can have 150-160 nodes) for each session/day (to completely observe $\phi^t$) requires a lot of work which is not possible with the shelters limited resources \cite{rice2010positive}. 

In each session $t \in [1,T]$, a maximum of $L$ actions can be taken, each of which can be of three possible types: \textit{queries}, \textit{invites} and \textit{end-session} actions. Query action $q_{\bm{a}}$ in session $t \in [1,T]$ ascertains the availability/unavailability of a subset of nodes $\bm{a}$ ($\|\bm{a}\| \leqslant Q_{max}$, the maximum query size) in session $t$ with certainty. Thus, query actions in session $t$ provide partial \textit{observations} about the realization of nodes $\phi^t$ in session $t$. On the other hand, invite action $m_{\bm{a}}$ invites a subset of nodes $\bm{a} \subset V$ ($\|\bm{a}\| \leqslant K$) to the intervention. Upon taking an invite action, we \textit{observe} which invited nodes are present (according to $\phi^t$) in the session and which of them accepted our invitation. Each invited node that is present accepts the invitation with a probability $\epsilon$. We refer to the nodes that accept our invitation  as ``\textit{locked} nodes" (since they are guaranteed to attend the intervention). Finally, we can also take an \textit{end-session} action, if we choose not to invite/query any more nodes in that session. 

The observations received from query and invite actions (\textit{end-session} action provides no observation) taken in a session allows us to update the original prior distribution $\bm{\Phi}$ to generate a posterior distribution $\bm{\Phi}^{pos}_t(i)\mbox{ }\forall i \in [0,L]$ for session $t$ (where $i$ actions have been taken in session $t$ so far). These posteriors can then be used to decide future actions that need to be taken in a session. Note that for every session $t$, $\bm{\Phi}^{pos}_t(0) = \bm{\Phi}$, i.e., at the beginning of each session, we start from the original prior distribution $\Phi$ and then get new posteriors every time we take an action in the session.


Note that even though query actions provide strictly lesser information than invite actions (for the same subset of nodes), their importance in CAIM is highlighted as follows: recall that the optimal set of $2$ nodes in Figure~\ref{fig:ex2} is $\{C1,C3\}$. If we remove the ability to query, we would invite nodes $C1$ and $C3$. In case $C1$ is not present and $C3$ accepts our invitation, we would be stuck with conducting intervention with only node $C3$ (since invited nodes who accept the invitation cannot be un-invited). Thus, we realize that inviting $C3$ is desirable only if $C1$ is present and accepts our invitation. Query actions allow us to query the presence or absence of both nodes $C1$ and $C3$ (so that we don't waste an invite action in case node $C1$ is found to be not present according to the query action's observation).

Informally then, given a friendship based social network $G=(\bm{V},\bm{E})$, the integers $T$, $K$, $L$, $Q_{max}$ and $\epsilon$, and prior distribution $\bm{\Phi}$, the goal of CAIM is to find a policy for choosing $L$ sequential actions for $T$ sessions s.t. the expected influence spread (according to our influence model) achieved by the set of locked nodes (i.e., nodes which finally attend the intervention) is maximized.

Let $\bm{\mathcal{Q}} = \{ q_{\bm{a}}\mbox{ s.t. } 1 \leqslant \|\bm{a}\| \leqslant Q_{max}\}$ denote the set of all possible query actions that can be taken in any given session $t \in [1,T]$. Similarly, let $\bm{\mathcal{M}} = \{ m_{\bm{a}}\mbox{ s.t. } 1 \leqslant \|\bm{a}\| \leqslant K\}$ denote the set of all possible invite actions that can be taken in any given session $t \in [1,T]$. Also, let $\bm{\mathcal{E}}$ denote the end-session action. Let $\mathcal{A}^t_i \in  \bm{\mathcal{Q}} \cup \bm{\mathcal{M}} \cup \bm{\mathcal{E}}$ denote the $i^{th}$ action ($i \in [1,L]$) chosen by CAIM's policy in session $t \in [1,T]$.

Upon taking action $\mathcal{A}^t_i$ ($i \in [1,L], t \in [1,T]$), we receive observations which allow us to generate posterior distribution $\bm{\Phi}^{pos}_t(i)$. Denote by $\bm{M}^t_i$ the set of all locked nodes after the $i^{th}$ action is executed in session $t$. Denote by $\bm{\Delta}$ the set of all possible posterior distributions that we can obtain during the CAIM problem. Denote by $\bm{\Gamma}$ all possible sets of locked nodes that we can obtain during the CAIM problem. Finally, we define CAIM's policy $\bm{\Pi} : \bm{\Delta} \times \bm{\Gamma} \times [0, L] \times [1,T] \rightarrow \bm{\mathcal{Q}} \cup \bm{\mathcal{M}} \cup \bm{\mathcal{E}}$ as a function that takes in a posterior distribution, a set of locked nodes, the number of actions taken so far in the current session, and the session-id as input, and outputs an action $\mathcal{A}^t_i$ for the current timestep. 


\begin{problem}{\textbf{CAIM Problem}}
Given as input a social network $G=(\bm{V}, \bm{E})$ and integers $T$, $K$, $L$, $Q_{max}$ and $\epsilon$, and a prior distribution $\bm{\Phi}$ (as defined above), denote by $\mathcal{R}(\bm{M}^T_{L})$ the \textit{expected total influence spread (i.e., number of nodes influenced)} achieved by nodes in $\bm{M}^T_{L}$ (i.e., locked nodes at the end of $T$ sessions). Let $\mathbb{E}_{\bm{M}^T_L \sim \bm{\Pi}}[\mathcal{R}(\bm{M}^T_{L})]$ denote the expectation over the random variable $\bm{M}^T_L$, where $\bm{M}^T_L$ is updated according to actions recommended by policy $\bm{\Pi}(\bm{\Phi}^{pos}_T(L-1), \bm{M}^T_{L-1}, L-1, T)$. More generally, in session $t \in [1,T]$, $\bm{M}^t_i \mbox{ }\forall i \in [0,L]$  is updated according to actions recommended by policy $\bm{\Pi}(\bm{\Phi}^{pos}_t(i-1), \bm{M}^t_{i-1}, i-1, t)$. Then, the objective of CAIM is to find an optimal  policy $\bm{\Pi^*} = argmax_{\Pi} \mathbb{E}_{\bm{M}^T_L \sim \bm{\Pi}}[\mathcal{R}(\bm{M}^T_{L})]$. 
\end{problem} 


We now theoretically analyze the CAIM problem. 

\begin{lemma}\label{CaimTh:1}
The CAIM problem is NP-Hard.
\end{lemma}
\begin{proof}
Consider an instance of the CAIM problem with prior probability distribution $\bm{\Phi}$ that is the realization $\phi_*$ with probability $1$, where $\phi_*$ is a vector of all $1$s. Such a problem reduces to the standard influence maximization problem, wherein we need to find the optimal subset of $K$ nodes to influence to have maximum influence spread in the network. But, the standard influence maximization problem is an NP-Hard problem, making CAIM NP-Hard too.
\end{proof}

Some NP-Hard problems exhibit nice properties that enable approximation guarantees for them. \cite{golovin2011adaptive} introduced adaptive submodularity, the presence of which would ensure that a simple greedy algorithm provides a $(1-1/e)$ approximation w.r.t. the optimal CAIM policy. However, we show that while CAIM can be cast into the adaptive stochastic optimization framework of \cite{golovin2011adaptive}, our objective function is not adaptive submodular, because of which their Greedy algorithm does not have a $(1-1/e)$ approximation guarantee. 

\begin{lemma}\label{CaimTh:2}
The objective function of CAIM is not adaptive submodular.
\end{lemma}
\begin{proof}
The key idea is that taking a particular action (say $a_o$) now, may have a low marginal gain because of the realization of the current session, but after a few actions, taking the same action $a_o$ might have a high marginal gain because of a change of session.

More formally, consider the following example. At the beginning of the first session, we take a query action and ask about nodes $\{1,2,3\}$. We get the observation that each of them is absent. At this point, if we take the invite action $a_o = \tuple{\{2\}, i}$, we get a marginal gain of $0$. On the other hand, suppose we took the end-session action after the query, advance to the next session, again take a query action and ask about nodes $\{1,2,3\}$ and this time get the observation that $2$ is present (while others are absent). Now if we take the same invite action $a_o$, we get a positive marginal gain. This shows that the objective function of CAIM is not adaptive submodular.
\end{proof}

These theorems show that CAIM is a computationally hard problem and it is difficult to even obtain any good approximate solutions for it. In this paper, we model CAIM as a POMDP.

\section{POMDP Model}
We cast the CAIM problem using POMDPs \cite{puterman2009markov}, as the uncertainty about the realization of nodes $\phi^t$ is similar to partial state observability in POMDPs. Finally, actions (queries and invites) that are chosen for the current session depend on the actions that are taken in future sessions (for e.g., influencing node $A$ might be really important, but he/she may not be available in session $t$, therefore invite actions in session $t$ can focus on other nodes, and influencing node $A$ can be left to future sessions). This suggests the need to do lookahead search, which is the main motivation behind solving a POMDP. We now explain how we map CAIM onto a POMDP.

\sloppy \textbf{States} A POMDP state consists of four entities $s=\tuple{\phi, \bm{M}, numAct, sessID}$. Here, $sessID \in [1,T]$ identifies the session we are in. Also, $numAct \in [0 , L]$ determines the number of actions that have been taken so far in session $sessID$. $\bm{M}$ denotes the set of locked nodes so far (starting from the first session). Finally, $\phi$ is the node realization $\phi^{sessID}$ in session $sessID$. In our POMDP model, states with $sessID=T$ and $numAct=L$ are terminal states, since they represent the end of all sessions.

\textbf{Actions} A POMDP action is a tuple $a = \tuple{\bm{S}, type}$. Here, $type$ is a symbolic character which determines whether $a$ is a query action (i.e., $type=q$), an invite action (i.e., $type=i$) or an \textit{end-session} action (i.e., $type=e$). Also, $\bm{S} \subset \bm{V}$ denotes the subset of nodes that is queried ($type=q$) or invited ($type=i$). If $type=q$, the size of subset $\|\bm{S}\| \in [1,Q_{max}]$. Similarly, if $type=i$, $\|\bm{S}\| \in [1, K]$ . Finally, if $type=e$, subset $\bm{S}$ is empty. 

\textbf{Observations} Upon taking a query action $a = \tuple{\bm{S}, q}$ in state $s=\tuple{\phi, \bm{M}, numAct, sessID}$, we receive an observation that is completely determined by state $s$. In particular, we receive the observation $o_q=\{\phi(v)\mbox{ } \forall v \in \bm{S}\}$, i.e., the availability status of each node in $\bm{S}$. And, by taking an invite action $a = \tuple{\bm{S}, i}$ in state $s=\tuple{\phi, \bm{M}, numAct, sessID}$, we receive two kinds of observations. Let $\bm{\Gamma} = \{ v \in \bm{S} \mbox{ s.t. } \phi(v)=1\}$ denote the set of available nodes in \textit{invited set} $\bm{S}$. First, we get observation $o^1_i=\{\phi(v)\mbox{ } \forall v \in \bm{S}\}$ which specifies the availability status of each node in invited set $\bm{S}$. We also get an observation $o^2_i=\{b(v) \mbox{ } \forall v \in \bm{\Gamma}\}$ for each available node $v \in \bm{\Gamma}$, which denotes whether node $v$ accepted our invitation and joined the locked set of nodes ($b(v)=1$) or not ($b(v)=0$). Finally, the \textit{end-session} action does not generate any observations. 


\textbf{Rewards} We only get rewards when we reach terminal states $s'=\tuple{\phi, \bm{M}, numAct, sessID}$ with $sessID=T$, $numAct=L$. The reward attained in terminal state $s'$ is the expected influence spread (as per our influence model) achieved by influencing nodes in the locked set $\bm{M}$ of $s'$.

\textbf{Transition And Observation Probabilities} Due to our exponential sized state and action spaces, maintaining transition and observation probability matrices is not feasible. Hence, we follow the paradigm of large-scale online POMDP solvers \cite{silver2010monte} by using a generative model  $\Lambda(s, a) \sim (s', o, r)$ of the transition and observation probabilities. This generative model allows generating on-the-fly samples from the exact distributions $T(s'|s,a)$ and $\Omega(o|a,s')$ at very low computational costs. In our generative model, the state undergoes transitions as follows. On taking a query action, we reach a state $s'$ which is the same as $s$ except that $s'.numAct = s.numAct + 1$. On taking an invite action $\tuple{\bm{S}, i}$, we reach $s'$ which is the same as $s$ except that $s'.numAct = s.numAct + 1$, and $s'.\bm{M}$ is $s.\bm{M}$ appended with nodes of $\bm{S}$ that accept the invitation. Note that binary vector $\phi$ stays unchanged in either case (since the session does not change). Finally, on taking the end-session action, we start a new session by transitioning to state $s'$ s.t., $s'.numAct = 0$, $s'.sessID = s.sessID + 1$, $s'.\bm{M} = s.\bm{M}$ and $s'.\phi$ is resampled i.i.d. from the prior distribution $\bm{\Phi}$. Note that the components $\bm{M}$, $numAct$ and $sessID$ of a state are fully observable. The observations (obtained on taking any action) are deterministically obtained as given in the ``Observations" sub-section given above.

\textbf{Initial Belief State} The prior distribution $\bm{\Phi}$, along with other completely observable state components (such as $sessID =1$, $numAct = 0$, and an empty locked set $\bm{M} = \{\}$) forms our initial belief state.


\section{CAIMS: CAIM Solver}
Our POMDP algorithm is motivated by the design of FV-POMCP, a recent online POMDP algorithm \cite{amato2015scalable}. Unfortunately, FV-POMCP has several limitations which make it unsuitable for solving the CAIM problem. Thus, we propose CAIMS, a Monte-Carlo (MC) sampling based online POMDP algorithm which makes key modifications to FV-POMCP, and solves the CAIM problem for real-world sized networks. Next, we provide a brief overview of FV-POMCP.

\subsection{Background on FV-POMCP} FV-POMCP extends POMCP to deal with large action spaces. It assumes that the action space of the POMDP can be factorized into a set of $\ell$ factors, i.e., each action $a$ can be decomposed into a set of sub-actions $a_l \forall l \in [1,\ell]$. Under this assumption, the value function of the original POMDP is decomposable into a set of overlapping factors. i.e., $Q(b,a) = \sum\limits_{l \in [1,\ell]} \alpha_l Q_l(b,a_l)$, where $\alpha_l$ ($\forall l \in [1,\ell]$) are factor-specific weights. FV-POMCP maintains a single UCT tree (similar to standard POMCP), but it differs in the statistics that are maintained at each node of the UCT tree. Instead of maintaining $\hat{Q}(b_h,a)$ and $n_{ha}$ statistics for every action in the global (unfactored) action space at tree node $h$, it maintains a set of statistics that estimates the values $\hat{Q}_l(b_h,a_l)$ and $n_{ha_l}\mbox{ } \forall l \in [1,\ell]$. 

Joint actions are selected by the UCB1 rule across all factored statistics, i.e., $a = argmax_a \sum\limits_{l \in [1,\ell]} \hat{Q}_l(b_h,a_l) + c \sqrt{log(N_h+1)/n_{ha_l}}$. This maximization is efficiently done using variable elimination (VE) \cite{guestrin2002multiagent}, which exploits the action factorization appropriately. Thus, FV-POMCP achieves scale-up by maintaining fewer statistics at each tree node $h$, and by using VE to find the maximizing joint action. 

However, there are two limitations which makes FV-POMCP unsuitable for solving CAIM. First, the VE procedure used in FV-POMCP (as described above) may return an action (i.e., a set of nodes) which is infeasible in the CAIM problem (e.g., the action may have more than $K$ nodes). We elaborate on this point later. Second, FV-POMCP uses unweighted particle filters to represent belief states, which becomes highly inaccurate with exponentially sized state spaces in CAIM. We address these limitations in CAIMS.

\subsection{CAIMS Solver} 
CAIMS is an online Monte-Carlo sampling based POMDP solver that uses UCT based Monte-Carlo tree search to solve the CAIM problem. Similar to FV-POMCP, CAIMS also exploits action factorization to scale up to large action spaces. We now explain CAIMS's action factorization. 

\textbf{Action Factorization} Real world social networks generally exhibit a lot of community structure, i.e., these networks are composed of several tightly-knit communities (partitions), with very few edges going across these communities \cite{seshadhri2012community}. This community structure dictates the action factorization in CAIMS. As stated before, the POMDP model has each action of the form $\tuple{\bm{S}, type}$, where $\bm{S}$ is a subset of nodes (that are being queried or invited). This (sub)set $\bm{S}$ can be represented as a boolean vector $\vec{S}$ (denoting which nodes are included in the set). Let $Q_q(\vec{S})$ denote the Q-value of the query action $\tuple{\bm{S},q}$, $Q_i(\vec{S})$ denote the Q-value of the invite action $\tuple{\bm{S},i}$ and let $Q_e$ denote the Q-value of the end-session action $\tuple{\{\},e}$. Now, suppose the real-world social network is partitioned into $\ell$ partitions (communities) $P_1, P_2, \cdots P_{\ell}$. Let $\vec{S}_{P_x}$ denote the sub-vector of $\vec{S}$ corresponding to the $x^{th}$ partition. Then, the action factorization used is: $Q_q(\vec{S}) = \sum_{x=1}^\ell Q^{P_x}_q(\vec{S}_{P_x})$ for query actions and $Q_i(\vec{S}) = \sum_{x=1}^\ell Q^{P_x}_i(\vec{S}_{P_x})$ for invite actions.


Intuitively, $Q_i^{P_x}(\vec{S}_{P_x})$ can be seen as the Q-value of inviting only nodes given by $\vec{S}_{P_x}$ (and no other nodes). Now, if querying/inviting nodes of one partition has negligible effect/influence on the other partitions, then the Q-value of the overall invite action $\tuple{\bm{S}, i}$ can be approximated by the sum of the Q-values of the sub-actions $\tuple{\bm{S_{P_x}}, i}$. The same holds for query actions. We now show that this action factorization is appropriate for CAIM as it introduces \textit{minimal} error into the influence spread calculations for stochastic block model (SBM) networks, which mimic many properties of real-world networks \cite{seshadhri2012community}. Note that we consider a single round of influence spread (T=1) as empirical research by \cite{goel2012structure} shows that influence usually does not spread beyond the first hops (T=1) in real-world social networks.


\begin{theorem} \label{thm:factorization-loss}
Let $\mathcal{I}(\bm{S})$ denote the expected influence in the whole network when nodes of set $\bm{S}$ are influenced, and we have one round of influence spread. For an SBM network with $n$ nodes and parameters $(p, q)$ that is partitioned into $\ell$ communities, the difference between the true and factored expected influences can be bounded as $\mathbb{E} \left[ \max_{\bm{S}} \left| \mathcal{I}(\bm{S}) - \sum_{x=1}^\ell \mathcal{I}(S_{P_x}) \right| \right] \leq qn^2 \left(1 - \frac{1}{\ell}\right) p_{m}$, where $p_{m} = \max_{e \in E} p(e)$ is the maximum propagation probability. Note that the (outer) expectation is over the randomness in the SBM network model.
\end{theorem}
\begin{proof}
The difference between $\sum_{x=1}^\ell \mathcal{I}(S_{P_x})$ and $\mathcal{I}(\bm{S})$ comes from the fact that $\sum_{x=1}^\ell \mathcal{I}(S_{P_x})$ over-counts influence spread across communities [since $\mathcal{I}(S_{P_x})$ equals the expected influence in the whole graph when $S_{P_x}$ is influenced, assuming no nodes of other communities are influenced, while in fact some actually may be].

Edges going across communities lead to this double counting of influence spread. We'll call these edges as cross-edges. Let $M_a$ denote the total number of such cross-edges, i.e. $M_a = |\{(u,v) \in E : u \in P_x, v \in P_y \text{ and } x \neq y\}|$. Each cross-edge can lead to at most two nodes being double counted. This is because of the following: Let $(u,v)$ be a cross-edge (where $u \in P_x$ and $v \in P_y$), and suppose that both these nodes are influenced. On computing $\mathcal{I}(S_{P_x})$, $v$ might be counted as being influenced by it [even though $v$ is already influenced beforehand], hence leading to an over-count of $1$ [Note that, since we're considering one round of influence spread, $\mathcal{I}(S_{P_x})$ assumes that $v$ does not propagate influence further]. Similar holds with $\mathcal{I}(S_{P_y})$.

Hence, $\sum_{x=1}^\ell \mathcal{I}(S_{P_x}) - \mathcal{I}(\bm{S})$ is bounded by twice the expected number of cross-edges that are activated (for arbitrary $\bm{S}$). Let $E_{ij}$ be the random variable denoting whether there's an edge from node $i$ to $j$ in the SBM network. Then, the number of cross-edges is given as
$$M_a = \frac{1}{2} \sum_{x=1}^\ell \sum_{i \in P_x} \sum_{j \notin P_x} E_{ij},$$
Hence, the expected number of cross edges is
\begin{align*}
\E[M_a] &= \E\left[\frac{1}{2} \sum_{x=1}^\ell \sum_{i \in P_x} \sum_{j \notin P_x} E_{ij}\right]\\
&= \frac{1}{2} \sum_{x=1}^\ell \sum_{i \in P_x} \sum_{j \notin P_x} q
= \frac{1}{2} \sum_{x=1}^\ell \sum_{i \in P_x} (n - |P_x|) q\\
&= \frac{q}{2} \sum_{x=1}^\ell |P_x| (n - |P_x|)\\
&= \frac{q}{2} \left(n^2 - \sum_{x=1}^\ell  |P_x|^2\right).
\end{align*}
Since $\sum_{x=1}^\ell |P_x|$ is equal to $n$, $\sum_{x=1}^\ell  |P_x|^2$ is minimized when each $|P_x|$ is equal to $n/\ell$, i.e.
$$\sum_{x=1}^\ell  |P_x|^2 \geq \sum_{x=1}^\ell \left(\frac{n}{\ell} \right)^2 = \frac{n^2}{\ell}.$$
Substituting it above:
\begin{align*}
\E[M_a] \leq \frac{q}{2}\left(n^2 - \frac{n^2}{\ell}\right) = \frac{qn^2}{2} \left(1 - \frac{1}{\ell}\right).
\end{align*}

Let $p_m = \max_{e \in E} p(e)$. Remember that each cross edge $e$ is activated with probability $p(e)$ ($\leq p_m$). So, we have
$$\E\left[\max_{\bm{S}} \left(\sum_{x=1}^\ell \mathcal{I}(S_{P_x}) - \mathcal{I}(\bm{S})\right)\right] \leq 2 \cdot \E[M_a] \cdot p_{m}$$
And therefore,
$$\E\left[\max_{\bm{S}} \left(\sum_{x=1}^\ell \mathcal{I}(S_{P_x}) - \mathcal{I}(\bm{S})\right)\right] \leq qn^2 \left(1 - \frac{1}{\ell}\right) p_{m}$$

Also, note that $\sum_{x=1}^\ell \mathcal{I}(S_{P_x})$ is always at least as large as $\mathcal{I}(\bm{S})$, i.e. $\sum_{x=1}^\ell \mathcal{I}(S_{P_x}) - \mathcal{I}(\bm{S}) \geq 0$. This gives us the desired result:
$$\E\left[\max_{\bm{S}} \left|\sum_{x=1}^\ell \mathcal{I}(S_{P_x}) - \mathcal{I}(\bm{S})\right|\right] \leq qn^2 \left(1 - \frac{1}{\ell}\right) p_{m}$$
\end{proof}

This action factorization allows maintaining separate Q-value statistics ($\hat{Q}_{type}^{P_x}(\vec{S}_{P_x})\mbox{ }\forall type \in \{q,i,e\}$) for each factor (i.e., network community) at each node of the UCT tree maintained by CAIMS. However, upon running MC simulations in this UCT tree, we acquire samples of only $Q_{type}$ (i.e., rewards of the joint \textit{un-factored} actions). We learn factored estimates $Q_{type}^{P_x}$ from estimates $Q_{type}$ of the \textit{un-factored} actions by using mixture of experts optimization \cite{amato2015scalable}, i.e. we estimate the factors as $\hat{Q}_{type}^{P_x}(\vec{S}_{P_x}) = \alpha_{P_x} \mathbb{E}[Q_{type}(\vec{S}) | \vec{S}_{P_x}]$, where this expectation is estimated by using the empirical mean. Please refer to \cite{amato2015scalable} for more details. We now describe action selection in the UCT tree. 

\textbf{Action Selection} At each node in the UCT tree, we use the UCB1 rule (over all factors) to find the best action. Let $n^q_{h\vec{S}_{P_x}}$ (or $n^i_{h\vec{S}_{P_x}}$) denote the number of times a query (or invite) action with sub-action $\vec{S}_{P_x}$ has been taken from node $h$ of the UCT tree. Let $N_h$ denote the number of times tree node $h$ has been visited. The best query action to be taken is given as $\tuple{\bm{S_q}, q}$, where $\vec{S}_q = argmax_{\|\vec{S}\|_1 \leq Q_{max}} \sum_{x=1}^\ell \hat{Q}^{P_x}_q(b_h, \vec{S}_{P_x}) + c \sqrt{log(N_h + 1)/n^q_{h\vec{S}_{P_x}}}$. Similarly, the best invite action to be taken is given as $\tuple{\bm{S_i}, i}$, where $\vec{S}_i = argmax_{\|\vec{S}\|_1 \leq K - |M|} \sum_{x=1}^\ell \hat{Q}^{P_x}_i(b_h, \vec{S}_{P_x}) + c \sqrt{log(N_h + 1)/n^i_{h\vec{S}_{P_x}}}$ (where $M$ is the set of locked nodes at tree node $h$). Let $V_q$ and $V_i$ denote the value attained at the maximizing query and invite actions, respectively. Finally, let $V_e$ denote the value of the end-session action, i.e. $V_e = \hat{Q}_e + c \sqrt{log(N_h + 1)/n^e_h}$ where $n^e_h$ is the number of times the end-session action has been taken from tree node $h$. Then, the values $V_q, V_i$ and $V_e$ are compared and the action corresponding to $max(V_q,V_i,V_e)$ is chosen.

\textbf{Improved VE} Note that the UCB1 equations to find maximizing query/invite actions (as described above) are of the form $argmax_{\|\vec{a}\|_1 \leq z} \sum_{x=1}^\ell f_x(\vec{a}_x)$ (where $\vec{a} \in \{0,1\}^n$). Unfortunately, plain application of VE (like FV-POMCP) to this results in infeasible solutions which may violate the L-1 norm constraint. Thus, FV-POMCP's VE procedure may not produce feasible solutions for CAIM. 

CAIMS addresses this limitation by using two adjustments. First, we incorporate this L-1 norm constraint as an additional factor in the objective function: $argmax_{\vec{a} \in \{0,1\}^n} \sum_{x=1}^\ell f_x(\vec{a}_x) + f_c(\vec{a})$. This constraint factor $f_c$'s scope is all the $n$ variables (as it represents a global constraint connecting actions selected across all factors), and hence it can be represented using a table of size $O(2^n)$ in VE. Unfortunately, the exponentially sized table of $f_c$ eliminates any speed-up benefits that VE provides, as the induced width of the tree formed (on running VE) will be $n$, leading to a worst possible time-complexity of $O(2^n)$. 

To resolve this, CAIMS leverages a key insight which allows VE to run efficiently even with the additional factor $f_c$. The key idea is that, if all variables of a community are eliminated at once, then both (i)$f_c$; and (ii) the factors derived from a combination of $f_c$ and other community-specific factors during such elimination, can be represented very concisely (using just tables of size $z+1$ elements), instead of using tables of size $O(2^n)$. This fact is straightforward to see for the original constraint factor $f_c$ (as $f_c$'s table only depends on $\|\vec{a}\|_1$, it has value 0 if $\|\vec{a}\|_1 \leq z$ and $-\infty$ otherwise). However, it is not obvious why this holds for derived factors, which need to maintain optimal assignments to community-specific variables, for every possible combination of \textit{un-eliminated} variable values (thereby requiring $O(2^n)$ elements). However, it turns out that we can still represent the derived factors concisely. The key insight is that even for these derived factors, all variable assignments with the same L-1 norm have the same value (Lemma~\ref{lem:concise-factors}). This allows us to represent each of these derived factor as a table of only $z+1$ elements (as we need to store one unique value when the L-1 norm is at most $z$, and we use $-\infty$ otherwise).

\begin{lemma}	\label{lem:concise-factors}
Let $\psi_i(\vec{v})$ denote the $i^{th}$ factor generated during CAIMS's VE. Then, $\psi_i(\vec{v}_1) = \psi_i(\vec{v}_2)$ if $\|v_1\|_1 = \|v_2\|_1$. Further $\psi_i(\vec{v}) = -\infty$ if $\|v\|_1 > z$.
\end{lemma}

These compact representations allow CAIMS to efficiently run VE in time $\sum_{i=1}^\ell O\left(2^{s_i}\right)$ ($s_i = $ size of $i^{th}$ community) even after adding the global constraint factor $f_c$ (Lemma~\ref{lem:time-complex}). In fact, this is the best one can do, because any algorithm will have to look at all values of each community-specific factor in order to solve the problem.

\begin{lemma} \label{lem:time-complex}
CAIMS's VE has time-complexity $\sum_{i=1}^\ell O\left(2^{s_i}\right)$, where $s_i$ is the size of the $i^{th}$ factor (community). There exists no procedure with better time complexity.
\end{lemma}
\begin{proof}[Proof of Lemma 3 \& 4]
We go over the exact procedure of the modified VE algorithm and prove Lemmas 3 and 4 in the process. For the forward pass, we compute $\max_{\vec{a}} \sum_{x=1}^\ell f_x(\vec{a}_x) + f_c(\vec{a})$. We know that $f_c$ depends only on the L-1 norm of $\vec{a}$, so we represent it as $f_c(\|\vec{a}\|_1)$. Also note that, the communities are disjoint, because of which each action bit $a_i$ (of action $\vec{a}$) appears in the argument of exactly one factor $f_x$ (other than the constraint factor $f_c$).

As mentioned in the paper, we eliminate all variables of a community at once. So, to eliminate the first block of variables, we compute $\max_{\vec{a}_1} f_1(\vec{a}_1) + f_c(\|\vec{a}\|_1) = \psi_1(\|\vec{a}_{-1}\|_1)$, where $\vec{a}_{-1}$ denotes all action bits of $\vec{a}$ except those in $\vec{a}_1$. Note that, in the RHS of this expression, we use $\|\vec{a}_{-1}\|_1$ as opposed to $\vec{a}_{-1}$ itself because the LHS (before computing the max) depends only on $\vec{a}_1$ and $\|\vec{a}_1\|_1 + \|\vec{a}_{-1}\|_1$. Also, note that for $\|\vec{a}_{-1}\|_1 > z$, we have $\|\vec{a}\|_1 > z$ making $f_c(\|\vec{a}\|_1)$ and $\psi_1(\|\vec{a}_{-1}\|_1)$ equal to $-\infty$.

To make this more concrete, Table~\ref{ve-table} shows how $\psi_1$ is exactly computed. Here, $v_i^{(x)}$ denotes the maximum value of $f_x$ when exactly $i$ bits of $\vec{a}_x$ are $1$, and $s_x$ denotes the number of bits in $\vec{a}_x$.

\begin{table}[h]
\centering
\begin{tabular}{@{}cc@{}}
\toprule
$\|\vec{a}_{-1}\|_1$  & $\psi_1(\|\vec{a}_{-1}\|_1)$  \\  \midrule
$0$ & $\max\left(v^{(1)}_0 + f_c(0), v^{(1)}_1 + f_c(1), \cdots v^{(1)}_{s_1} + f_c(s_1)\right)$ \\ \midrule
$1$ & $\max\left(v^{(1)}_0 + f_c(1), v^{(1)}_1 + f_c(2), \cdots v^{(1)}_{s_1} + f_c(s_1+1)\right)$ \\ \midrule
$\vdots$ & $\vdots$  \\ \midrule
$z$ & $v^{(1)}_0 + f_c(z)$  \\ \midrule
$> z$ & $-\infty$ \\ \bottomrule
\end{tabular}
\caption{Factor obtained on (first) block elimination}
\label{ve-table}
\end{table}

\sloppy Apart from computing the maximum objective value (forward pass), we also need to compute the maximizing assignment of the problem (backward pass). For this, we maintain another function $\mu_1(\|\vec{a}_{-1}\|_1)$ which keeps track of the value of $\vec{a}_1$ at which this maximum is attained (for each value of $\|\vec{a}_{-1}\|_1$), i.e. $\mu_1(v) = argmax_{\vec{a}_1} \left[f_1(\vec{a}_1) + f_c(\|\vec{a}_1\|_1 + v)\right]$. After eliminating variables of the first community, we are left with $\max_{\vec{a}_{-1}} \sum_{x=2}^\ell f_x(\vec{a}_x) + \psi_1(\|\vec{a}_{-1}\|_1)$. We repeat the same procedure and eliminate $\vec{a}_2$ by computing $\max_{\vec{a}_2} f_2(\vec{a}_2) + \psi_1(\|\vec{a}_{-1}\|_1)$, to obtain $\psi_2(\|\vec{a}_{-1,-2}\|_1)$. Note that, again, $\psi_2$ depends only on the L-1 norm of the remaining variables. Also, for $\|\vec{a}_{-1,-2}\|_1 > z$, $\psi_2$ becomes $-\infty$. In a similar way, this holds for the remaining generated factors, giving Lemma 4.

Once we complete the forward pass, we are left with $\psi_\ell(0)$ which is the maximum value of the objective function. Then, as in standard VE, we backtrack and use the $\mu_x$ functions to obtain the maximizer $argmax_{\vec{a}} \sum_{x=1}^\ell f_x(\vec{a}_x) + f_c(\|\vec{a}\|_1)$, i.e. $\mu_\ell(0)$ gives us the value of $\vec{a}_\ell$, then $\mu_{\ell-1}(\|\vec{a}_\ell\|_1)$ gives us the value of $\vec{a}_{\ell-1}$, $\mu_{\ell-2}(\|\vec{a}_\ell\|_1 + \|\vec{a}_{\ell-1}\|_1)$ gives us the value of $\vec{a}_{\ell-2}$ and so on.

Observe that to compute the $i^{th}$ derived factor, we needed to compute $\max_{\vec{a}_i} f_i(\vec{a}_i) + \psi_{i-1}(\|\vec{a}_{-1, -2, \cdots -(i-1)}\|_1) = \psi_i(\|\vec{a}_{-1, -2, \cdots -i}\|_1)$. And for this, we just need to compute $v^{(i)}_s$ for each $s = 0, 1, \cdots s_i$, as evident from Table~\ref{ve-table}. This takes time $O(2^{s_i})$, where $s_i$ denotes the size of the $i^{th}$ community. Hence, the time complexity of the whole algorithm is $\sum_{i=1}^\ell O\left(2^{s_i}\right)$.
\end{proof}

\textbf{Markov Net Beliefs}  FV-POMCP uses unweighted particle filters to represent beliefs, i.e. a belief is represented by a collection of states (also known as particles), wherein each particle has an equal probability of being the true state. Unfortunately, due to CAIM's exponential state-space, this representation of beliefs becomes highly inaccurate which leads to losses in solution quality.

To address this limitation, CAIMS makes the following assumption: availability of network nodes is positively correlated with the availability of their neighboring nodes in the social network. This assumption is reasonable because homeless youth usually go to shelters with their friends \cite{rice2013should}. Thus, the confirmed availability of one homeless youth increases the likelihood of the availability of his/her friends (and vice versa). Under this assumption, the belief state in CAIM can be represented using a Markov Network. Formally, the belief is given as $b = \tuple{\mathcal{N}, \bm{M}, numAct, sessID}$, where $\mathcal{N}$ is a Markov Network representing our belief of the true realization $\phi$ (note that the other three components of a state are observable). With the help of this Markov Network, we maintain \textit{exact} beliefs throughout the POMCP tree of CAIMS. As mentioned before, the prior distribution $\bm{\Phi}$ that serves as part of the initial belief state is also represented using a Markov Network $\mathcal{N}_0$. This prior can be elicited from field observations made by homeless shelter officials, and can be refined over multiple runs of CAIMS. \textit{In our simulations, the social network structure $G=(\bm{V},\bm{E})$ is used as a surrogate for the Markov network structure, i.e., the Markov network only has potentials over two variables/nodes (one potential for each pair of nodes connected by an edge in social network $G$)}. Thus, we start with the initial belief as $\tuple{\mathcal{N}_0, \{\}, 0, 1}$. Upon taking actions $a = \tuple{S, type}$ and receiving observations $o$, the belief state can be updated by conditioning the Markov network on the observed variables (i.e., by conditioning the presence/absence of nodes based on observations received from past query actions taken in the current session). This helps us maintain exact beliefs throughout the POMCP tree efficiently, which helps CAIMS take more accurate decisions.

\section{Evaluation}\label{sec:eval}
We show simulation results on artificially generated (and real-world) networks to validate CAIMS's performance in a variety of settings. We also provide results from a real-world feasibility study involving 54 homeless youth which shows the real-world usability of CAIMS. For our simulations, all the networks were generated using NetworkX library \cite{networkX}. All experiments are run on a 2.4 GHz 8-core Intel machine having 128 GB RAM. Unless otherwise stated, we set $L=3$, $Q_{max}=2$, $K=2$, and all experiments are averaged over 50 runs. \textit{All simulation results are statistically significant under t-test ($\alpha=0.05$)}.

\textbf{Baselines} We use two different kinds of baselines. For influence maximization solvers, we use Greedy \cite{kempe2003maximizing}, the gold-standard in influence maximization as a benchmark. We subject Greedy's chosen nodes to contingencies drawn from the same prior $\bm{\Phi}$ distribution that CAIMS uses. We also compare against the overprovisioning variant of Greedy (Greedy+) where instead of selecting $K$ nodes, we select $2K$ nodes and influence the first $K$ nodes that accept the invitation. This was proposed as an ad-hoc solution in \cite{yadav2017influence} to tackle contingencies, and hence, we compare CAIMS against this. We also compare CAIMS against state-of-the-art POMDP solvers such as SARSOP and POMCP. Unfortunately, FV-POMCP cannot be used for comparison as its VE procedure is not guaranteed to satisfy the $K$ budget constraint used inside CAIMS.

\textbf{Solution Quality Comparison}  Figures \ref{fig:1}, \ref{fig:2} and \ref{fig:5} compares influence spread of CAIMS, Greedy, Greedy+ and POMCP on SBM ($p=0.4, q=0.1$), Preferential Attachment (PA) ($n=5$) and real-world homeless youth networks (used in \cite{yadav2016using}), respectively. We select $K=2$ nodes, and set $T=6, L=3$ for CAIMS. The X-axis shows the size of the networks and the Y-axis shows the influence spread achieved. Figures \ref{fig:1} and \ref{fig:2} show that on SBM and PA networks, POMCP runs out of memory on networks of size 120 nodes. Further, these figures also show that CAIMS significantly outperforms Greedy and Greedy+ on both SBM (by $\sim$73\%) and PA networks (by $\sim$58\%). Figure \ref{fig:5} shows that even on real-world networks of homeless youth (which had $\sim$160 nodes each) , POMCP runs out of memory, while CAIMS outperforms Greedy and Greedy+ by $\sim$25\%. This shows that state-of-the-art influence maximization solvers perform poorly in the presence of contingencies, and a POMDP based method (CAIMS) outperforms them by explicitly accounting for contingencies. Figures~\ref{fig:1} and~\ref{fig:2} also show that Greedy+ performs worse than Greedy.
 
\begin{figure}[t]
\subfloat[SBM Networks]{\includegraphics[height=1.7in,width=0.47\columnwidth]{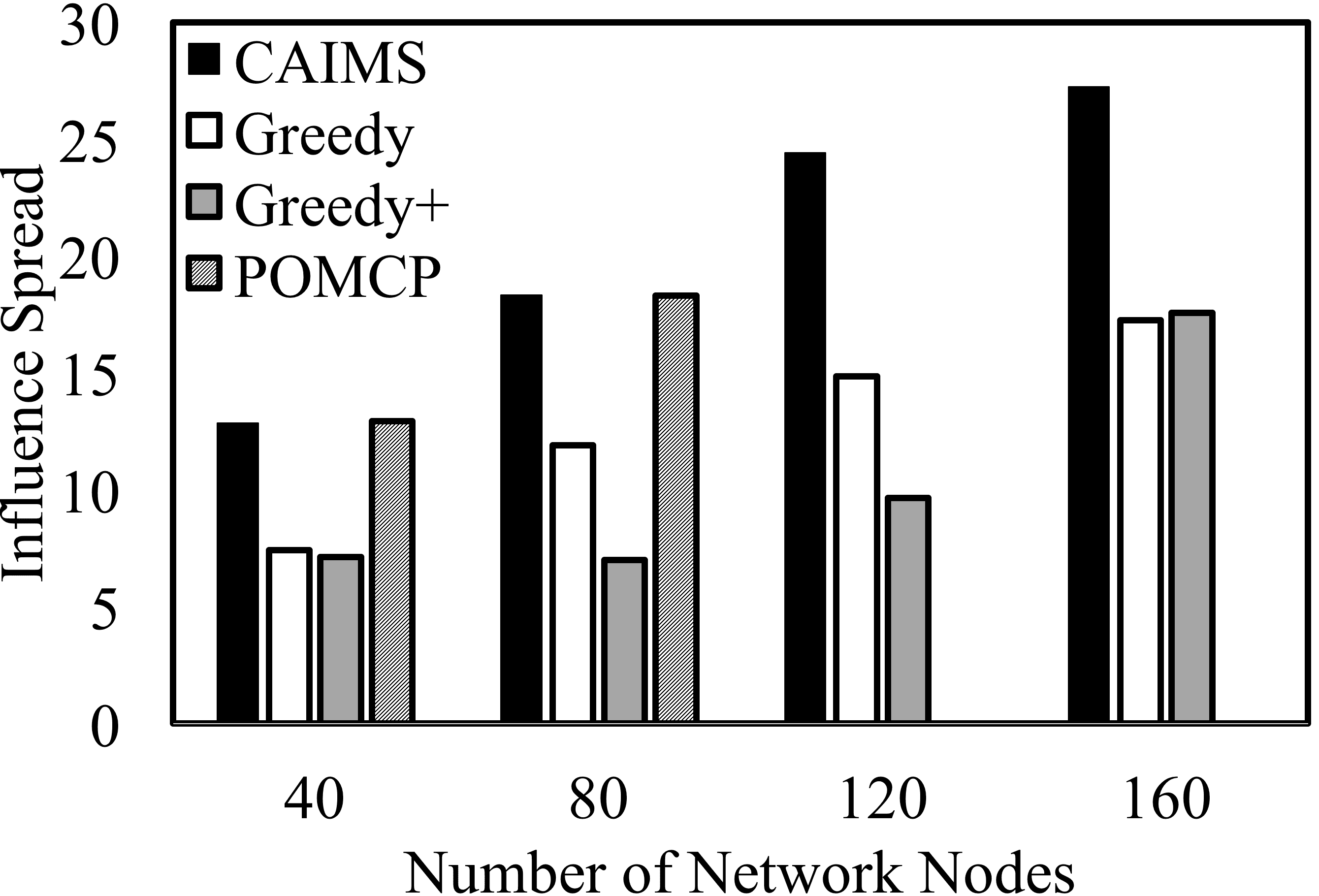}\label{fig:1}}
\hspace{2mm}
\subfloat[PA networks]{\includegraphics[height=1.7in,width=0.47\columnwidth]{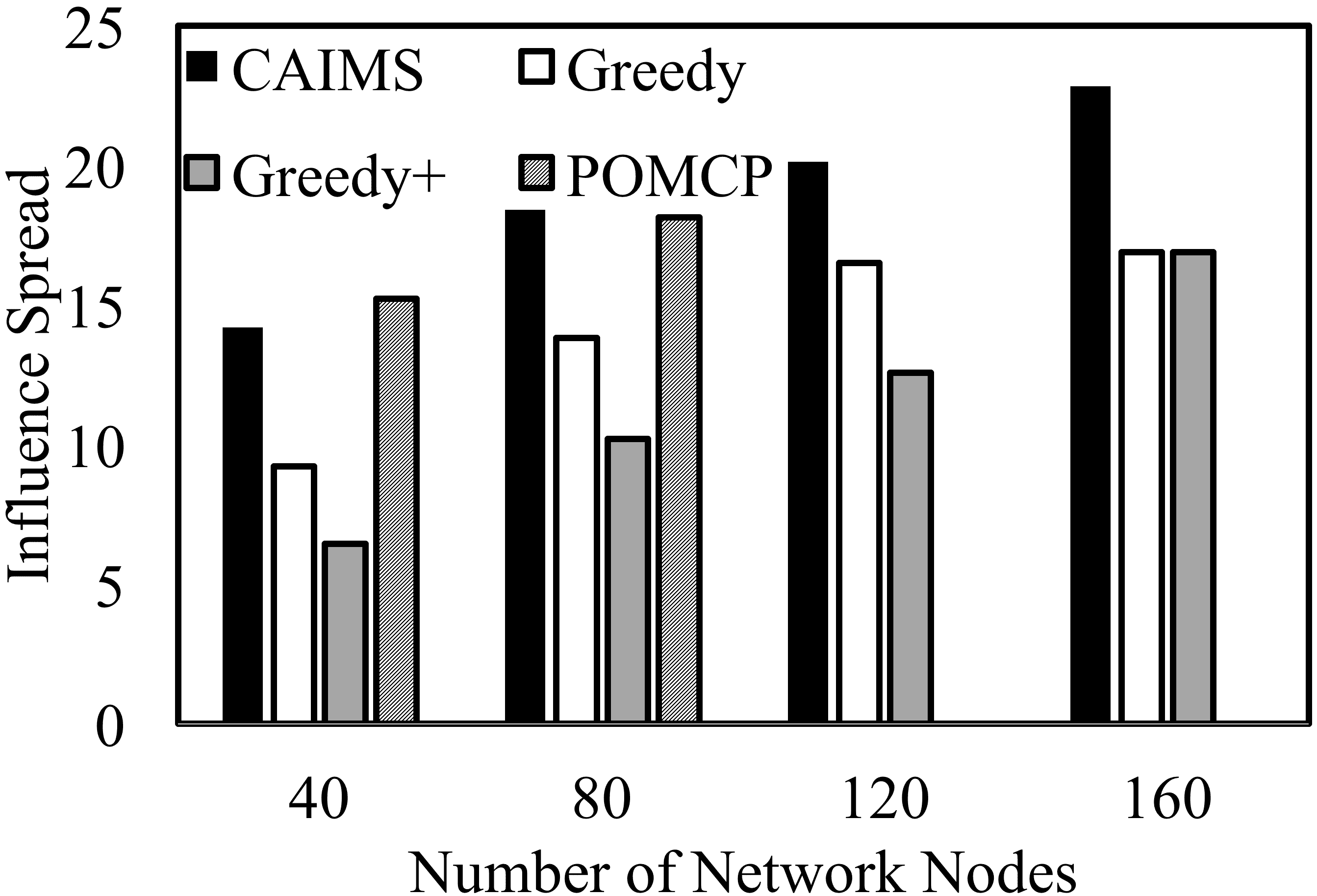}\label{fig:2}}
\caption{Influence Spread Comparison}
\end{figure}

\begin{figure}[t]
\subfloat[Scale up in $T$]{\includegraphics[height=1.7in,width=0.47\columnwidth]{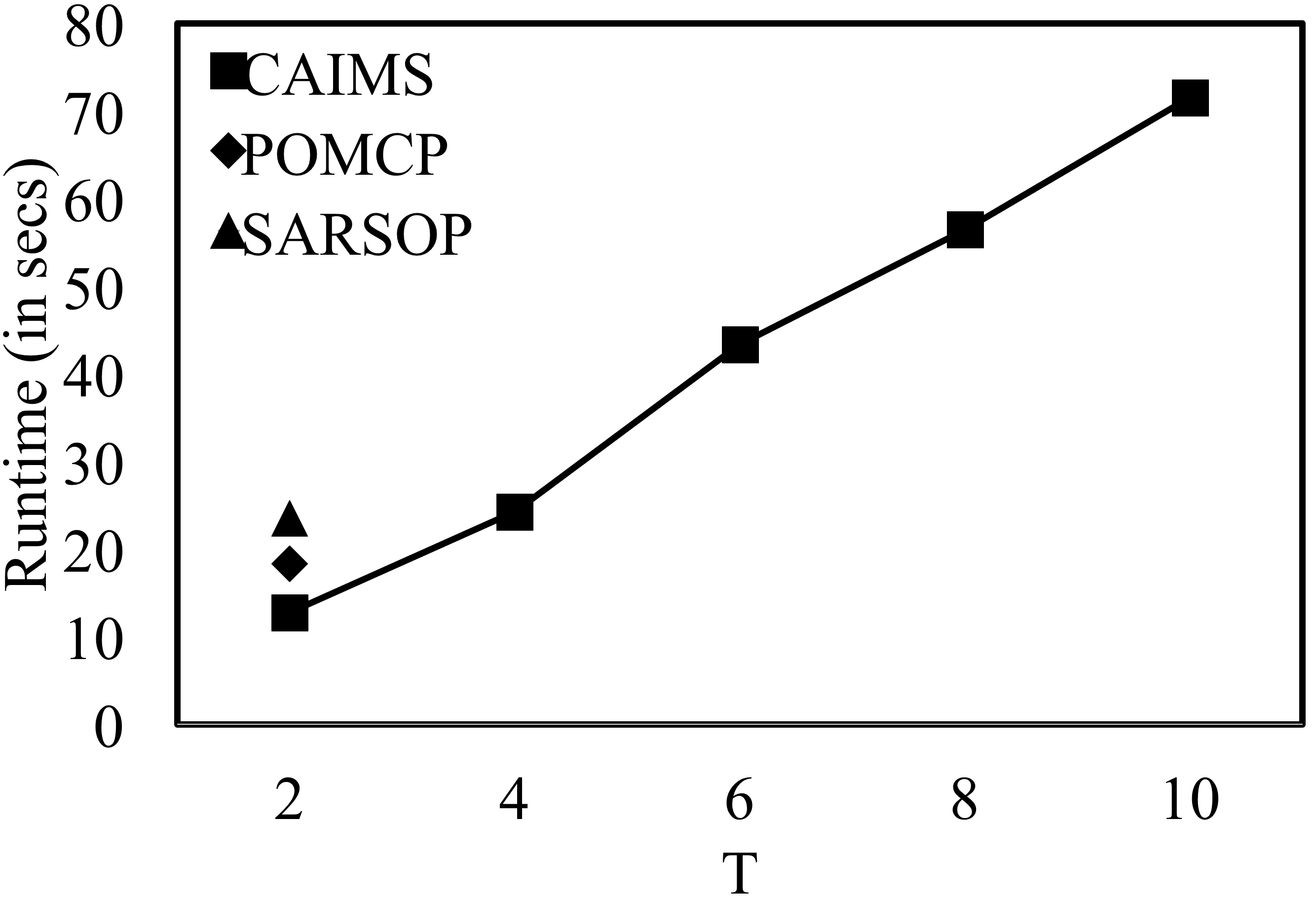}\label{fig:3}}
\hspace{2mm}
\subfloat[Scale up in $K$]{\includegraphics[height=1.7in,width=0.47\columnwidth]{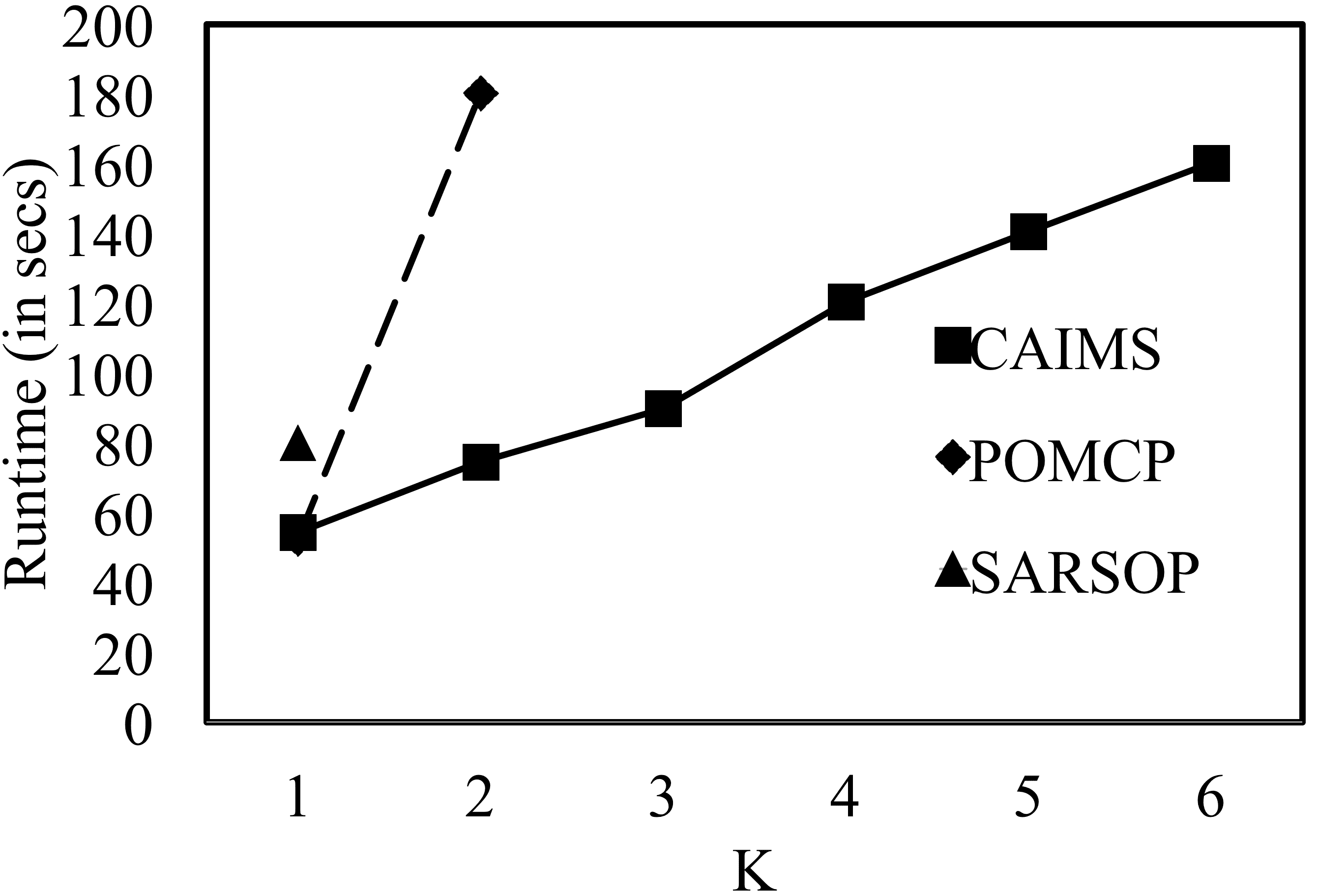}\label{fig:4}}
\caption{Scale Up Results}
\end{figure}

\textbf{Scale up} Having established the value of POMDP based methods, we now compare CAIMS's scale-up performance against other POMDP solvers. Figures \ref{fig:3} and \ref{fig:4} compares the runtime of CAIMS, POMCP and SARSOP on a 100 node SBM network with increasing values of $T$ and $K$ respectively. The X-axis shows $T$ (or $K$) values and the Y-axis shows the influence spread. Figure \ref{fig:3} shows that both POMCP and SARSOP run out of memory at $T=2$ sessions. On the other hand, CAIMS scales up gracefully to increasing number of sessions. Similarly, Figure \ref{fig:4} ($T=10$) shows that SARSOP runs out of memory at $K=1$, whereas POMCP runs out memory at $K=2$, whereas CAIMS scales up to larger values of $K$. These figures establish the superiority of CAIMS over its baselines as it outpeforms them over a multitude of parameters and network classes.

\begin{figure}[ht]
\subfloat[Influence Spread]{\includegraphics[height=1.7in,width=0.47\columnwidth]{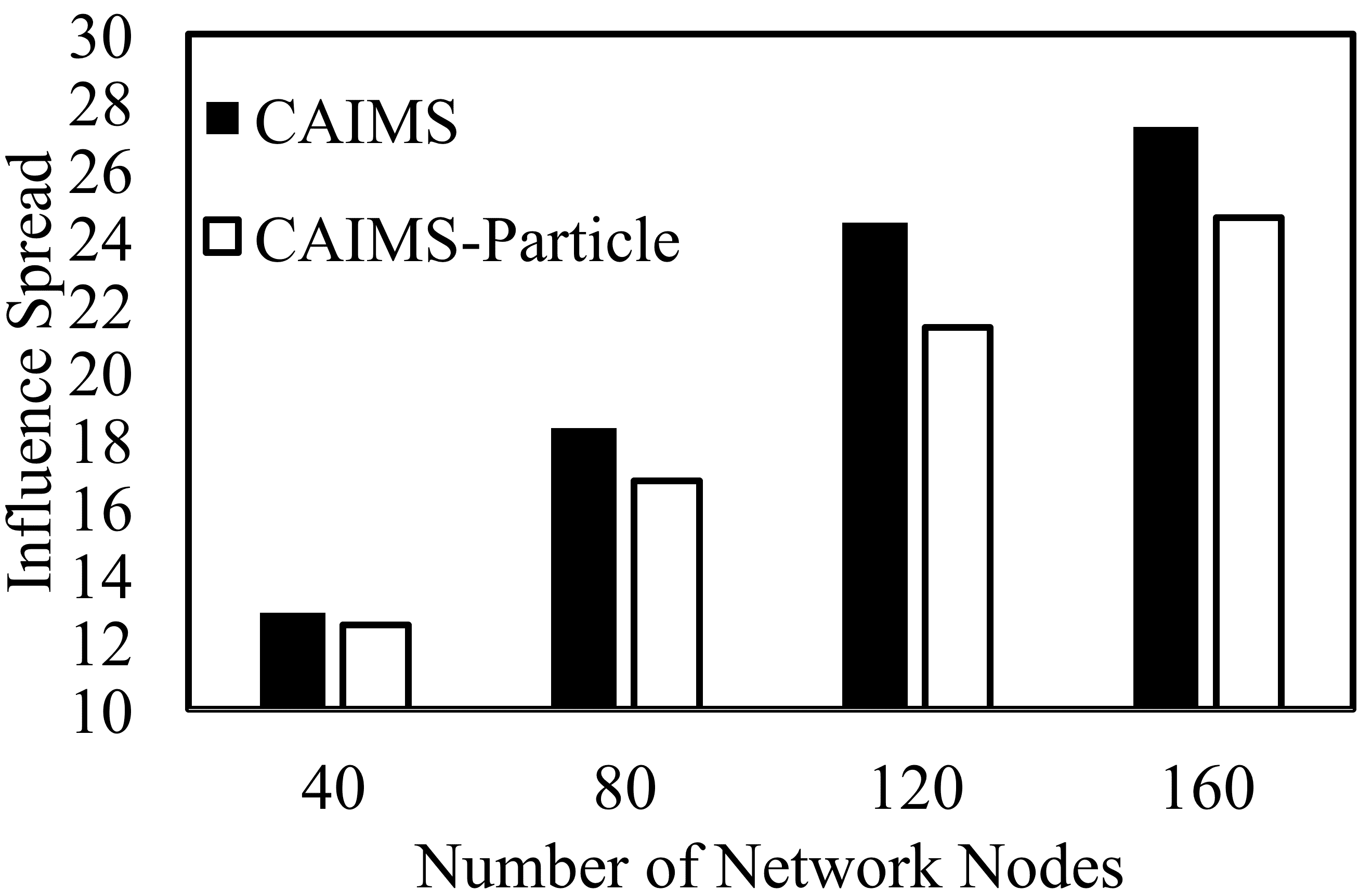}\label{fig:7}}
\hspace{2mm}
\subfloat[Runtime]{\includegraphics[height=1.7in,width=0.47\columnwidth]{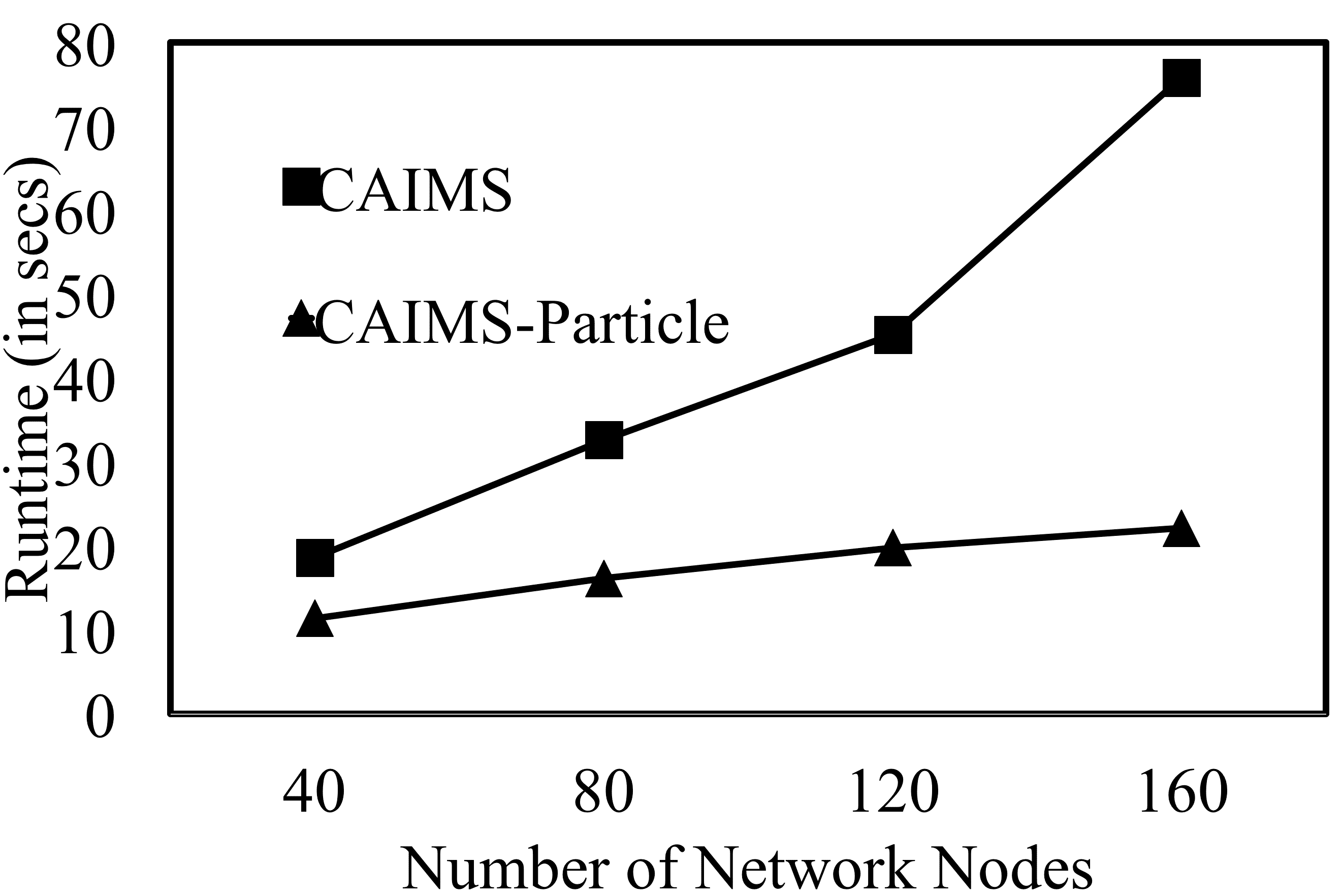}\label{fig:8}}
\caption{Value of using Markov Networks}
\end{figure}

\textbf{Markov Nets} We illustrate the value of Markov networks to represent belief states in CAIMS. We compare CAIMS with and without Markov nets (in this case, belief states are represented using unweighted particle filters) on SBM networks of increasing size. Figure \ref{fig:7} shows influence spread comparison between CAIMS and CAIMS-Particle (the version of CAIMS which uses unweighted particle filters to represent belief states). Figure \ref{fig:8} shows runtime comparison of CAIMS and CAIMS-Particle on the same SBM networks. These figures shows that using a more accurate representation for the belief state (using Markov networks) improved solution qualities by $\sim$15\% at the cost of $\sim$3X slower runtime. However, the loss in speed due to Markov networks is not a concern (as even on 160 node networks, CAIMS with Markov networks runs in $\sim$75 seconds).

\begin{figure}[ht]
\subfloat[Homeless Youth Networks]{\includegraphics[height=1.7in,width=0.47\columnwidth]{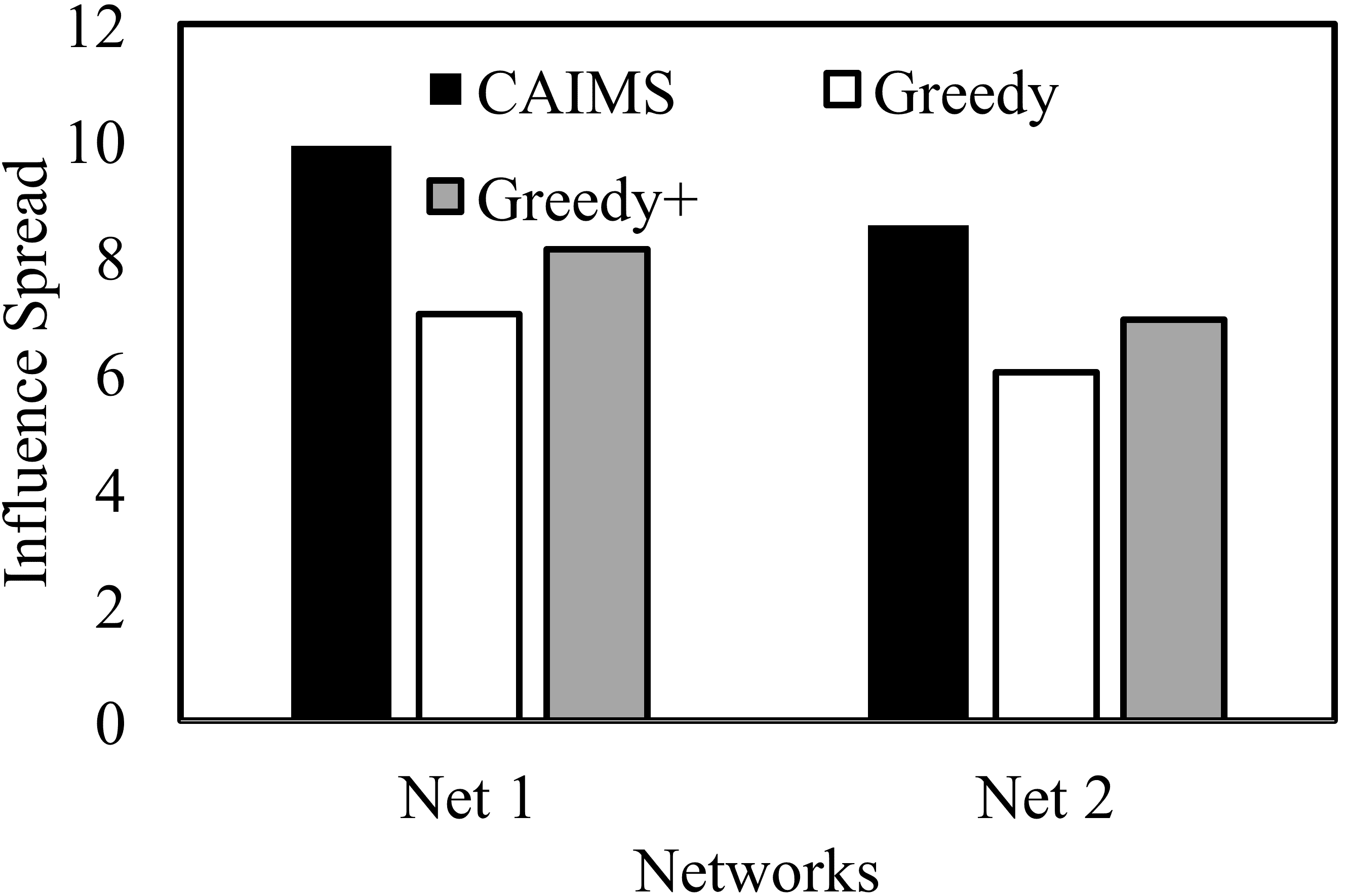}\label{fig:5}}
\hspace{2mm}
\subfloat[Feasibility Trial]{\includegraphics[height=1.7in,width=0.47\columnwidth]{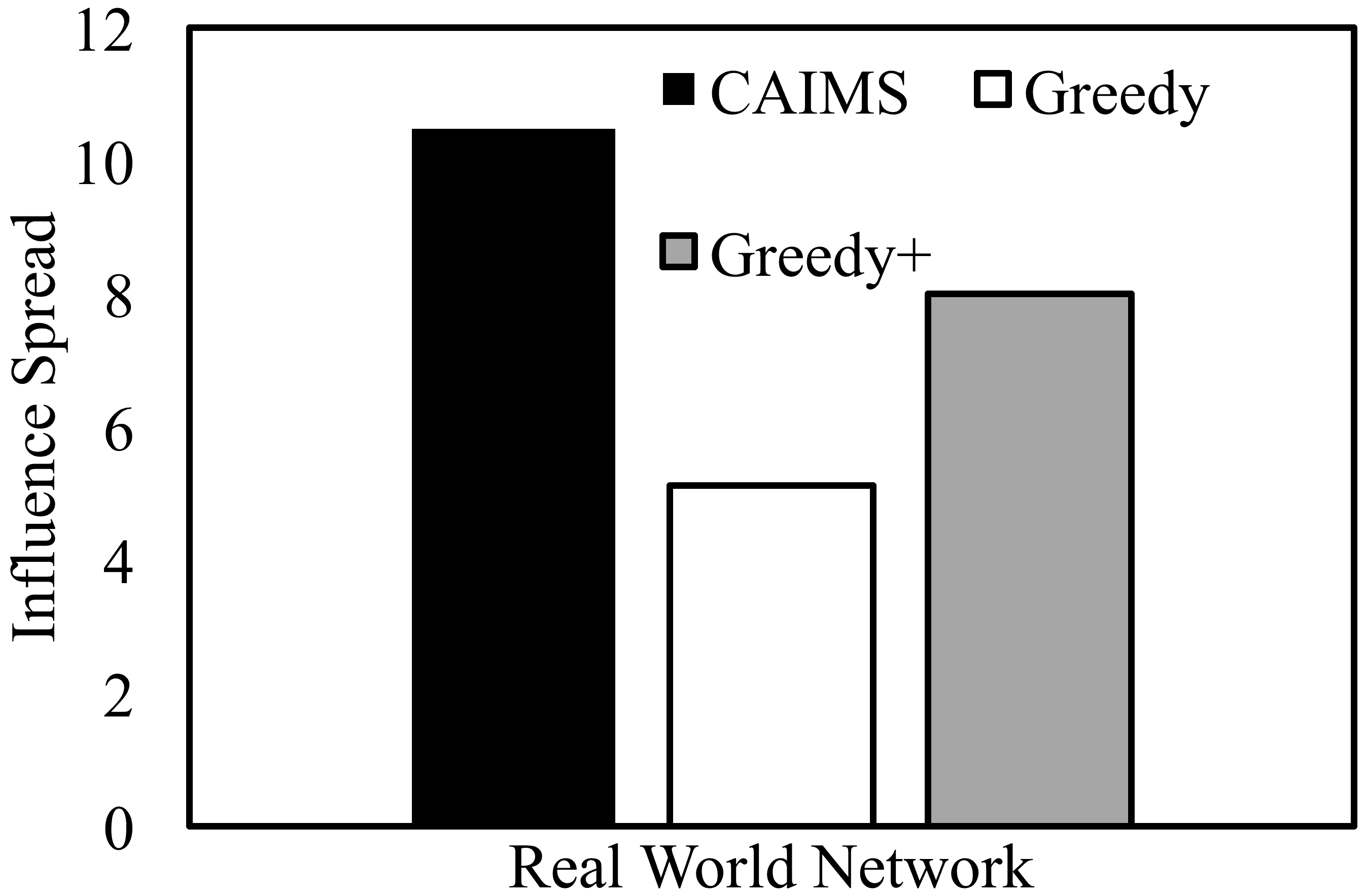}\label{fig:6}}
\caption{Real World Experiments}
\end{figure}

\textbf{Real World Trial}
We conducted a real-world feasibility trial to test out CAIMS with a homeless shelter in Los Angeles. We enrolled 54 homeless youth from this shelter into our trial and constructed a friendship based social network for these youth (using social media contacts). The prior $\bm{\Phi}$ was constructed using field observations made by shelter officials. We then executed policies generated by CAIMS, Greedy and Greedy+ on this network ($K=4$, $Q_{max}=4$ and $L=3$) on three successive days ($T=3$) in the shelter to invite homeless youth to attend the intervention. In reality, 14 out of 18 invitations ($\sim$80\%) resulted in contingency events, which illustrates the importance of accounting for contingencies in influence maximization.  Figure \ref{fig:6} compares influence spread (in simulation) achieved by nodes in invited sets selected by CAIMS, Greedy and Greedy+. This figure shows that CAIMS is able to spread 31\% more influence as compared to Greedy and Greedy+.

\section{Conclusion}
Most previous influence maximization algorithms rely on the following assumption: seed nodes can be influenced with certainty. Unfortunately, this assumption does not hold in most real-world domains. This paper presents CAIMS, a contingency-aware influence maximization algorithm for selecting key influencers in a social network.  Specifically, this paper makes the following five contributions: (i) we propose the Contingency-Aware Influence Maximization problem and provide a theoretical analysis of the same; (ii) we cast this problem as a Partially Observable Markov Decision Process (POMDP); (iii) we propose CAIMS, a novel POMDP planner which leverages a natural action space factorization associated with real-world social networks; (iv) we provide extensive simulation results to compare CAIMS with existing state-of-the-art influence maximization algorithms; and (v) we test CAIMS in a real-world feasibility trial which confirms that CAIMS is indeed a usable algorithm in the real world. Currently, CAIMS is being reviewed by homeless youth service providers for further deployment.

\chapter{Conclusion and Future Work}
\label{chapter:conclusion}
Artificial Intelligence has made rapid advances in the last couple of decades, which has led to an enormous array of AI applications and tools that play an integral role in our society today. However, despite this exciting success, strong market forces ensure that most AI applications are developed to mostly serve people in urban areas, who are financially and geographically well placed to access these AI applications and tools for their personal benefit. Unfortunately, even in 2017, a large proportion of the world population (45\% to be precise) lives in extremely rural areas. Even more tragically, almost 80\% of the world population still lives in extreme poverty (i.e., they survive on less than USD 2.5 a day). Thus, a very large fraction of people on planet Earth do not have either financial resources, or cannot access most AI based applications. Moreover, these low-resource communities suffer from a large range of problems (such as access to healthcare, good education, nutritional food and respectful employment, among others), which have not been tackled by Artificial Intelligence and Computer Science as much.

This thesis takes a forward step towards solving challenges faced by some low-resource communities. In order to better understand these problems, this thesis is the result of long-lasting collaborations and communications with researchers at the USC School of Social Work and practitioners at several homeless youth service providers. I have been very fortunate to work directly with these domain experts, learning from their experience and expertise to improve my influence maximization solutions, and especially going on the field with them to deploy my algorithms in the real-world. While this thesis primarily uses public-health issues faced by homeless youth for motivation and exposition, the algorithms, methodological advances and insights derived from this thesis could easily follow over to other domains involving low-resource communities. In particular, it emphasizes how several challenges faced by these low-resource communities can be tackled by harnessing the real-world social networks of these communities. In the recent past, governments and non-profit organizations have utilized the power of these social networks to conduct social and behavioral interventions among low-resource communities, in order to enable positive behavioral change among these communities. This thesis focuses on how the delivery of these social and behavioral interventions can be improved via algorithmic techniques, in the presence of real-world uncertainties, constraints and challenges.

While a lot of algorithmic techniques exist in the field of influence maximization to conduct these interventions, these techniques completely fail to address the following challenges: (i) unlike previous work, most problems in real-world domains are sequential decision making problems, where influence is spread in multiple stages; (ii) there is a lot of uncertainty in social network structure, evolving network state, availability of which nodes can get influenced, and the overall influenced model, which needs to be handled in a principled manner; (iii) despite two decades of research in influence maximization algorithms, none of these previous algorithms have been tested in the real-world, which leads us to question the effectiveness of influence maximization techniques in the real-world. While adopting such simplistic assumptions is a reasonable start for developing the first generation of influence maximization algorithms, it is critical to address these aforementioned challenges in order to obtain effective strategies which not only work on paper, but also in practice.

This thesis tackles each of these three challenges by providing innovative techniques and significant methodological advances for addressing real-world challenges and uncertainties in influence maximization problems. Specifically, this thesis has the following five key contributions.

\section{Contributions}
\begin{enumerate}
\item \textbf{Definition of the DIME Problem}: Over the past twenty years, researchers have mostly looked at the standard single-stage influence maximization problem (or some slight variations), and proposed increasingly complex and sophisticated algorithms to solve that problem. Informed by lots of discussions with domain experts, this thesis proposes the Dynamic Influence Maximization under Uncertainty (or DIME) problem, which is much more representative of real-world problems faced in many domains involving low-resource communities. Further, this thesis characterizes the theoretical complexity of the DIME problem and shows that it is not amenable to standard approximation techniques.

\item \textbf{POMDP Based Algorithms for the DIME Problem}: This thesis presents two algorithms to solve the DIME problem: PSINET and HEALER. PSINET is a novel Monte Carlo (MC) sampling online POMDP algorithm which makes two significant advances over POMCP (the previous state-of-the-art). First, it introduces a novel transition probability heuristic (by leveraging ideas from social network analysis) that allows storing the entire transition probability matrix in an extremely compact manner. Second, PSINET utilizes the QMDP heuristic to enable scale-up and eliminates the search tree of POMCP. On the other hand, HEALER uses several novel heuristics. First, it uses a novel two-layered \textit{hierarchical ensembling heuristic}. Second, it uses graph partitioning techniques to partition the uncertain network, which generates partitions that minimize the edges going across partitions (while ensuring that partitions have similar sizes). Since these partitions are ``almost" disconnected, we solve each partition separately. Third, it solves the \textit{intermediate POMDP} for each partition by creating smaller-sized \textit{sampled POMDPs} (via sampling uncertain edges), each of which is solved using a novel tree search algorithm, which avoids the exponential branching factor seen in PSINET \cite{yadav2015preventing}. Fourth, it uses novel aggregation techniques to combine solutions to these smaller POMDPs rather than simple plurality voting techniques seen in previous ensemble techniques \cite{yadav2015preventing}. 

\item \textbf{First-ever real-world deployment of influence maximization algorithms}: This thesis also presents first-of-its-kind results from three real-world pilot studies, involving 173 homeless youth in Los Angeles. The pilot studies helped in answering several questions that were raised in Section \ref{sec:intro}. First, we learnt that peer-leader based interventions are indeed successful in spreading information about HIV through a homeless youth social network (as seen in Figures \ref{fig:inf-spread}). These pilot studies also helped to establish the superiority (and hence, their need) of HEALER and DOSIM -- we are using complex agents (involving POMDPs and robust optimization), and they outperform DC (the modus operandi of conducting peer-led interventions) by 160\% (Figures \ref{fig:inf-spread}, \ref{fig:behav}). The pilot studies also helped us gain a deeper understanding of how HEALER and DOSIM beat DC (shown in Figures \ref{fig:edge}, \ref{fig:comcov}, \ref{fig:com}) --  by minimizing redundant edges and exploiting community structure of real-world networks.

\item \textbf{POMDP Based Algorithm to Handle Uncertainty in Availability of Nodes}: Based on experiences faced during the pilot studies, this thesis also proposes the Contingency-Aware Influence Maximization problem and provide a theoretical analysis of the same. Further, it proposes CAIMS, an algorithm to avoid contingencies (common events when homeless youth fail to show up for the interventions). CAIMS is a novel POMDP planner which leverages a natural action space factorization associated with real-world social networks, and belief space compaction using Markov networks. Results from a real-world feasibility trial involving CAIMS are also presented, which confirms that CAIMS is indeed a usable algorithm in the real world.
\end{enumerate}

\section{Future Work}
The field of Artificial Intelligence stands at an inflection point, and there could be many different directions in which the future of AI research could unfold. My long-term vision is to push AI research in a direction where it is used to help solve the most difficult social problems facing the world today. Within this broad goal, I plan to tackle fundamental computer research challenges in areas such as \textit{multiagent systems}, \textit{reasoning with uncertainty}, \textit{multiagent learning}, \textit{optimization}, and others to model social interactions and phenomena, which can then be used to assist decision makers in the real world in critical domains such as healthcare, education, poverty alleviation \cite{yadavFarmerML}, environmental sustainability \cite{yadav2015handling,nguyen2014regret,nguyen2015making,ford2016protecting,gholami2019don,gamesec19}, etc.

\begin{figure}[h]    
    \centering
    \subfloat[PTSD Veterans: Promoting Mental Health Awareness]{
    \includegraphics[height = 0.18\textwidth]{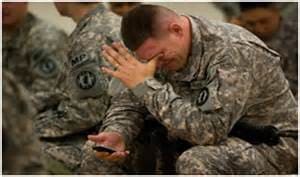}    
    }
    \subfloat[Food for Poverty Stricken: Provisioning Food at Low Costs for Poor Children ]{
    \includegraphics[height = 0.18\textwidth]{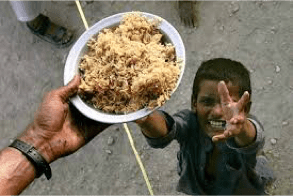}    
    }
    \subfloat[Obesity Prevention among School Children: Promoting Healthier Eating Habits]{
    \includegraphics[height = 0.18\textwidth]{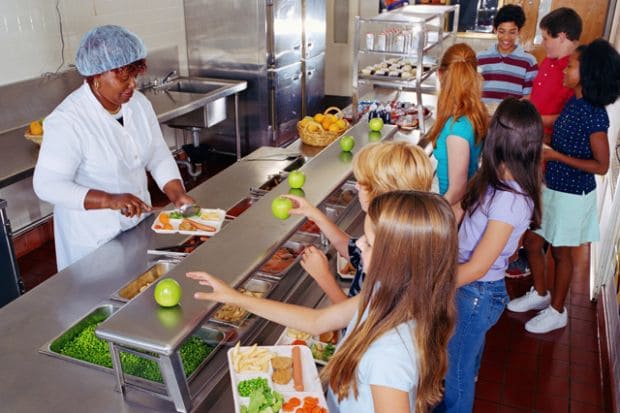}    
    }
    \caption{Potential Domains for Future Work}.
    \label{fig:example2}
\end{figure}

One example  near-term research project I will focus on is \textbf{fundamental research at the intersection of game-theory and influence maximization}, that arises from considering nodes of the social network (i.e., human beings) as self-interested agents. For example, in domains such as poverty alleviation and environment sustainability, humans (in the social network) have their own personal incentives which need to be satisfied in order for them to get influenced (and for them to spread influence). I aim to answer basic questions including how to model game theory and influence maximization in an integrated manner, defining appropriate equilibrium solution concepts, and incentivization mechanisms to achieve these notions of equilibrium. Moreover, introducing game theoretic aspects into influence maximization would require tackling a multitude of fundamental research challenges such as uncertainties about game and model parameters, learning accurate human behavior models to find optimal game theoretic strategies.

I also plan to work on introducing \textbf{spatio-temporal dynamics in influence maximization, and more generally, social network problems in the presence of data}. Most previous work in influence maximization assumes that influence spreads in the network in discrete time steps, with no regards to the spatial and temporal factors that may hinder or facilitate influence spread. As my work on activation jump model (Section 1.2) illustrates, these assumptions are not adequate to model real-world social phenomena. Moreover, in many real-world domains, the nodes in a social network (or the influencers) act in a geographical space over time. Therefore, it is important to develop models and algorithms which tackle spatio-temporal aspects such as continuity of the influence spread process over space and time, complex spatial constraints and dynamic behavior patterns  (that limit possible paths of influence spread).

\begin{singlespace} 

\bibliographystyle{plain}
\bibliography{main}

\end{singlespace}  

%

\end{document}